%% file: main_arxiv.tex
\documentclass[10pt,twocolumn,letterpaper]{article}

\usepackage{cvpr}      % To produce the REVIEW version
\definecolor{cvprblue}{rgb}{0.21,0.49,0.74}
\usepackage[pagebackref,breaklinks,colorlinks,allcolors=cvprblue]{hyperref}
\usepackage{multicol}
\usepackage{multirow} 
\usepackage{mathastext}
\usepackage{amsthm} 
\usepackage{textgreek}
\usepackage{algorithm}
\usepackage{algpseudocode}
\usepackage{amsmath, amssymb}
\usepackage{bm}
\usepackage{overpic}
\usepackage{xspace}         
\usepackage{booktabs}       
\usepackage{tabularx}
\usepackage{utfsym}
\usepackage{microtype}
\usepackage{rotating}
\usepackage{pdflscape}

\usepackage[utf8]{inputenc}
\usepackage{tcolorbox}
\usepackage{lipsum} 
\usepackage{float}
\usepackage{placeins}

\definecolor{lightlightgray}{gray}{0.9}

\input{preamble}

\def\paperID{} % *** Enter the Paper ID here
\def\confName{CVPR}
\def\confYear{2026}

\title{\papertitle}

\author{
Zanxi Ruan$^{1}$ \quad
Songqun Gao$^{2}$\\\quad
Qiuyu Kong$^{1,3}$\thanks{Corresponding author.} \quad
Yiming Wang$^{4}$ \quad
Marco Cristani$^{1,5}$\\[0.5em]
$^{1}$University of Verona \quad
$^{3}$University of Trento\\
$^{2}$Sapienza University of Rome \quad
$^{4}$Fondazione Bruno Kessler \quad
$^{5}$Reykjavik University
}

\begin{document}
\maketitle
\input{sec/abstract}
\input{sec/intro}

\input{sec/related}

\input{sec/method}

%\input{sec/3bis_proof}
\input{sec/experiments}
%\input{sec/why}
\input{sec/conclusion}

\newpage
{
    \small
    \bibliographystyle{ieeenat_fullname}
    \bibliography{main_new}
}

\input{sec/X_suppl}


\end{document}

%% file: preamble.tex
\newcommand{\mc}[1]{{\color{teal}[MC: #1]}}

\newcommand{\ourmethod}{StructXLIP\xspace}
\newcommand{\ourmethodfull}{$\text{\ourmethod}_{\text{FULL}}$\xspace}

\newcommand{\ourmethodaux}{$\text{\ourmethod}_{\text{AUXL}}$\xspace}
\newcommand{\ft}{fine-tuning\xspace}
\usepackage{multirow}
\usepackage[table]{xcolor}
\usepackage{cuted}  % required by teaserfigure in some templates
\usepackage{wrapfig}
\definecolor{morandiblue}{RGB}{104,137,182}
\definecolor{morandigreen}{HTML}{CBBF9D}
\definecolor{morandipink}{HTML}{EFDAD9}

\renewcommand{\arraystretch}{1.12}
\definecolor{cvprblue}{rgb}{0.21,0.49,0.74}
\usepackage[pagebackref,breaklinks,colorlinks,allcolors=cvprblue]{hyperref}

\newtheorem{theorem}{Theorem}

\newtheorem{lemma}{Lemma}

\newcommand{\papertitle}{\spaceskip=0.18 em\mbox{StructXLIP: Enhancing Vision-language Models with Multimodal Structural Cues}}

\newwrite\titlefile
\definecolor{MySageGreen}{RGB}{188, 204, 176}

\definecolor{panelgray}{RGB}{242,244,248}
\definecolor{datasetbar}{RGB}{242, 242, 247}
\definecolor{deltabg}{RGB}{242, 242, 247}  
\definecolor{lightpink}{RGB}{251, 220, 235}
\definecolor{lightblue}{RGB}{203, 220, 235}
\definecolor{lightgreen}{RGB}{118,199,86}
\definecolor{lightred}{RGB}{255, 107, 107}
\definecolor{MyBlue}{HTML}{2271B6}
\definecolor{MyGreen}{HTML}{16AE11} \definecolor{myyellow}{HTML}{F7DB91}
\definecolor{MyRed}{HTML}{B85D89}   
\definecolor{freshgreen}{HTML}{A8DF8E}
\definecolor{lightlightgray}{gray}{0.9}
\newcommand{\gain}[1]{\textcolor{lightgreen}{$\uparrow$ \textbf{#1}}}
\newcommand{\loss}[1]{\textcolor{lightred}{$\downarrow$ \textbf{#1}}}
\newcommand{\neutral}[1]{#1}

\newcommand{\lightmidrule}{%
  \arrayrulecolor{gray!40}\midrule\arrayrulecolor{black}%
}

\usepackage{siunitx}

\newcommand{\inlineColorbox}[2]{\begingroup\setlength{\fboxsep}{1pt}\colorbox{#1}{\hspace*{2pt}\vphantom{Ay}#2\hspace*{2pt}}\endgroup}

\newcommand{\suppmat}{\textit{Supp. Mat.}}

\input{notations}

%% file: notations.tex
\newcommand{\Real}{\mathbb{R}}
\newcommand{\sharedspace}{\Real^d}%
\newcommand{\textspace}{\mathcal{L}}
\newcommand{\visualspace}{\mathcal{V}} 
\newcommand{\image}{I}
\newcommand{\desc}{T}
\newcommand{\imagefeat}{\mathbf{i}}
\newcommand{\textfeat}{\mathbf{t}}
\newcommand{\structimagefeat}{\mathbf{i'}}
\newcommand{\structtextfeat}{\mathbf{t'}}
\newcommand{\prompfiltering}{P}
\newcommand{\imagestruct}{I'}
\newcommand{\textstruct}{T'}
\newcommand{\tempstruct}{\tau'}
\newcommand{\vocappearance}{\mathcal{V}^{a}}
\newcommand{\structextractor}{\mathcal{E}}
\newcommand{\originalloss}{\mathcal{L}_{I,T}}
\newcommand{\structloss}{\mathcal{L}_{I',T'}}
\newcommand{\structlosslocal}{\mathcal{L}_{I',T'}^{local}}
\newcommand{\consistencyloss}{\mathcal{L}_{I,I'}}
\newcommand{\ourloss}{\mathcal{L}^{*}}
\newcommand{\visualencoder}{f_{\mathtt{img}}} % Visual encoder
\newcommand{\textencoder}{f_{\mathtt{txt}}}   % Text encoder
\newcommand{\structsimilarity}{s'}
\newcommand{\imageregion}{r}
\newcommand{\textchunk}{c}
\newcommand{\imageregionset}{\mathcal{R}}
\newcommand{\textchunkset}{\mathcal{C}}
\newcommand{\maskset}{\mathcal{K}}

\newcommand{\mutualinfo}{I_{\text{MI}}}

%% file: sec/abstract.tex
\begin{abstract}
Edge-based representations are fundamental cues for visual understanding, a principle rooted in early vision research and still central today. 
We extend this principle to vision-language alignment, showing that isolating and aligning structural cues across modalities can greatly benefit \ft on long, detail-rich captions, with a specific focus on improving cross modal retrieval.
We introduce \ourmethod,
a \ft alignment paradigm that extracts edge maps (e.g., Canny), treating them as proxies for the visual structure of an image, and filters the corresponding captions to emphasize structural cues, making them ``structure-centric''.
Fine-tuning augments the standard alignment loss with three structure-centric losses: (i) aligning edge maps with structural text, (ii) matching local edge regions to textual chunks, and (iii) connecting edge maps to color images to prevent representation drift.
From a theoretical standpoint, while standard CLIP maximizes the mutual information between visual and textual embeddings,
\ourmethod~additionally maximizes the mutual information between multimodal structural representations.
This auxiliary optimization is intrinsically harder, guiding the model toward more robust and semantically stable minima, enhancing vision-language alignment.
Beyond outperforming current competitors on cross-modal retrieval on both general and specialized domains, our method serves as a general boosting recipe that can be integrated into future approaches in a plug-and-play manner. Code and pretrained models are publicly available at: 
\url{https://github.com/intelligolabs/StructXLIP}.
\end{abstract}

%% file: sec/intro.tex
\section{Introduction}
\label{sec:intro}
Vision-language models (VLMs)~\cite{radford2021learning,zhai2023siglip}
have become foundational ingredients in modern computer vision, thanks to their aligned representations across visual and textual modalities.
While the pre-trained VLMs, such as CLIP~\cite{radford2021learning}, have demonstrated impressive generalization capability, \ft is often required to be competitive on downstream tasks in specific domains~\cite{zhang2024longclip,he2025double, xu2025fate,noori2025test,zhai2022lit,zhou2022coop,zhou2022cocoop,khattak2023maple,zhu2023prograd,zhang2022tipadapter,gao2023clipadapter,hu2022lora,liu2024dora,yang2024mma,guo2025mmrl,jiang2025ucdr}. 
%has become the primary means of interacting with and adapting these models to specific domains.
\begin{figure}[t!]
    \centering
    \includegraphics[width=0.95\linewidth]{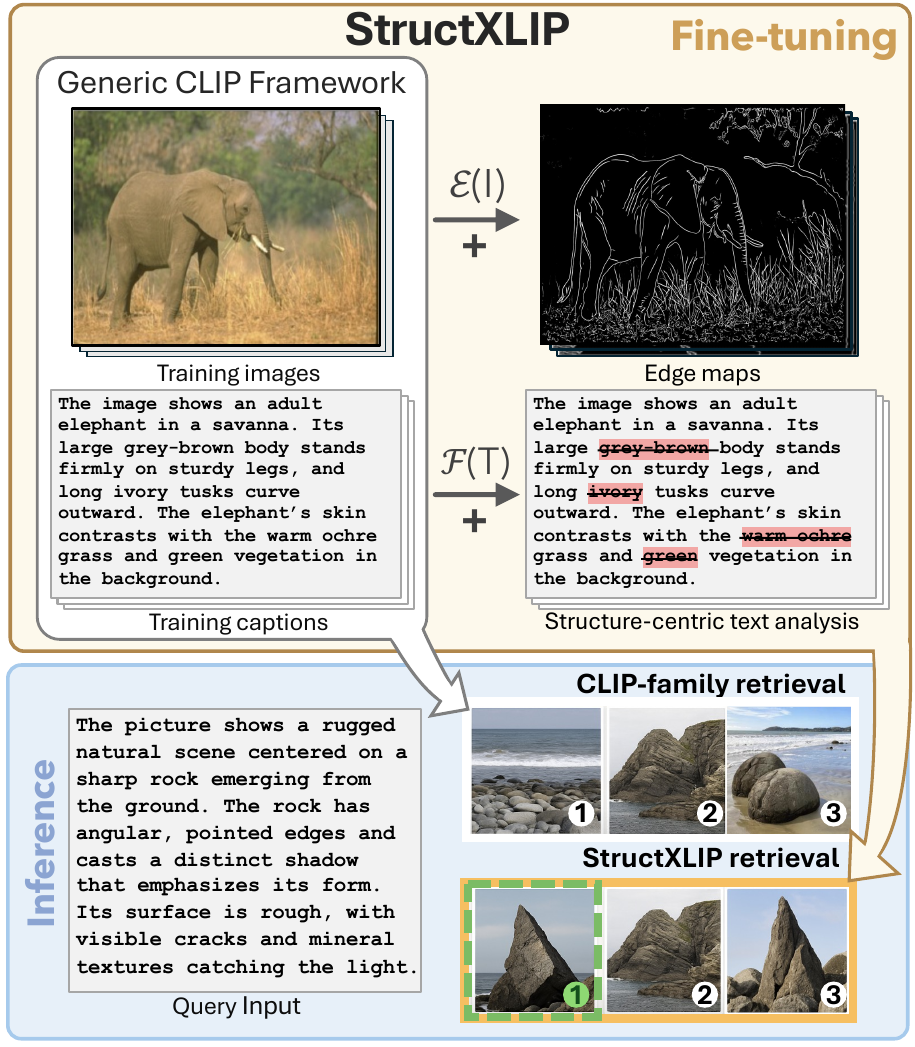}
    \caption{\ourmethod~performs \ft by adding to the standard image-text alignment with multimodal structural cue: edge maps and structure-centric captions. \ourmethod~consistently improves downstream vision–language retrieval inference.}
    %\caption{\ourmethod~can be seamlessly integrated into any CLIP-family model. It operates by generating edge maps using standard edge detectors and by refining captions to remove non-structural information such as color or material terms. This structure-centric fine-tuning consistently improves downstream vision–language retrieval performance.}
    \label{fig:example}
    \vspace{-10pt}
\end{figure}
Yet, \ft VLMs remains challenging when images contain rich visual structures and captions are long and semantically dense.
Pre-trained VLMs~\cite{radford2021learning,zhai2023siglip} are limited in such scenarios due to the token-length constraint~\cite{zhang2024longclip,goal2025,smartclip2025}. To overcome this limitation, Long-CLIP~\cite{zhang2024longclip} extends CLIP through positional interpolation to accommodate longer captions. More recent advances, such as GOAL~\cite{goal2025} and SmartCLIP~\cite{smartclip2025}, improve the efficiency of long-text alignment by reinforcing global-local correspondences and modular token-level interactions.
%enabling robust fine-tuning even with compact datasets and modest computational budgets.
Yet, these approaches rely on the semantic patterns learned from pre-training corpora, and their generalization can degrade in domains where visual or linguistic cues are scarce or poorly represented.

Departing from such dependency on
%model- or data-dependent 
semantic alignment, we introduce \ourmethod, a vision-language \ft framework that shifts the focus toward a more fundamental and accessible signal: the geometrical structure of the image, expressed through its \emph{edges}. In decades of computer vision research, edges have long been recognized as essential primitives for visual understanding~\cite{marr1980theory,survey2002edge}, encoding object boundaries and spatial layout.
Building on this foundation, \ourmethod extracts edge maps to capture visual structures (see Fig.~\ref{fig:example}). Crucially, for the first time, we extend this structural perspective to language counterpart. We filter structure-centric captions by removing tokens tied to appearance-driven cues as color and material, aiming to retain only linguistic elements that feature shape, geometry, and spatial relations. We pair edge maps with such structure-centric text, enforcing multimodal alignment on the basis of structural cues rather than appearance. The alignment process supplements existing text-image alignment
%a Long-CLIP-derived contrastive loss
with three structure-centric losses. The first aligns edge maps with the filtered text, the second establishes fine-grained correspondences between local edge regions and chunks of the filtered text, and the third links edge maps to the original color images to prevent representational drift. After fine-tuning, \ourmethod yields encoders that process simple images and text and produce embeddings enriched with structure-centered information, resulting in stronger alignment. Importantly, inference requires no additional processing and introduces no extra computational overhead. 

%From a theoretical standpoint, \ourmethod can be viewed as maximizing the mutual information between the visual and textual channels, approximated by the InfoNCE loss~\cite{math_lower_bound,poole2019variational}. \ourmethod extends this objective by introducing an additional maximization between the structure-centric multimodal representations. This objective serves as a correlated, information-reduced task which introduces gradient diversity~\cite{gradientSGDvariantNIPS20} without conflicting with the main optimization.
From a theoretical standpoint, \ourmethod can be interpreted through an information-theoretic lens. Standard alignment methods~\cite{radford2021learning,zhang2024longclip} minimize the InfoNCE contrastive loss, which corresponds to maximizing a mutual information lower bound between images and text~\cite{math_lower_bound,poole2019variational}. \ourmethod extends this objective by adding a second maximization composite term focused on structure-centric multimodal representations. This auxiliary objective acts as a correlated, information-reduced task that introduces gradient diversity~\cite{gradientSGDvariantNIPS20} without interfering with the main optimization.
In fact, information-theoretic analyses of multitask learning~\cite{MTLasMOO.NIPS18} show that auxiliary losses with lower mutual information act as an implicit regularizer, expanding the effective search space and improving convergence. 
\ourmethod sets new state-of-the-art cross-modal retrieval results on both general-purpose dense-caption benchmarks and domain-specific datasets dominated by fine visual details, such as fashion and fine-grained classification. Its three structure-centric losses also demonstrate potential as a ``universal'' booster: when plugged into existing fine-tuning frameworks, performance improves across the board while maintaining competitive generalization. Ablations show that any edge extractor is effective, with classical operators like Canny and LoG performing the best. Our explainability analyses also confirm that the model consistently attends to human-interpretable, shape-driven regions. %Mutual-information analyses further validate our theoretical claims by revealing more stable and informative optimization dynamics. 
Notably, \ourmethod excels even in low-data regimes, benefiting from a strong structure-centric inductive bias that enables robust learning with limited data and supervision.

%Contribution summary
%\noindent \textbf{Our main contributions} are summarized as follows:
%\begin{itemize}
%    \item We demonstrate that incorporating multimodal, structure-centric cues into standard contrastive learning significantly enhances vision–language alignment.
%    \item Our \ourmethod~is plug-and-play and effective, outperforming recent vision–language finetuning methods without introducing architectural complexity.
 %   \item Our theoretical motivation based on mutual information and optimization analysis, shows how structure-centric alignment enhances stability and convergence.
 %   \item Our extensive analysis demonstrates that \ourmethod achieves strong generalization and interpretability, remaining robust to the edge extraction choice.
%\end{itemize}

\noindent\textbf{Our main contributions} are as follows:
\begin{itemize}
\item We show that injecting structure-centric, multimodal cues into contrastive learning substantially improves long-text vision-language alignment.
\item \ourmethod outperforms recent fine-tuning methods without adding architectural complexity.
\item Our mutual-information and optimization analysis explains why structure-centric alignment stabilizes fine-tuning and accelerates convergence.
\item Our structure-centric alignment losses serve as a plug-and-play booster to diverse \ft frameworks, while maintaining competitive generalization.
\item Extensive experiments confirm the effectiveness of our loss designs, and the robustness to different edge/textual extraction choices and qualities.
\end{itemize}

%[MARCO this below can help anchoring our idea to the literature]
% \yiming{need to inject it somewhere}Recent studies show that introducing structural priors early in the learning pipeline can guide multimodal alignment without replacing the raw input. 
% Disentangled or geometry-aware representations have been used to improve compositional reasoning and retrieval performance \cite{li2024deal,zhou2022vdr,jin2023dicosa}.
%This unified treatment assumes that both modalities encode comparable semantic content, overlooking the fact that images and texts inherently capture different kinds of information. 

%Recent works such as DEAL~\cite{li2024deal}, DisCoCLIP~\cite{discoclip2025}, TaSe~\cite{tase2025}, and SmartCLIP~\cite{smartclip2025} have started to explore disentanglement in vision--language alignment, mainly by decomposing textual semantics or modularizing alignment mechanisms. 
%Yet these approaches still treat the visual and textual streams as if they carried the same type of information.

%% file: sec/related.tex
\section{Related Works}
\label{sec:sota}

\noindent\textbf{Vision-language alignment.}
Contrastive learning on image-text pairs remains the dominant paradigm for VLMs, with two main loss formulations: CLIP’s InfoNCE objective~\cite{radford2021learning} and SigLIP/SigLIP2’s pairwise sigmoid matching~\cite{zhai2023siglip,zhai2023siglip2}. Subsequent dual-encoder variants mostly improve this framework through debiasing, adaptive weighting, or geometric regularization (DeCLIP~\cite{li2021declip}, ALIP~\cite{dong2023alip}, GOAL~\cite{goal2025}, SmartCLIP~\cite{smartclip2025}, CoAPT~\cite{wanglearning}, AVSE~\cite{liu2025asymmetric}). 
Architectural scaling (EVA-CLIP~\cite{fang2023eva}, Florence-2~\cite{yuan2024florence}) and text-capacity extensions such as Long-CLIP~\cite{zhang2024longclip} to enable long-text handling.
Fine-tuning strategies span full-model adaptation to parameter-efficient tuning: Double-Filter, FATE, MLMP~\cite{he2025double,xu2025fate,noori2025test}, text-only tuning (LiT~\cite{zhai2022lit}), prompt learning (CoOp, CoCoOp, MaPLe, ProGrad~\cite{zhou2022coop,zhou2022cocoop,khattak2023maple,zhu2023prograd}), and lightweight adapters or low-rank modules (CLIP-Adapter, Tip-Adapter, LoRA, DoRA, Mma, Mmrl, UCDR-Adapter~\cite{zhang2022tipadapter,gao2023clipadapter,hu2022lora,liu2024dora,yang2024mma,guo2025mmrl,jiang2025ucdr}). These methods retain the standard contrastive setup and mainly enhance alignment efficiency or generalization.

\noindent\textbf{VLMs for long-text alignment.}
Despite recent advances in visual-textual alignment, most VLMs remain constrained by the $\sim$77-token limit of CLIP-style text encoders, hindering their ability to model rich detail descriptions. This has motivated a line of methods targeting long-text alignment in VLMs. Position-encoding extensions (Long-CLIP~\cite{zhang2024long}, TULIP~\cite{najdenkoska2024tulip}, V2PE~\cite{ge2025v2pe}, FineLIP~\cite{asokan2025finelip}) expand effective context length while preserving pretrained semantics. Caption decomposition approaches (DreamLIP~\cite{zheng2024dreamlip}, longAlign~\cite{liu2024improving}) align sub-captions with local regions for finer grounding. Dataset- or structure-driven strategies introduce long-text corpora and learnable structural tokens (LoTLIP~\cite{wu2024lotlip}) or region-level masked captioning~\cite{urbanek2024picture}. LLM-assisted pipelines such as MATE~\cite{jang2024mate} further project long texts into the VLM embedding space through multi-stage transformation.  
These methods operate on CLIP-based foundations and differ mainly in how rich semantics can be better \emph{aligned} via \ft among both modalities. 
Differently, \ourmethod augments standard contrastive fine-tuning with structural cues, serving as general booster to further enhance vision-language alignment.

%% file: sec/method.tex
\section{\ourmethod}

\ourmethod performs \ft in two stages (Fig.~\ref{fig:overview}). In the \textbf{structure-centric multimodal extraction} stage, it produces complementary structural views by generating edge maps and filtering captions to retain object layout and geometric cues by removing appearance terms. In the \textbf{structure-centric alignment} stage, these structural representations are aligned at both global and local levels together with the original image-text pairs, using dedicated losses to strengthen alignment under semantically rich conditions. The structure-centric signals are used only during fine-tuning. 
Inference solely involves original images and texts, thus no additional cost compared to CLIP, and no need for edge extraction nor structure-centric text filtering. 
%\ourmethod learns unified image and text projectors, so downstream evaluation uses only the original inputs and runs at the same inference cost as CLIP.
%
\subsection{Structure-centric multimodal extraction}
\label{Structural_extraction}
%We first extract structure-centric cues from both visual and textual modalities. 
Let us consider a color image $\image \in \visualspace$ which is visually described by a textual description $\desc \in \textspace$. Each image $\image_i$ and its visual description are paired $(\image_i,~\desc_i)$. 

\noindent\textbf{Visual extraction.}
For each image $\image_i$, we extract its structural view as $\imagestruct_i = \structextractor(\image_i)$, where $\mathcal{E}(\cdot)$ denotes any filter-based~\cite{canny2009computational,marr1980theory} or learning-based~\cite{xie2015holistically,zhang2023adding,liu2021learn} edge detector. \ourmethod is edge-extractor-agnostic as shown in Sec.~\ref{sec:exp}.
\begin{figure}[!t]
    \centering
    \includegraphics[width=1.0\linewidth]{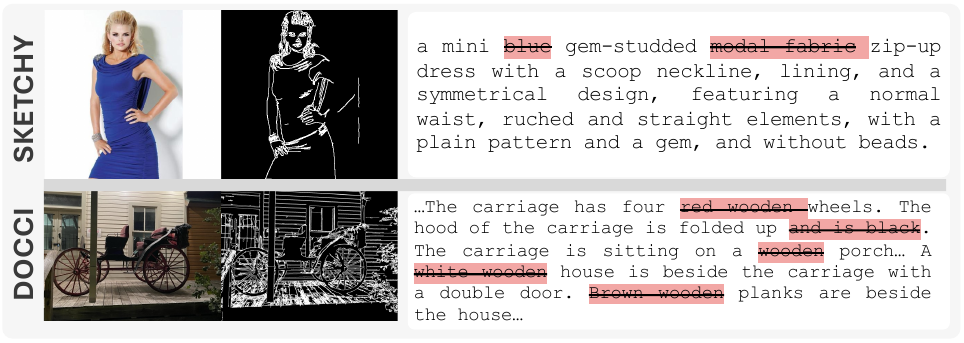} \vspace{-0.5cm}
    \caption{Examples of Visual (left) and Textual (right) extraction. %on the left, visual edges are extracted using Canny~\cite{gonzalez_digital_2018}, while on the right, words related to color and material are removed from the text without compromising its linguistic naturalness.
    }
    \label{fig:edge}
    \vspace{-4pt}
\end{figure}

%in such as learning-based edge detectors or sketch-based generation approaches.

\noindent\textbf{Textual extraction.}
Complementary to the edge-based visual cues, we extract a structure-centric textual description that emphasizes layout, geometric attributes, and spatial relations, without appearance-related terms.
%This choice is motivated by evidence that appearance-based language (specifically, color and material) often fails to align with visual discriminability, whereas spatial and shape-related semantics better capture grounded meaning in vision–language tasks~\cite{xu2023vgse}.
While directly prompting the LLM to perform rewriting of the original $\desc_i$ seems a straightforward solution, it is fragile to hallucinations~\cite{talon2025seeing}. Alternatively, we obtain the structure-centric textual description by \textit{lexicon filtering} to remove appearance terms, specifically, color and materials, from the original captions.
We leverage the strong semantic understanding capability of a large language model (LLM), prompting it to construct a set of appearance-related vocabulary $\vocappearance$ to be filtered from the original description.
%We opt not to directly prompt the LLM to perform rewriting of the original $\desc_i$ to remove appearance-related terms, as such rewriting process is difficult to control in terms of hallucinations or semantic drift. 
Specifically, 
%\begin{equation}
%\vocappearance = \mathcal{LLM}(\prompfiltering),
%\end{equation}
%where $\prompfiltering$ is the prompt that queries the LLM to produce appearance-centric semantic concepts that cannot be directly inferred from the structure-centric visual content $\imagestruct$. The resulted vocabulary $\vocappearance$ broadly covers colors, materials and textures. 
%\begin{equation}
$\vocappearance = \mathcal{LLM}(\prompfiltering)$, 
%\end{equation}
where $\prompfiltering$ is the prompt that queries the LLM to produce semantic concepts related to colors and materials, resulting in the vocabulary $\vocappearance$. %The resulted vocabulary $\vocappearance$ broadly covers colors, materials and textures.
We then obtain the structure-centric text $\textstruct_i$ by filtering out terms that are inside $\vocappearance$ from $\desc_i$:  
%\begin{equation}
$\textstruct_i = \mathcal{F}(\desc_i, \vocappearance)$,
%\end{equation}
where $\mathcal{F}(\cdot)$ is the filtering function based on regular expression matching to remove any semantic concept are within $\vocappearance$, we refer to this process as the \emph{Lexicon Filter}.  Please refer to \suppmat~for the prompt $\prompfiltering$ and vocabulary $\vocappearance$ used in the filtering processes. Fig.~\ref{fig:edge} illustrates examples of the extracted edge maps and the corresponding structure-centric captions.
%The mechanism is annotation-free and transferable to diverse domains.

\begin{figure*}[!ht]
    \centering
    \includegraphics[width=0.8\linewidth]{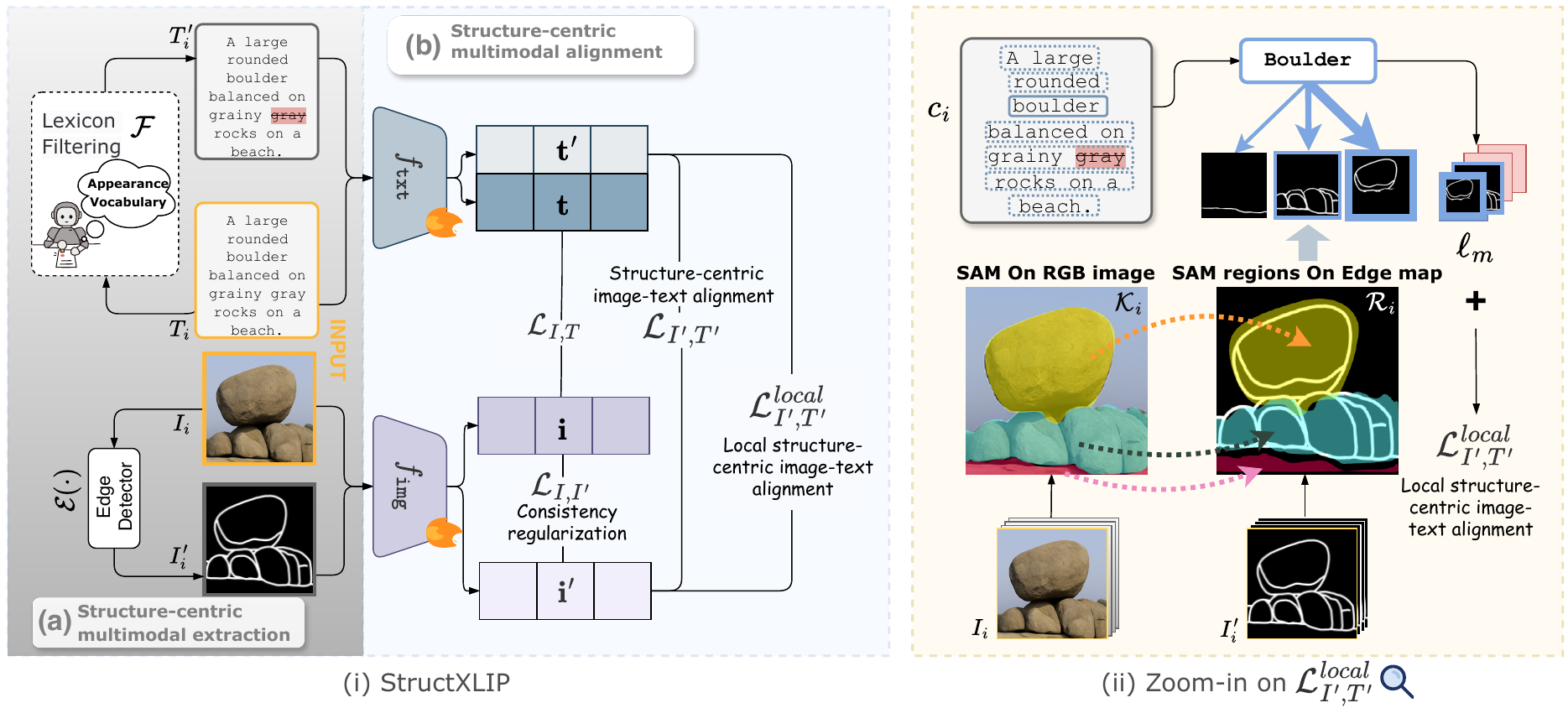}
    \vspace{-5pt}
    \caption{(i) Overview of \ourmethod fine-tuning which operates in two stages. The first \textbf{structure-centric multimodal extraction} stage extracts structural views by generating edge maps from original images via edge detector and performing lexicon filtering on original captions to remove appearance-related terms. The second \textbf{structure-centric multimodal alignment} stage fine-tunes the encoders $\visualencoder$ and $\textencoder$ by joining original image-text alignment $\originalloss$ with our newly introduced structure-centric alignment objectives: the \emph{structure-centric image-text alignment loss} $\structloss$ enforces global structural alignment, the \emph{Local structure-centric image-text alignment loss} $\structlosslocal$ captures local compositional semantics, and the \emph{Consistency regularization loss} $\consistencyloss$ aligns raw images and edge maps to prevent representation drift. At inference, only color image ($\image$) and captions ($\desc$) are input to the fine-tuned $\visualencoder$ and $\textencoder$, neither edge extraction nor lexicon filtering is required. (ii) Zoom-in on the Local structure-centric image-text alignment loss $\structlosslocal$.}
    %from the left, a generic type of CLIP could be adopted. In this paper, we adopt a Long-CLIP-derived contrastive loss. The three dots in parenthesis indicate this freedom of choice. LOSS 1 considers the embedding of the edge maps $\{I'_i\}$ and of the structure-centric textual descriptions $\{T'_i\}$, pairing them as an ordinary CLIP alignment. LOSS 2 introduces fine-grained cross-modal analysis associating every text chunk embedding with the $K$-most related edge regions. Finally the LOSS 3 is a general regularizer which enforce visual and structural visual embeddings to stay close in the embedding space.}
    \label{fig:overview}
    \vspace{-3pt}
\end{figure*}

\subsection{Structure-centric multimodal alignment}\label{Sec:alignment}
Let us consider a VLM $f_\mathtt{VLM}$ with two main elements: the visual encoder $\visualencoder$ and the text encoder $\textencoder$.
As shown in~\cref{fig:overview}, the encoders map their respective inputs into a shared $d$-dimensional embedding space, \ie, $\visualencoder:\visualspace \rightarrow \sharedspace$ and $\textencoder:\textspace \rightarrow \sharedspace$. The resulting features are $\imagefeat_i=\visualencoder(\image_i)$ and $\textfeat_i=\textencoder(\desc_i)$, with structure-centric counterparts $\structimagefeat_i$ and $\structtextfeat_i$.
The original image-text alignment loss $\originalloss$, serving as the alignment objective basis, is expressed as a symmetric InfoNCE loss~\cite{math_lower_bound} with cosine similarity as similarity measure.
In addition, we introduce three auxiliary learning objectives aligns structure-centric multimodal contents at multiple levels: the \emph{Structure-centric image-text alignment loss} $\structloss$ enforces global structural alignment, the \emph{Local structure-centric image-text alignment loss} $\structlosslocal$ captures part-level compositional semantics, and the \emph{Consistency regularization loss} $\consistencyloss$ aligns raw images and structure-centric views to prevent modality drift.
Note that our auxiliary structure-centric losses, could be flexibly integrated to any finetuning framework, as we show in Sec.~\ref{Sec:main.comparisons}.
%To maintain compatibility with the pre-trained VLM space, we keep the original image-text alignment loss $\originalloss$~\yiming{cite CLIP?}, preserving high-level semantic consistency.
%The three auxiliary learning objectives which align structure-centric visual and textual contents at multiple levels are: i) the primary \emph{Structure-centric image-text alignment loss} $\structloss$, featuring the structural-centric image-text alignment to promote the learning of structurally-aligned representation across both modalities; ii) the \emph{Local structure-centric image-text alignment loss}  $\structlosslocal$, which captures compositional semantics between local parts in both modalities; 
%iii) finally, the \emph{Consistency regularization loss} $\consistencyloss$ between the raw color image and the extracted structure-centric visual content, preventing modality drift.  
%Together, these objectives progressively refine the model’s ability to ground structural semantics with increasing granularity.

\noindent\textbf{Structure-centric image-text alignment $\structloss$.}
Using the structure-oriented visual and textual views, we introduce a contrastive objective that aligns $\imagestruct$ and $\textstruct$. We compute their cosine similarity with a dedicated temperature $\tempstruct$, and apply a symmetric InfoNCE loss over the batch. This loss has the same form as the standard image-text contrastive loss, but operates purely on structural cues to enforce multimodal alignment between the structure-preserving representations.
\iffalse
\noindent\textbf{Structure-centric image-text alignment $\structloss$.}
Building upon the extracted structure-centric visual and textual content,
we introduce a structural contrastive objective 
to explicitly align the representation between $\imagestruct$ with $\textstruct$. 
We quantify the similarity $\structsimilarity$ between  $\visualencoder(\imagestruct)$ and $\textencoder(\textstruct)$ as:
\begin{equation}
\structsimilarity = \tempstruct<\visualencoder(\imagestruct),~\textencoder(\textstruct)>,
\end{equation}
where $\tempstruct$ is a separate learnable temperature dedicated to the structure-centric image-text alignment, and $<\cdot,~\cdot>$ is the operator computing the cosine similarity of two vectors.
We aim to minimize the structure-centric cross-modal similarity via the symmetric InfoNCE loss, expressed as:
\begin{equation}
% \resizebox{0.99\linewidth}{!}{$
\structloss 
= \frac{1}{2N} \sum_{i=1}^{N}
\left(
-\log \frac{\exp(\structsimilarity_{i,i})}{\sum_{j=1}^{N}\exp(\structsimilarity_{i,j})}
-\log \frac{\exp(\structsimilarity_{i,i})}{\sum_{j=1}^{N}\exp(\structsimilarity_{j,i})}
\right),
% $}
\label{eq:structloss}
\end{equation}
where $N$ is the batch size. $\structloss$ follows the same loss expression as the original image-text alignment loss $\originalloss$, but focuses on capturing  the semantic alignment on the structural semantics on both modalities.
%$\mathcal{L}_{\text{ET}}$ focuses on the structural correspondence between modalities, 
%thereby decoupling geometric reasoning from appearance modeling.
\fi

\noindent\textbf{Local structure-centric image-text alignment $\structlosslocal$.} 
In addition to $\structloss$, which features the global alignment between between $\imagestruct$ with $\textstruct$, we further introduce a local-level alignment to capture the fine-grained correspondences of atomic structural semantic, as depicted in Fig.~\ref{fig:overview}. 
Inspire by Local Image–Sentence Matching (LISM)~\cite{goal2025}, we first segment both the structure-centric image $\imagestruct$ and the text $\textstruct$, where the segments from both modalities are then aligned during fine-tuning. 
Specifically, we apply the Segment Anything Model (SAM)~\cite{kirillov2023segment} on the original color image $\image_i$ to generate a set of visually and semantically coherent masks $\maskset_i$. With these visual masks, we then obtain the set of local structure-centric visual regions $\imageregionset_i$. On the textual side, we segment the textual counterpart $\textstruct$ into multiple phrases based on sentence delimiters (\eg, periods or semicolons), forming a set of text chunks $\textchunkset_i$ of size $M_i$ that varies based on each textual description. We employ a multi-positive contrastive learning strategy to jointly model the one-to-many correspondence between each local phrase $\textchunk_{m}\in \textchunkset_i$ and a total of $K$ structure-centric visual regions that are most semantically aligned $\imageregionset_m^K$~based on the pre-trained VLM encoders. %For each local phrase $\textchunk_{i,n}$, we compute cross-modal similarity between and all candidate regions using the basic CLIP model, and retain the top-$K$ most relevant structural embeddings as the local positive set:
% \begin{equation}
%     \mathcal{P}_i = \{\mathbf{v}_{i,k}\}_{k=1}^{K}.
% \end{equation}
%To achieve alignment within the structural modality, we crop the corresponding regions in the structural image $I^{\text{str}}$ based on the coordinates of the selected regions in the RGB image, thus obtaining matched structural-text local pairs. 
For each text chunk $\textchunk_{m}$ of structure-centric image $\imagestruct_i$, the local alignment loss $\ell_{m}$ is:
\begin{equation}
    \ell_{m}
    = -\log
    \frac{
        \sum\limits_{\imageregion_k \in \imageregionset_n^K}
        \exp~\bigl(\gamma\, \textencoder(\textchunk_{m})\cdot \visualencoder(\imageregion_k)\bigr)
    }{
        \sum\limits_{\imageregion_j \in \imageregionset^{B}}
        \exp~\bigl(\gamma\, \textencoder(\textchunk_{m})\cdot \visualencoder(\imageregion_j)\bigr)
    },
    \label{eq:local_loss}
\end{equation}
where $\gamma$ is the temperature and $\imageregionset^{B}$ represents the set of visual local regions within a batch.

The final loss $\structlosslocal$ is computed by averaging $\ell$ over the visual regions of all samples in the batch of size $N$:
%Finally, to obtain the overall local alignment objective we average the loss across all valid samples in the batch to obtain the overall local alignment objective:
\begin{equation}
\structlosslocal =
\frac{1}{N}
\sum_{i=1}^{N}
\frac{1}{M_i}
\sum_{m=1}^{M_i}
\ell_{m}.
\label{eq:chunkloss}
\end{equation}
$\structlosslocal$ does not require additional annotations, and its multi-positive formulation allows each text phrase to align with multiple structure-centric visual regions simultaneously, enhancing the model’s capacity to capture compositional semantics
and multi-part descriptions.

\noindent\textbf{Consistency regularization $\consistencyloss$.}
While the structure-centric alignment provides additional supervision, its representation may gradually deviate from the original semantic manifold during the VLM fine-tuning. To mitigate this issue, we introduce the consistency regularization loss $\consistencyloss$, using the visual representation of the original image to anchor the updated representation within the pre-trained vision-language latent space.
$\consistencyloss$ in the N-batch is computed as:
\begin{equation}
\consistencyloss =
\frac{1}{N} \sum_{i=1}^{N}
\left(
1 - <\visualencoder(\image_i),~\visualencoder(\imagestruct_i)>
% \frac{
% \visualencoder(\image_i) \cdot \visualencoder(\imagestruct_i)
% }{|\visualencoder(\image_i)|_2 , |\visualencoder(\imagestruct_i)|_2
% }
\right).
\end{equation}
$\consistencyloss$ serves as a regularization to ensure fine-tuning stability and semantic consistency.

\noindent\textbf{Full Learning Objective.}
%\ourmethod is designed as a structure-aware finetuning framework 
%built upon pretrained vision-language models, 
%targeting more stable and precise alignment under complex and long-text descriptions.
Our overall structure-centric losses $\ourloss$ can be combined in a weighted manner: 
\begin{equation}
    \ourloss
%= \lambda_{1}\originalloss
= \lambda_{1}\structloss
+ \lambda_{2}\consistencyloss
+ \lambda_{3}\structlosslocal,
\end{equation}
where $\lambda_{1\text{--}3}$ are weight coefficients that are empirically set. 
The full objective can then be expressed by $\mathcal{L}_{\text{total}} 
%= \lambda_{1}\originalloss
= \originalloss
+ \ourloss
$.
Note that $\ourloss$ are axillary and additional, which can be easily integrated into any fine-tuning framework and losses, substituting $\originalloss$.
% $\lambda_{1}=1.0$, $\lambda_{2}=0.25$, $\lambda_{3}=0.1$, and $\lambda_{4}=0.1$.
Through this fine-tuning objective, 
\ourmethod can effectively enhance the multimodal alignment in appearance and structural semantics, especially for long texts with rich semantic details.
\subsection{Theoretical view}
\label{sec:theory}
%We consider \ourmethod on visual-language alignment from an information-theoretic perspective~\cite{poole2019variational,wu2020mutualinformationcontrastivelearning}, following common practice in contrastive representation learning, supported by a numerical analysis conducted on a real \ft experiment (Sec.~\ref{sec:exp} Tab.~\ref{tab:main_results}).
We analyze \ourmethod from an information-theoretic perspective~\cite{poole2019variational,wu2020mutualinformationcontrastivelearning}, following common practice in contrastive representation learning. Our analysis is supported by numerical simulations performed across all four experimental datasets (see Sec.~\ref{sec:exp}), multiple training data regimes (Sec.~\ref{sec:exp}, \emph{Sample efficiency experiment}), and considering \ourmethod aa s plug-in to diverse alignment strategies (\emph{Finetuning-agnostic improvement} experiment). Here we report one representative study (Sec.~\ref{sec:exp} Tab.~\ref{tab:main_results}, \emph{Sketchy} dataset), with the remaining experiments provided in the \suppmat

The overall objective $\mathcal{L}_{\text{total}}$ %in Eq.~\ref{eq:total.loss} 
maximizes the mutual information $\mutualinfo(\image,~\desc)$ between the image representation $\image$ and its corresponding text $\desc$. Our analysis focuses on the three losses $\originalloss$, $\structloss$ and $\consistencyloss$, since the local loss $\structlosslocal$ can be seen as a specification of $\structloss$.
\begin{figure}[h!]
    \centering
    \includegraphics[width=1.0\linewidth]{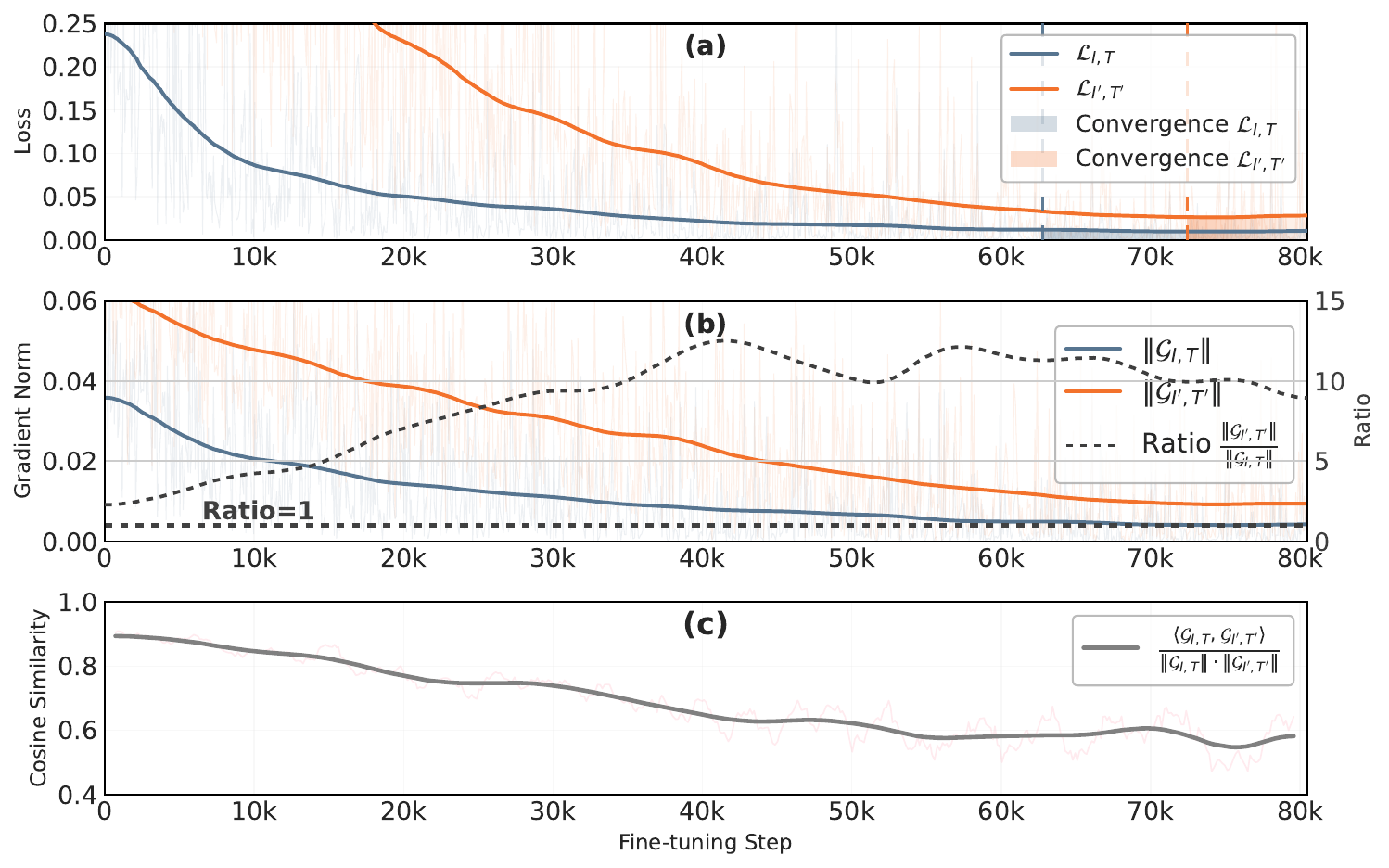}
    \caption{Empirical analysis confirming the information-theoretic view of \ourmethod. a) Loss convergence; b) Gradient norms and ratio evolution; c) Inter-task cosine similarity.}
    \label{fig:theory.exp}
    \vspace{-4pt}
\end{figure}
Let $\structextractor(\cdot)$ and $\mathcal{F}(\cdot)$ be deterministic mappings that produce information-reduced counterparts
$\imagestruct=\structextractor(\image)$ and $\textstruct = \mathcal{F}(\desc)$, respectively. In practice, $\imagestruct$ is the edge map contains binary ${0,1}$ values instead of RGB triplets, and $\textstruct$ is stripped of color and material terms. %Still,  preserving structure while discarding appearance cues.
By the Data Processing Inequality~\cite{tishby99information}, 
the compressed representations satisfy
$\mutualinfo(\imagestruct,~\textstruct) \le \mutualinfo(\image,~\desc)$.\footnote{We also verify numerically by KSG mutual information neural estimator~\cite{kraskov2004estimating} that $\mutualinfo(\imagestruct,~\textstruct)\simeq0.20$; $ \mutualinfo(\image,~\desc)\simeq 0.52$  on the DOCCI dataset~\cite{docci2024} in the experiment of Tab.~\ref{tab:main_results}.}
As commonly done in information-theoretic representation learning, we approximate the mutual information quantities using their InfoNCE lower bounds~\cite{poole2019variational,wu2020mutualinformationcontrastivelearning}: $\mutualinfo(\image,\desc) \ge \log N - \mathbb{E}[\originalloss]$ and
$\mutualinfo(\imagestruct,\textstruct) \ge \log N - \mathbb{E}[\structloss]$,
where $N$ is the batch size.
Fig.~\ref{fig:theory.exp}~a) shows the two losses $\originalloss$ and $\structloss$. We observe that the gradients of the two losses,
$\nabla_\theta \originalloss$ and $\nabla_\theta \structloss$,
are positively correlated in their direction (Fig.\ref{fig:theory.exp}~c)), considering cosine similarity ($\mu=0.6872, \sigma=0.1028$, min$=0.57$, max$=0.89$), indicating that the two optimization objectives
pursue compatible optima in the parameter space. Empirically, we observed that the consistency loss $\consistencyloss$ strengthens this correlation, while disabling it does not significantly reduce the correlation, indicating that the two problems are inherently related.

However, since $\imagestruct$ and $\textstruct$ encode less information, the auxiliary maximization $\mutualinfo(\imagestruct,\textstruct)$
is inherently more difficult, leading to a later convergence of $\structloss$, as evidenced from Fig.\ref{fig:theory.exp}~a) where we show the iteration after which we measured convergence of the two losses, and Fig.\ref{fig:theory.exp}~b) where we plot the magnitude of the gradients of $\nabla_\theta \originalloss$ and $\nabla_\theta \structloss$ and their ratio. 
Convergence has been evaluated by adopting the two-regime learning-curve method of Gaussian-process theory, which defines the onset of convergence as the point where model error becomes comparable to or lower than the estimated noise level~\cite{sollich2002learning,viering2022shape}.

As a result, $\structloss$ provides a persistent and informative gradient
even when $\originalloss$ flattens, effectively steering optimization toward
semantically coherent minima.
This mechanism can be interpreted via multitask optimization
and gradient diversity.
Information-theoretic analyses of multitask learning~\cite{MTLasMOO.NIPS18}, multi-objective optimization theory~\cite{gradientSGDvariantNIPS20} and contrastive representation learning~\cite{wu2020mutualinformationcontrastivelearning}
show that auxiliary objectives with lower mutual information act as implicit
regularizers, introducing controlled gradient diversity that expands the
effective search space and improves convergence stability.
Our structure-centric auxiliary alignment $\structloss$ plays exactly this role:
it introduces a correlated, information-reduced objective that contributes
complementary gradient directions, maintaining active updates when the main
contrastive gradients approach zero.
%This interpretation is consistent with multi-objective optimization theory~\cite{gradientSGDvariantNIPS20},
%where correlated auxiliary tasks help avoid premature convergence, thus promoting better generalization.

%% file: sec/experiments.tex
% =========================
% Experiments
% =========================
\section{Experiments}
\label{sec:exp}
\iffalse
We evaluate the complete framework \ourmethodfull, trained with all four losses of $\mathcal{L}_{\text{total}}$ of Sec.~\ref{Sec:alignment}, %Eq.~\ref{eq:total.loss},
referred to simply as \ourmethod unless stated otherwise, against state-of-the-art VLM fine-tuning methods on standard cross-modal retrieval benchmarks.
Our experiments span four datasets covering both general-domain and specific-domain scenarios with rich natural-language descriptions.
We also demonstrate the plug-and-play nature of our structure-centric alignment losses \ourmethodaux (the terms $\lambda_{1}\structloss + \lambda_{2}\consistencyloss + \lambda_{3}\structlosslocal$ of the total loss $\mathcal{L}_{\text{total}}$ in Sec.~\ref{Sec:alignment})
%Eq.~\ref{eq:total.loss}) 
by integrating them into several existing fine-tuning frameworks.
Finally, we analyze several aspects of \ourmethod, including its generalization, sensitivity to training set size, and the effect of different edge detectors and loss components, and conclude with qualitative evidence and a theoretical explanation for its strong empirical performance. 
\fi 

\noindent\textbf{Datasets.} We evaluate \ourmethod on both general-domain and specific-domain benchmarks featuring long, information-rich descriptions. For the general domain, we consider DCI~\cite{urbanek2024picture} and DOCCI~\cite{docci2024}, two human-annotated dense captioning datasets with long, detailed descriptions. DCI contains 7.8k images with mask-aligned captions exceeding 1000 words on average, while DOCCI provides 15k diverse scenes paired with highly discriminative descriptions averaging 136 words.
For specific-domain evaluation, we adopt SKETCHY~\cite{girella2025lots}, a 46k multimodal fashion dataset with fine visual attributes, and Insect~\cite{truong2025insect}, a 6k fine-grained biology dataset with expert-verified morphological descriptions (81 words on average). %We additionally remove duplicates from Insect for consistency.
For the dataset splits, we follow the protocols of~\cite{goal2025} for DCI and DOCCI, use the original split for SKETCHY, and adopt an 8:2 train/test split for Insect. Further dataset details in the \suppmat

\noindent\textbf{Compared methods.}
We consider \emph{sota} CLIP finetuning methods for long-text alignment: Long-CLIP \cite{zhang2024long} extends the text encoder via positional interpolation, FineLIP \cite{asokan2025finelip} improves long-text understanding through multi-stage fine-tuning with adaptive attention fusion, SmartCLIP \cite{smartclip2025} uses dynamic token routing and progressive context extension to emphasize salient segments, and GOAL \cite{goal2025} jointly models global/local correspondences for fine-grained alignment.
%We compare our method with a set of state-of-the-art CLIP finetuning methods that are proposed for long-text alignment:
%i) \textit{LongCLIP}~\cite{zhang2024long} extends CLIP’s text encoder via positional interpolation to handle longer sequences while maintaining pre-trained alignment;
%ii) \textit{FineLIP}~\cite{asokan2025finelip} enhances long-text understanding through multi-stage fine-tuning with adaptive attention fusion;
%iii) \textit{SmartCLIP}~\cite{smartclip2025} introduces dynamic token routing and progressive context extension to selectively encode semantically salient textual segments;
%iv) \textit{GOAL}~\cite{goal2025} jointly optimizes global and local image–text correspondences via Local Image-Sentence Matching (LISM) and Token Similarity-based Learning (TSL) for fine-grained alignment of detailed captions.

\noindent\textbf{Performance metrics.}
We report Recall@K (K=1/5/10) of the cross-modal retrieval task in both \inlineColorbox{lightpink}{\emph{Text$\rightarrow$Image} (\emph{T$\rightarrow$I})} and \inlineColorbox{lightblue}{\emph{Image$\rightarrow$Text} (\emph{I$\rightarrow$T})} settings. 
Recall@K measures the proportion of queries for which the corresponding ground-truth match appears within the top-K retrieved results.
Each image is unique and paired to one textual description.

\noindent\textbf{Implementation details.}
As visual encoder $\visualencoder$ we use ViTB/16~\cite{dosovitskiy2021vit} and as text encoder $\textencoder$, we employ the Knowledge Preserving Stretching (KPS) strategy from Long-CLIP~\cite{zhang2024long} to extend CLIP’s text encoder for long-text inputs. 
%KPS enlarges the positional embedding range beyond CLIP’s original 77-token limit by reusing the first 20 well-trained embeddings and interpolating the remaining ones according to a stretching ratio. 
We report in main comparison with edges extracted by the Canny detector.
Visual inputs are resized to $224 \times 224$.
For fair comparison, all methods share the same tokenizer, batch size, and fine-tuning budget. We empirically set $\lambda_{1}=0.25$, $\lambda_{2}=0.1$, and $\lambda_{3}=0.1$, as this configuration consistently provided strong overall performance. Both models are fine-tuned for 10 epochs on an \textit{NVIDIA RTX 5090} GPU, with each run taking approximately 1–3 hours depending on the dataset. Results are averaged over 3 random seeds; we report mean values (and $\pm$std in the appendix). For GOAL~\cite{goal2025}, the results on the DOCCI and DCI datasets reported in Tab.~\ref{tab:main_results} are taken directly from the original paper. 
For reference, the per-batch runtime of our method is 0.17s. 
The underlying CLIP variants used in our experiments span a similar range, 
with Long-CLIP running at 0.10s, GOAL at 0.18s, 
SmartCLIP at 0.07s, and FineLIP at 0.05s. Inference time is equivalent across all approaches, since it only involves projecting data into the embedding space. Additional implementation details and computational analysis in the \suppmat

\input{tables/table1}

\begin{figure}
    \centering
    \includegraphics[width=1.0\linewidth]{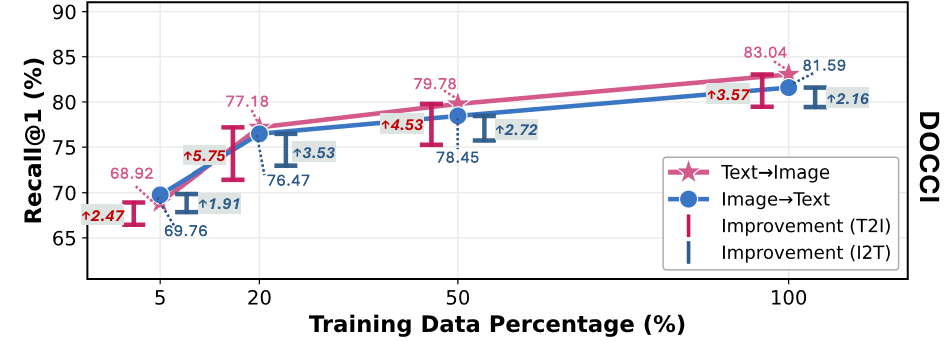}
    \caption{Data amount vs.~performance: Retrieval recall R@1 in case of 5\%, 20\%, 50\% and the standard 100\% of training data. The improvement line indicates the margin of \ourmethod against the second best approach, trained with the same amount of data.}
    \label{fig:data.amount}
    \vspace{-8pt}
\end{figure}

\subsection{Main comparisons}\label{Sec:main.comparisons}
\noindent\textbf{Best-performing on cross-modal retrieval.} As shown in Tab.~\ref{tab:main_results}, our \ourmethod~consistently achieves the best recall at every rank on both directions (T$\rightarrow$I / I$\rightarrow$T), among all compared methods, except on DCI R@5 I$\rightarrow$T.  On the SKETCHY dataset, the improvement obtained by \ourmethod~is most significant: R@1 on the increases by +6.65\% and +5.78\% compared to the second-best method GOAL, on the T$\rightarrow$I and I$\rightarrow$T task, respectively. Despite being with the shortest word length, SKETCHY dataset contains rich structure-centric multimodal content concerning clothing items (\eg, shape, layering, designs, contour, and outfit compositions), making our structure-centric multimodal alignment most effective to enhance the representation alignment on fine-grained details. 
%By explicitly isolating and aligning such structure-centric multimodal contents, \ourmethod~can effectively  the interference from appearance-related attributes, enabling the model to match hierarchical semantic structures and thus achieve higher retrieval precision, especially at top ranks. 
On the other hand, the performance of all models on the Insect dataset is relatively low as the dataset belongs to a rare domain, \ie specific families of insects, that is underrepresented in the pre-training dataset for VLMs.
Nevertheless, \ourmethod~still improves R@1 in the T$\rightarrow$I direction by +1.12\% (a relative gain of 12.7\%) with stable improvements for I$\rightarrow$T.
%This improvement mainly comes from the geometric constraints introduced by 
We believe this is attributed to the edge maps that provide a degree of morphological invariance across species highlighting subtle visual differences during the text-visual alignment.
Finally, \ourmethod~also shows consistent gains over compared methods on the general-domain datasets DOCCI and DCI, demonstrating its broad effectiveness.  We also perform strongly at deeper ranks, detailed results are provided in the \suppmat
%structural alignment not only benefits domain-specific understanding but also serves as a regularizing constraint in natural image domains, helping the model learn a more stable multimodal representation.

\begin{table}[t]
\caption{
\textbf{Plug-and-play enhancement of our $\ourloss$ on CLIP-based finetuning.}
Results on SKETCHY and DOCCI for \inlineColorbox{lightpink}{\emph{Text$\rightarrow$Image}} and \inlineColorbox{lightblue}{\emph{Image$\rightarrow$Text}} retrieval.
Upper: full-parameter finetuning; lower: parameter-efficient tuning.
Our method consistently boosts diverse CLIP variants. Best in \textbf{bold}, with gain in \gain{green}.
}
\vspace{-5pt}
\label{tab:finetuning_agnostic}
\centering
\footnotesize
\resizebox{\columnwidth}{!}{%
% {\fontsize{20}{24}\selectfont
% {\fontsize{14}{16}\selectfont
{\fontsize{24}{28}\selectfont
\begin{tabular}{@{} l*{5}{S[table-format=2.2] !{\color{lightlightgray}\vrule width 0.5pt}} S[table-format=2.2] | *{5}{S[table-format=2.2] !{\color{lightlightgray}\vrule width 0.5pt} } S[table-format=2.2] @{}}
% \begin{tabular*}{\textwidth}{@{\extracolsep{\fill}}lcccccccccccc@{}}
\toprule
\multirow{2}{*}{\textbf{Method}} &
\multicolumn{6}{c|}{\cellcolor{gray!15}\textbf{SKETCHY}} &
\multicolumn{6}{c}{\cellcolor{gray!15}\textbf{DOCCI}} \\
\cmidrule(lr){2-7}\cmidrule(lr){8-13}
& \multicolumn{1}{c}{\cellcolor{lightpink}{R@1}} & \multicolumn{1}{c}{\cellcolor{lightpink}{R@5}} & \multicolumn{1}{c}{\cellcolor{lightpink}{R@10}}
& \multicolumn{1}{c}{\cellcolor{lightblue}{R@1}} & \multicolumn{1}{c}{\cellcolor{lightblue}{R@5}} & \multicolumn{1}{c|}{\cellcolor{lightblue}{R@10}}
& \multicolumn{1}{c}{\cellcolor{lightpink}{R@1}} & \multicolumn{1}{c}{\cellcolor{lightpink}{R@5}} & \multicolumn{1}{c}{\cellcolor{lightpink}{R@10}}
& \multicolumn{1}{c}{\cellcolor{lightblue}{R@1}} & \multicolumn{1}{c}{\cellcolor{lightblue}{R@5}} & \multicolumn{1}{c}{\cellcolor{lightblue}{R@10}} \\
\midrule

Long-CLIP
& 54.32 & 80.14 & 88.43 & 52.76 & 80.31 & 88.08
& 64.49 & 87.67 & 93.43 & 63.08 & 87.45 & 93.14 \\
\textbf{+our $\ourloss$}
& \textbf{59.24} 
& \textbf{85.32} 
& \textbf{91.45} 
& \textbf{59.59} 
& \textbf{84.37} 
& \textbf{91.02} 
& \textbf{67.67} 
& \textbf{90.82} 
& \textbf{95.59} 
& \textbf{67.92} 
& \textbf{90.16} 
& \textbf{95.10}  \\

\textit{$\Delta$}
& \gain{4.92} & \gain{5.18} & \gain{3.02} 
& \gain{6.83} & \gain{4.06} & \gain{2.94}
& \gain{3.18} & \gain{3.15} & \gain{2.16} 
& \gain{4.84} & \gain{2.71} & \gain{1.96}
\\
\lightmidrule

FineLIP
& 40.59 & 71.16 & 81.78 & 40.33 & 72.11 & 82.38
& 67.80 & 90.22 & 94.84 & 66.39 & 89.12 & 94.47 \\
\textbf{+our $\ourloss$}
& \textbf{59.15} 
& \textbf{85.23} 
& \textbf{91.28} 
& \textbf{58.55} 
& \textbf{84.54} 
& \textbf{90.07} 
& \textbf{74.06}
& \textbf{94.24} 
& \textbf{97.35} 
& \textbf{72.94} 
& \textbf{93.27}
& \textbf{96.55} \\

\textit{$\Delta$}
& \gain{18.56} & \gain{14.07} & \gain{9.50} 
& \gain{18.22} & \gain{12.43} & \gain{7.69}
& \gain{6.26} & \gain{4.02} & \gain{2.51} 
& \gain{6.55} & \gain{4.15} & \gain{2.08}
\\
\lightmidrule

SmartCLIP
& 50.73 & 81.09 & 94.56 & 51.30 & 80.83 & \textbf{94.04}
    & 74.92 & 94.08 & 97.31 & 74.91 & 94.04 & 97.29 \\
\textbf{+our $\ourloss$}
& \textbf{52.94}
& \textbf{81.26} 
& \textbf{94.77} 
& \textbf{52.33} 
& \textbf{80.92} 
& \textbf{94.04} 
& \textbf{77.39} 
& \textbf{95.57} 
& \textbf{98.66} 
& \textbf{77.10} 
& \textbf{95.49} 
& \textbf{98.34}  \\

\textit{$\Delta$}
& \gain{2.21} & \gain{0.17} & \gain{0.21} 
& \gain{1.03} & \gain{0.09} & \neutral{0.00}
& \gain{2.47} & \gain{1.49} & \gain{1.35} 
& \gain{2.19} & \gain{1.45} & \gain{1.05}
\\
\lightmidrule

GOAL
& 63.21 & 87.13 & 93.44 & 62.44 & 87.82 & 92.31
& 79.47 & 96.65 & 98.69 & 79.43 & 96.14 & 98.51 \\
\textbf{+our $\ourloss$}
& \textbf{67.88}
& \textbf{90.33} 
& \textbf{95.16} 
& \textbf{68.48} 
& \textbf{89.81} 
& \textbf{94.82} 
& \textbf{80.96} 
& \textbf{96.90} 
& \textbf{98.96} 
& \textbf{80.31}
& \textbf{96.73} 
& \textbf{98.84} \\

\textit{$\Delta$}
& \gain{4.67} & \gain{3.20} & \gain{1.72} 
& \gain{6.04} & \gain{1.99} & \gain{2.51}
& \gain{1.49} & \gain{0.25} & \gain{0.27} 
& \gain{0.88} & \gain{0.59} & \gain{0.33}
\\
\lightmidrule
SigLIP2
& 68.91 & 90.85 & 95.16 & 66.75 & 90.24 & 93.96
& 71.80 & 92.53 & 95.88 & 71.51 & 92.41 & 96.06 \\
\textbf{+our $\ourloss$}
& \textbf{73.49}
& \textbf{91.97} 
& \textbf{95.77} 
& \textbf{70.38} 
& \textbf{91.54} 
& \textbf{95.77} 
& \textbf{75.47} 
& \textbf{94.82} 
& \textbf{97.67} 
& \textbf{73.59}
& \textbf{94.33} 
& \textbf{97.43} \\

\textit{$\Delta$}
& \gain{4.58} & \gain{1.12} & \gain{0.61} 
& \gain{3.63} & \gain{1.3} & \gain{1.81}
& \gain{3.67} & \gain{2.29} & \gain{1.79} 
& \gain{2.08} & \gain{1.92} & \gain{1.37}
\\

\midrule[2pt]
LoRA
& 57.08 & 84.72 & 91.80 & 56.74 & 85.32 & 91.71
& 77.80 & 96.45 & 98.55 & 77.00 & 96.02 & 98.05 \\\textbf{+our $\ourloss$}
& \textbf{62.09}
& \textbf{86.79} 
& \textbf{93.95} 
& \textbf{59.41} 
& \textbf{85.92} 
& \textbf{92.75} 
& \textbf{79.35} 
& \textbf{96.61} 
& \textbf{98.57} 
& \textbf{78.65} 
& \textbf{96.31} 
& \textbf{98.45} \\

\textit{$\Delta$}
& \gain{5.01} & \gain{2.07} & \gain{2.15} 
& \gain{2.67} & \gain{0.60} & \gain{1.04}
& \gain{1.55} & \gain{0.16} & \gain{0.02} 
& \gain{1.65} & \gain{0.29} & \gain{0.40}
\\
\lightmidrule

DoRA
& 61.77 & 86.18 & 91.88 & 60.94 & 87.33 & 92.31
& 66.27 & 90.50 & 95.18 & 65.65 & 89.55 & 94.78 \\
\textbf{+our $\ourloss$}
& \textbf{65.20}
& \textbf{90.26} 
& \textbf{94.91} 
& \textbf{64.94} 
& \textbf{88.35} 
& \textbf{93.52} 
& \textbf{70.65}
& \textbf{92.47} 
& \textbf{96.57} 
& \textbf{69.22} 
& \textbf{91.80}
& \textbf{95.98}  \\

\textit{$\Delta$}
& \gain{3.43} & \gain{4.08} & \gain{3.03}
& \gain{4.00} & \gain{1.02} & \gain{1.21}
& \gain{4.38} & \gain{1.97} & \gain{1.39} 
& \gain{3.57} & \gain{2.25} & \gain{1.20}
\\
\bottomrule
\end{tabular}
}
}
\vspace{-8pt}
\end{table}

\noindent\textbf{Finetuning-agnostic improvement.} 
Our introduced losses $\ourloss$ are additional auxiliary losses, which can be flexibly integrated on top of any CLIP-based finetuning methods. To investigate its effectiveness, we experiment with both classic parameter-efficient finetuning (PEFT) techniques, including LoRA~\cite{hu2022lora} and DoRA~\cite{liu2024dora}, and the state-of-the-art (sota) CLIP-based finetuning methods presented in the main comparison (Tab.~\ref{tab:main_results}).
As shown in Tab.~\ref{tab:finetuning_agnostic}, $\ourloss$ is plug-and-play on any CLIP-based finetuning techniques, demonstrating \textbf{X}LIP-agnostic improvements on both general and specific domains\footnote{Results on the other datasets will be shown in the \suppmat}. The benefits also extend to the adapter-based PEFT techniques, when finetuning models with a few additional learnable parameters. 
%In the full-parameter finetuning setting, it achieves the most notable gain on the fashion-oriented SKETCHY dataset, boosting R@1 (T$\rightarrow$I) by +18.6\% when combined with FineLIP. 
In PEFT, $\ourloss$ with LoRA and DoRA achieves about +4-5\% R@1 improvement, showing that it can serve as a lightweight and plug-and-play enhancement module to improve multimodal alignment with negligible additional computational cost. 
 Notably, on SKETCHY we get an average improvement of 6.20\% (relative improvement of 12.5\%) at R@1 and on DOCCI an average improvement of 3.28\% (relative improvement of 4.72\%). Also, it seems that in general the positive margin diminishes as the recall ranking does augment; nonetheless, it is important to say that no negative margins (=worsening) happen at higher recall ranks, as we show in the \suppmat. 
%Overall, \ourmethod~is finetuning-agnostic and can reliably inject structural priors under different finetuning strategies to enhance multimodal alignment quality.

\begin{table}[t]
\centering
\caption{
\textbf{Cross-domain generalization between DCI and DOCCI.}
Train on one dataset and test on another.
Values are Recall@K (\%), using \inlineColorbox{lightpink}{\emph{Text$\rightarrow$Image}} 
and \inlineColorbox{lightblue}{\emph{Image$\rightarrow$Text}} retrieval. In-domain best in \textit{\textbf{italic bold}}, cross-domain best in \textbf{bold}.
}
\vspace{-5pt}
\label{tab:cross_generalization}
\scriptsize
\setlength{\tabcolsep}{3.6pt}
\renewcommand{\arraystretch}{1.1}

\begin{tabular*}{\columnwidth}{@{\extracolsep{\fill}}lccc|ccc@{}}
\toprule
\textbf{Setting} &
\cellcolor{lightpink}\textbf{R@1} &
\cellcolor{lightpink}\textbf{R@5} &
\cellcolor{lightpink}\textbf{R@10} &
\cellcolor{lightblue}\textbf{R@1} &
\cellcolor{lightblue}\textbf{R@5} &
\cellcolor{lightblue}\textbf{R@10} \\
\midrule

\multicolumn{7}{c}{\cellcolor{gray!10}\textbf{Train on DCI → Test on DCI vs DOCCI}} \\
\midrule
Long-CLIP (DCI→DCI) & \textit{59.23} & \textit{80.89} & \textit{87.04} & \textit{60.13} & \textit{81.44} & \textit{87.54} \\
Long-CLIP (DCI→DOCCI) & 58.73 & 84.66 & 90.75 & 58.08 & 83.53 & 90.47 \\

\lightmidrule
GOAL (DCI→DCI) & \textit{72.64} & \textit{89.89} & \textit{93.70} & \textit{72.84} & \textbf{\textit{90.50}} & \textit{93.20} \\
GOAL (DCI→DOCCI) & 71.22 & 92.39 & 96.47 & 72.18 & 92.88 & 96.49 \\

\lightmidrule
\textbf{\ourmethod} (DCI→DCI) & \textbf{\textit{75.90}} & \textbf{\textit{90.00}} & \textbf{\textit{95.15}} & \textbf{\textit{74.39}} & \textit{89.99} & \textbf{\textit{94.30}} \\
\textbf{\ourmethod} (DCI→DOCCI) & \textbf{75.47} & \textbf{93.41} & \textbf{97.22} & \textbf{72.98} & \textbf{93.24} & \textbf{96.75} \\
\midrule

\multicolumn{7}{c}{\cellcolor{gray!10}\textbf{Train on DOCCI → Test on DOCCI vs DCI}} \\
\midrule
Long-CLIP (DOCCI→DOCCI) & \textit{64.49} & \textit{87.67} & \textit{93.43} & \textit{63.08} & \textit{87.45} & \textit{93.14} \\
Long-CLIP (DOCCI→DCI) & 51.23 & 73.39 & 80.09 & 50.73 & 72.89 & 80.29 \\

\lightmidrule
GOAL (DOCCI→DOCCI) & \textit{79.47} & \textit{96.65} & \textit{98.69} & \textit{79.43} & \textit{96.14} & \textit{97.25} \\
GOAL (DOCCI→DCI) & 64.13 & 82.69 & 87.29 & 65.88 & 83.44 & 87.89 \\
\lightmidrule
\textbf{\ourmethod} (DOCCI→DOCCI) & \textbf{\textit{83.04}} & \textbf{\textit{97.06}} & \textbf{\textit{98.96}} & \textbf{\textit{81.59}} & \textbf{\textit{96.94}} & \textbf{\textit{98.78}} \\
\textbf{\ourmethod} (DOCCI→DCI) & \textbf{65.18} & \textbf{83.89} & \textbf{88.04} & \textbf{66.13} & \textbf{83.84} & \textbf{89.09} \\
\bottomrule
\end{tabular*}
\vspace{-8pt}
\end{table}

\noindent\textbf{Generalization across general domains.}
Following the evaluation of~\cite{goal2025}, we also investigate the finetuned models in cross-domain settings on the two general-domain DCI and DOCCI dataset. Tab.~\ref{tab:cross_generalization} shows that \ourmethod~demonstrates decent cross-domain generalization capability between DCI and DOCCI, with only a marginal decrease compared to in-domain results.
Nonetheless, \ourmethod~still outperforms GOAL and Long-CLIP by a clear margin.
This indicates that the auxiliary structure-centric alignment does not bias the model toward domain-specific representation, but instead promotes generalizable representation that can transfer well across different image-text distributions. In the case of cross general-specific domain generalization (see \suppmat), the decrease in performance is evident, with \ourmethod still showing comparable if not better robustness w.r.t. the competitors.

\subsection{Ablation studies}

\begin{figure*}[h!]
    \centering
    \includegraphics[width=\linewidth]{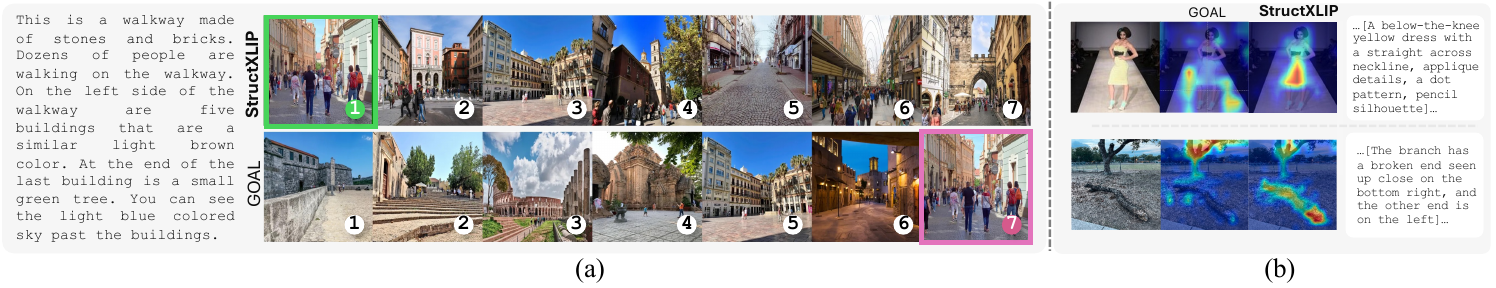}\vspace{-0.3cm}
    \caption{(a) Retrieval results; (b) GRAD-CAM~\cite{selvaraju2017grad} highlighting cross-modal attention between the object-centric text and image.}
    \label{Fig:qualitative}
\end{figure*}

\noindent\textbf{Structure-centric alignment losses.}  We ablate the impact of our introduced structure-centric alignment losses $\ourloss$, which consists of $\structloss$, $\consistencyloss$, and $\structlosslocal$. As shown in Tab.~\ref{tab:ablation_loss}, all three losses contribute positively to the retrieval improvement.
The global structure alignment loss $\structloss$ alone already yields a clear gain.
The consistency regularization $\consistencyloss$ further stabilizes learning while maintaining coherence with the original semantic space.
Finally, adding the local structure-centric alignment loss $\structlosslocal$ achieves the best results. Global alignment enforces holistic shape consistency, while local alignment strengthens fine-grained correspondences, leading to best multimodal alignment.

\noindent\textbf{Types of extraction methods.} We evaluate several visual and textual extraction strategies, all of which yield consistent gains, showing that the model adapts robustly to structure-centric extraction. On the visual extraction side, we consider two categories of structure extraction methods:
(i) classic \emph{filter-based}~(indicated by [F]) edge extraction methods (\eg, Canny~\cite{canny2009computational} and LoG~\cite{marr1980theory}), and
(ii) \emph{learning-based}~(indicated by [L]) methods (\eg, HED~\cite{xie2015holistically}, LineArtDetector (LAD)~\cite{zhang2023adding}, and Photo2Sketch (P2S)~\cite{liu2021learn}).
The filter-based approaches achieve the best overall performance, while learning-based extractors yield comparably strong results with only marginal differences. 
For the textual extraction, the Original Text approach retains appearance-related vocabulary, while the LLM-Extraction prompts a LLM to rewrite the description only with structure-centric semantics, but may introduce hallucinated attributes.
Our Lexicon Filter balances factual grounding and structural relevance, achieving the best results.

To further assess the model's robustness, we introduce a \emph{Noise Injection} setting.
On the visual side, Incorrect Match randomly pairs an image with edge maps from a different dataset, \ie matching DOCCI edges with SKETCHY images, vice versa.
On the textual side, Random Word Mask (RWM) randomly removes all non-structural and non-functional words outside the vocabulary $\vocappearance$,
and Word Reduction retains only about 10\% of the original words to simulate severe semantic loss. As shown in Tab.~\ref{tab:extraction_ablation}, both the visual and textual noise experiments cause only minor degradation, exhibiting strong robustnesss as the model remains stable and reliable even under structural noises and misalignment in $\ourloss$. The 90\% reduction variant underperforms on SKETCHY, as the short captions make it difficult for the remaining 10\% of words to cover most objects. 
%We present more robustness analysis in the \suppmat

\begin{table}[t]
\centering
\caption{
\textbf{Ablation on losses.}
Results on \inlineColorbox{lightpink}{\emph{Text$\rightarrow$Image}} and 
\inlineColorbox{lightblue}{\emph{Image$\rightarrow$Text}} retrieval for SKETCHY and DOCCI.
Each row adds one additional loss term to the previous configuration.}
\vspace{-5pt}
\label{tab:ablation_loss}
%\scriptsize
\setlength{\tabcolsep}{3.5pt}
\renewcommand{\arraystretch}{1.15}
\resizebox{\columnwidth}{!}{%
{\fontsize{23}{28}\selectfont
 % \begin{tabular}{@{} ccccccccc|cccccc @{}}
 \begin{tabular}{@{} ccc  *{5}{c!{\color{lightlightgray}\vrule width 0.5pt}}c| *{5}{c!{\color{lightlightgray}\vrule width 0.5pt}}c @{}}
\toprule
\multirow{2}{*}{$\structloss$} &  \multirow{2}{*}{$\consistencyloss$} &  \multirow{2}{*}{$\structlosslocal$} &
\multicolumn{6}{c|}{\cellcolor{gray!15}\textbf{SKETCHY}} &
\multicolumn{6}{c}{\cellcolor{gray!15}\textbf{DOCCI}} \\
\cmidrule(lr){4-9}\cmidrule(lr){10-15}
&&& \multicolumn{1}{c}{\cellcolor{lightpink}{R@1}} & \multicolumn{1}{c}{\cellcolor{lightpink}{R@5}} & \multicolumn{1}{c}{\cellcolor{lightpink}{R@10}}
& \multicolumn{1}{c}{\cellcolor{lightblue}{R@1}} & \multicolumn{1}{c}{\cellcolor{lightblue}{R@5}} & \multicolumn{1}{c|}{\cellcolor{lightblue}{R@10}}
& \multicolumn{1}{c}{\cellcolor{lightpink}{R@1}} & \multicolumn{1}{c}{\cellcolor{lightpink}{R@5}} & \multicolumn{1}{c}{\cellcolor{lightpink}{R@10}}
& \multicolumn{1}{c}{\cellcolor{lightblue}{R@1}} & \multicolumn{1}{c}{\cellcolor{lightblue}{R@5}} & \multicolumn{1}{c}{\cellcolor{lightblue}{R@10}} \\
\midrule

\textcolor{red!75}{\usym{2717}} & \textcolor{red!75}{\usym{2717}} & \textcolor{red!75}{\usym{2717}} & 
66.75 & 89.55 & 93.61 & 64.51 & 88.76 & 93.44 & 80.38 & 96.45 & 98.76 & 79.43 & 96.51 & 98.69 \\
\textcolor{lightgreen}{\usym{2713}} & \textcolor{red!75}{\usym{2717}} & \textcolor{red!75}{\usym{2717}} & 
67.88 & \textbf{90.85} & 95.34 & 67.18 & 89.21 & 94.91 & 81.20 & 96.83 & 98.95 & 80.69 & 97.04 & 98.73 \\
\textcolor{lightgreen}{\usym{2713}} & \textcolor{lightgreen}{\usym{2713}} & \textcolor{red!75}{\usym{2717}} & 
68.57 & 90.50 & \textbf{95.60} & \textbf{68.22} & 89.64 & 95.16 & 81.92 & 97.01 & \textbf{98.98} & 81.05 & \textbf{97.10} & 98.75 \\
\textcolor{lightgreen}{\usym{2713}} & \textcolor{lightgreen}{\usym{2713}} & \textcolor{lightgreen}{\usym{2713}} &
\textbf{69.86} & \textbf{90.85} & 95.42 & \textbf{68.22} & \textbf{90.67} & \textbf{95.68} & \textbf{83.04} & \textbf{97.06} & 98.96 & \textbf{81.59} & 96.94 & \textbf{98.78} \\

\bottomrule
\end{tabular}
}
}
\vspace{-8pt}
\end{table}

\noindent\textbf{Sample efficiency.} We present in Fig.~\ref{fig:data.amount} showing how \ourmethod performs with varying ratios of the fine-tuning dataset, in comparison with other sota methods. Our approach outperforms the competitors at every regime. Notably, the margin seems to be maximized around 20-50\% of the fine-tuning data, especially with the DOCCI dataset. %Notably, \ourmethod on SKETCHY has the steepest ascent, with an absolute gap of 16.5 in Text$\rightarrow$Image (relative increase of 45\%). This because SKETCHY is more fine-grained than DOCCI, and its statistics could explained with a larger amount of data, in percentage. 

%Our approach exhibits strong performance in the few-data fine-tuning regime because it focuses the optimization on the informative components of the signal. 
%In standard global fine-tuning, the model must adapt a large parameter space to account for every visual and textual variation, including nonsemantic noise, which requires extensive supervision. 
%In contrast, our disentangled extractions $X_E$ and $Y_E$ isolate structural cues such as edges and textual chunks that retain the core semantics while suppressing nuisance factors. 
%The accompanying objectives $\mathcal{L}_{\text{ST}}$, $\mathcal{L}_{\text{RS}}$, and $\mathcal{L}_{\text{CK}}$ further constrain the alignment toward consistent and interpretable relations across modalities. 
%This structural bias effectively reduces the hypothesis space and the variance of gradient updates, yielding faster convergence and improved generalization under limited data. 
%In simple terms, the model learns less but better, focusing its capacity on semantically stable features that global fine-tuning must otherwise relearn from larger datasets.

% \subsection{Interplay with GOAL}
% As shown in Table~\ref{tab:goal_interplay}, GOAL improves general-domain alignment but can degrade performance on structure-centric datasets; combining GOAL with our objectives mitigates this and brings further gains.

\subsection{Qualitative analysis}

Fig.~\ref{Fig:qualitative} (a) shows a qualitative retrieval comparison on the DOCCI dataset between \ourmethod and the second-best method GOAL~\cite{goal2025}. Apart from the better retrieval of \ourmethod, there is a substantial difference on the gist of the retrieved images: \ourmethod extracts images that match more structure details (lamps and people are more present). One may think we are forgetting colors due to our emphasis on structure, but this is not the case: the brown, green, and light azure patterns are also evident. 
%On the SKETCHY dataset, it appears that the shorts are more present in the \ourmethod retrieved results. This is probably due to the structure analysis which emphasis shape and contours of the garment item, absent in the competitor.  
In the attention map of Fig.~\ref{Fig:qualitative} (b), we visualize the cross-modal attention between two textual descriptions of an object and the image after fine-tuning. On the SKETCHY dataset (first row), it is evident how we match the dress with a very sharp attention map, in contrast to the one of GOAL. 
%This is also evident in the second row, where the attention is uniquely focused on the dress, differetnly from the baseline. 
The DOCCI example  
%(second line) the back of the cat's head is correctly highlight by \ourmethod. Surprisingly, even the corner of the white window is also salient, showing a good sensibility to colors. Finally 
(second row) showcases the \ourmethod's ability in capturing the correct layout of the image, with the GOAL completely ignoring the tree branch on the floor.

\begin{table}[t]
\centering
\caption{
\textbf{Ablation on extraction methods.}
\inlineColorbox{lightpink}{\emph{Text$\rightarrow$Image}} and 
\inlineColorbox{lightblue}{\emph{Image$\rightarrow$Text}} retrieval on SKETCHY and DOCCI datasets. 
[F] indicates \emph{filter-based} and [L] indicates \emph{learning-based} edge extraction methods. \textbf{Bold} numbers indicate the best results within each subtable. * denotes the option in our \ourmethod.
}
\label{tab:extraction_ablation}
\small
\vspace{-5pt}
\setlength{\tabcolsep}{1.5pt}
\renewcommand{\arraystretch}{1.05}

% -------------------- (a) VISUAL --------------------
\resizebox{\columnwidth}{!}{
\begin{tabular}{@{}l *{5}{c!{\color{lightlightgray}\vrule width 0.5pt}}c|*{5}{c!{\color{lightlightgray}\vrule width 0.5pt}}c @{}}
\toprule
\multirow{2}{*}{\shortstack[c]{\textbf{(a) Visual}\\\textbf{Extraction}}} &
\multicolumn{6}{c|}{\cellcolor{gray!15}\textbf{SKETCHY}} &
\multicolumn{6}{c}{\cellcolor{gray!15}\textbf{DOCCI}} \\
\cmidrule(lr){2-7}\cmidrule(lr){8-13}
& \multicolumn{1}{c}{\cellcolor{lightpink}{R@1}} & \multicolumn{1}{c}{\cellcolor{lightpink}{R@5}} & \multicolumn{1}{c}{\cellcolor{lightpink}{R@10}}
& \multicolumn{1}{c}{\cellcolor{lightblue}{R@1}} & \multicolumn{1}{c}{\cellcolor{lightblue}{R@5}} & \multicolumn{1}{c|}{\cellcolor{lightblue}{R@10}}
& \multicolumn{1}{c}{\cellcolor{lightpink}{R@1}} & \multicolumn{1}{c}{\cellcolor{lightpink}{R@5}} & \multicolumn{1}{c}{\cellcolor{lightpink}{R@10}}
& \multicolumn{1}{c}{\cellcolor{lightblue}{R@1}} & \multicolumn{1}{c}{\cellcolor{lightblue}{R@5}} & \multicolumn{1}{c}{\cellcolor{lightblue}{R@10}} \\
\midrule
\rowcolor{gray!8}
\multicolumn{13}{l}{\textit{Methods}} \\
Canny*~\cite{canny2009computational} [F] & \textbf{69.86} & 90.85 & 95.42 & \textbf{68.22} & \textbf{90.67} & \textbf{95.68} & 83.04 & 97.06 & 98.96 & 81.59 & 96.94 & 98.78 \\
LoG~\cite{marr1980theory} [F] & 68.48 & 89.64 & 94.91 & 67.36 & 90.24 & 94.47 & \textbf{83.59} & \textbf{97.36} & \textbf{99.20} & \textbf{81.98} & \textbf{97.41} & \textbf{99.04} \\
HED~\cite{xie2015holistically} [L] & 68.91 & 90.67 & \textbf{95.43} & 66.93 & 90.41 & 94.47 & 83.14 & 97.18 & 98.98 & 81.73 & 97.08 & 98.69 \\
LAD~\cite{zhang2023adding} [L] & 69.17 & \textbf{90.93} & 95.42 & 66.93 & 90.16 & 94.91 & 82.69 & 97.10 & 98.90 & 81.08 & 96.84 & 98.84 \\
P2S~\cite{liu2021learn} [L] & 68.14 & 90.76 & 94.99 & 67.53 & 90.16 & 94.04 & 81.49 & 97.02 & 98.73 & 80.20 & 96.43 & 98.47 \\
\midrule
\rowcolor{gray!8}
\multicolumn{13}{l}{\textit{Noise Injection}} \\
Incorrect match & 66.20 & 88.95 & 93.39 & 64.37 & 88.26 & 93.78 & 80.25 & 96.33 & 98.76 & 79.15 & 96.10 & 98.41 \\
\bottomrule
\end{tabular}
}

% -------------------- (b) TEXTUAL --------------------
\resizebox{\columnwidth}{!}{
\begin{tabular}{@{}l *{5}{c!{\color{lightlightgray}\vrule width 0.5pt}}c|*{5}{c!{\color{lightlightgray}\vrule width 0.5pt}}c @{}}
\toprule
\multirow{2}{*}{\shortstack[c]{\textbf{(b) Textual}\\\textbf{Extraction}}} &
\multicolumn{6}{c|}{\cellcolor{gray!15}\textbf{SKETCHY}} &
\multicolumn{6}{c}{\cellcolor{gray!15}\textbf{DOCCI}} \\
\cmidrule(lr){2-7}\cmidrule(lr){8-13}
& \multicolumn{1}{c}{\cellcolor{lightpink}{R@1}} & \multicolumn{1}{c}{\cellcolor{lightpink}{R@5}} & \multicolumn{1}{c}{\cellcolor{lightpink}{R@10}}
& \multicolumn{1}{c}{\cellcolor{lightblue}{R@1}} & \multicolumn{1}{c}{\cellcolor{lightblue}{R@5}} & \multicolumn{1}{c|}{\cellcolor{lightblue}{R@10}}
& \multicolumn{1}{c}{\cellcolor{lightpink}{R@1}} & \multicolumn{1}{c}{\cellcolor{lightpink}{R@5}} & \multicolumn{1}{c}{\cellcolor{lightpink}{R@10}}
& \multicolumn{1}{c}{\cellcolor{lightblue}{R@1}} & \multicolumn{1}{c}{\cellcolor{lightblue}{R@5}} & \multicolumn{1}{c}{\cellcolor{lightblue}{R@10}} \\
\midrule
\rowcolor{gray!8}
\multicolumn{13}{l}{\textit{Methods}} \\
Original Text & 66.93 & 89.38 & 93.35 & 64.51 & 88.00 & 93.01 & 81.02 & 96.88 & 98.67 & 80.29 & 96.61 & 98.73 \\
LLM-Extraction  & 67.36 & 90.24 & 94.65 & 67.70 & 89.12 & 94.30 & 81.08 & 97.02 & 98.59 & 80.22 & 96.61 & 98.65 \\
Lexicon Filter*   & \textbf{69.86} & \textbf{90.85} & \textbf{95.42} & \textbf{68.22} & \textbf{90.67} & \textbf{95.68} & \textbf{83.04} & \textbf{97.06} & \textbf{98.96} & \textbf{81.59} & \textbf{96.94} & \textbf{98.78} \\
\midrule
\rowcolor{gray!8}
\multicolumn{13}{l}{\textit{Noise Injection}} \\
RWM & 68.91 & 89.98 & 93.96 & 67.19 & 88.95 & 94.04 & 81.49 & 96.53 & 98.74 & 80.49 & 96.55 & 98.62 \\
Word Reduction & 62.00 & 86.96 & 92.66 & 60.62 & 86.01 & 92.57 & 81.00 & 96.48 & 98.76 & 79.61 & 96.53 & 98.62 \\
\bottomrule
\end{tabular}
}
\vspace{-6pt}
\end{table}

%The good retrieval performance obtained by \ourmethod thanks to the better semantic alignment especially along the structural direction. When visualizing the cross-modal attention between the textual description of each object on the image, as shown Fig.~\ref{fig:attentionmap}, we can clearly observe that \ourmethod~demonstrates a better semantic alignment especially on the structural descriptions, \eg shape, positioning, spatial relations. For instance, \yiming{referring to the images to give more grounded discussion}.

%% file: tables/table1.tex
\begin{table*}[t!]
\centering
\caption{
\textbf{Crossmodal retrieval performance of CLIP-based finetuning methods on multiple datasets}. We report Recall@K (\%) on both \inlineColorbox{lightpink}{\emph{Text$\rightarrow$Image}} and \inlineColorbox{lightblue}{\emph{Image$\rightarrow$Text}} settings.
Each row represents one method evaluated on two datasets (left--right paired).
Best results in \textbf{bold}; second best \underline{underlined}.
$\Delta$ denotes the margin over our best competitor, with gain in \gain{green}.
}
\label{tab:main_results}
% \tighttable
\vspace{-5pt}
\footnotesize
\resizebox{\textwidth}{!}{
{\fontsize{20}{24}\selectfont
 \begin{tabular}{@{} l *{6}{S[table-format=2.2]} | *{6}{S[table-format=2.2]} | *{6}{S[table-format=2.2]}| *{6}{S[table-format=2.2]} @{}}
 % \begin{tabular}{cccccc:cccccc:cccccc:cccccc}
\toprule

\multirow{2}{*}{\textbf{Method}} &
\multicolumn{6}{c|}{\cellcolor{panelgray}\textbf{SKETCHY}} &
\multicolumn{6}{c|}{\cellcolor{panelgray}\textbf{INSECT}}  &
\multicolumn{6}{c|}{\cellcolor{panelgray}\textbf{DOCCI}} &
\multicolumn{6}{c}{\cellcolor{panelgray}\textbf{DCI}} \\
\cmidrule(lr){2-7}\cmidrule(lr){8-13}\cmidrule(lr){14-19}\cmidrule(lr){20-25}

& \multicolumn{1}{c}{\cellcolor{lightpink}{R@1}} & \multicolumn{1}{c}{\cellcolor{lightpink}{R@5}} & \multicolumn{1}{c}{\cellcolor{lightpink}{R@10}}
& \multicolumn{1}{c}{\cellcolor{lightblue}{R@1}} & \multicolumn{1}{c}{\cellcolor{lightblue}{R@5}} & \multicolumn{1}{c|}{\cellcolor{lightblue}{R@10}}
& \multicolumn{1}{c}{\cellcolor{lightpink}{R@1}} & \multicolumn{1}{c}{\cellcolor{lightpink}{R@5}} & \multicolumn{1}{c}{\cellcolor{lightpink}{R@10}}
& \multicolumn{1}{c}{\cellcolor{lightblue}{R@1}} & \multicolumn{1}{c}{\cellcolor{lightblue}{R@5}} & \multicolumn{1}{c|}{\cellcolor{lightblue}{R@10}}
& \multicolumn{1}{c}{\cellcolor{lightpink}{R@1}} & \multicolumn{1}{c}{\cellcolor{lightpink}{R@5}} & \multicolumn{1}{c}{\cellcolor{lightpink}{R@10}}
& \multicolumn{1}{c}{\cellcolor{lightblue}{R@1}} & \multicolumn{1}{c}{\cellcolor{lightblue}{R@5}} & \multicolumn{1}{c|}{\cellcolor{lightblue}{R@10}}
& \multicolumn{1}{c}{\cellcolor{lightpink}{R@1}} & \multicolumn{1}{c}{\cellcolor{lightpink}{R@5}} & \multicolumn{1}{c}{\cellcolor{lightpink}{R@10}}
& \multicolumn{1}{c}{\cellcolor{lightblue}{R@1}} & \multicolumn{1}{c}{\cellcolor{lightblue}{R@5}} & \multicolumn{1}{c}{\cellcolor{lightblue}{R@10}}\\

\midrule
Long-CLIP{\Large\textcolor{gray}{[ECCV'24]} } & 54.32 & 80.14 & 88.43 & 52.76 & 80.31 & 88.08 & 8.20 & 23.83 & 34.97 & \underline{9.41} & 24.78 & \underline{37.31} & 64.49 & 87.67 & 93.43 & 63.08 & 87.45 & 93.14 & 59.23 & 80.89 & 87.04 & 60.13 & 81.44 & 87.54 \\
FineLIP{\Large\textcolor{gray}{[CVPR'25]} } & 40.59 & 71.16 & 81.78 & 40.33 & 72.11 & 82.38 & 8.46 & 23.32 & 33.59 & 6.86 & 23.75 & 34.46 & 67.80 & 90.22 & 94.84 & 66.39 & 89.12 & 94.47 & 66.13 & 85.34 & 89.79 & 64.58 & 84.59 & 89.54 \\
SmartCLIP{\Large\textcolor{gray}{[CVPR'25]} } & 50.73 & 81.09 & \underline{94.56} & 51.30 & 80.83 & \underline{94.04} & 4.84 & 16.84 & 34.63 & 4.66 & 15.46 & 34.02 & 74.92 & 94.08 & 97.31 & 74.91 & 94.04 & \underline{97.29} & 69.88 & 86.64 & \underline{94.05} & 70.94 & 87.04 & 92.77 \\
GOAL{\Large\textcolor{gray}{[CVPR'25]} }  & \underline{63.21} & \underline{87.13} & 93.44 & \underline{62.44} & \underline{87.82} & 92.31 & \underline{8.81} & \underline{24.35} & \underline{35.84} & 8.55 & \underline{25.91} & 36.18 & \underline{79.47} & \underline{96.65} & \underline{98.69} & \underline{79.43} & \underline{96.14} & {97.25} & \underline{72.64} & \underline{89.89} & {93.70} & \underline{72.84} & \textbf{90.50} & \underline{93.20} \\
\textbf{\ourmethod~} & \textbf{69.86} & \textbf{90.85} & \textbf{95.42} & \textbf{68.22} & \textbf{90.67} & \textbf{95.68} & \textbf{9.93} & \textbf{26.60} & \textbf{38.34} & \textbf{9.50} & \textbf{26.60} & \textbf{39.64} & \textbf{83.04} & \textbf{97.06} & \textbf{98.96} & \textbf{81.59} & \textbf{96.94} & \textbf{98.78} & \textbf{75.90} & \textbf{90.00} & \textbf{95.15} & \textbf{74.39} & \underline{89.90} & \textbf{94.30} \\
\cmidrule(lr){1-25}
\textit{$\Delta$} & \gain{6.65} & \gain{3.72} & \gain{0.86} & \gain{5.78} & \gain{2.85} & \gain{1.64} & \gain{1.12} & \gain{2.25} & \gain{2.50} & \gain{0.09} & \gain{0.69} & \gain{2.33} & \gain{3.57} & \gain{0.41} & \gain{0.27} & \gain{2.16} & \gain{0.80} & \gain{1.49} & \gain{3.26} & \gain{0.11} & \gain{1.10} & \gain{1.55} & \loss{0.60} & \gain{1.10} \\
\bottomrule
\end{tabular}
}
}
\vspace{0.5em}

\end{table*}

%% file: sec/conclusion.tex
\section{Conclusion}\label{sec:conclusion}
We introduced \ourmethod, a structure-centric fine-tuning framework that leverages edge-based visual cues and their textual counterparts to strengthen vision-language alignment for long, detail-rich captions.
Our mutual information analysis provides theoretical insight into its stabilizing effect on optimization. Extensive experiments across general and specialized domains demonstrate consistent gains by \ourmethod, its robustness to variations in structural extraction, and plug-and-play compatibility with diverse fine-tuning frameworks. 
\ourmethod highlights the value of introducing structural cues at VLM fine-tuning, which gives rise to an inviting future direction: \textit{What would emerge if we extend our framework to train a VLM from scratch?}. However, we acknowledge that exploring this direction is very computationally demanding.

%% file: sec/X_suppl.tex
\clearpage
\setcounter{page}{1}
\maketitlesupplementary
\setcounter{figure}{0}
\renewcommand{\thefigure}{\arabic{figure}}
\setcounter{table}{0}
\renewcommand{\thetable}{\arabic{table}}

\setcounter{section}{0}
\renewcommand{\thesection}{\Alph{section}}

In this supplementary material, we provide a comprehensive analysis and additional details to support the main paper. The content is organized as follows:

\begin{itemize}
    \item \textbf{Dataset details} (\cref{supp:dataset_details}): We provide visual samples and statistical breakdowns for all datasets (\cref{supp:dataset_details}). Specifically, we detail the construction and curation process of the specific-domain INSECT dataset~(\cref{insect:dataset}).
    
    \item \textbf{Lexicon filtering} (\cref{supp:lexion}): We present the details about lexicon filtering  in terms of the used prompts, the appearance-related vocabulary and the filtering process (\cref{supp:VocabularyGeneration}). We also provide a statistical analysis on the filtered texts (\cref{supp:filter_stats}).
    
    \item \textbf{Extended information-theoretic analysis} (\cref{supp:theory_analysis}): We provide detailed theoretical proof (\cref{supp:theory_analysis:infotheory}) and its full supporting empirical analyses~(\cref{supp:theory_analysis:empirical}) across all four datasets.
    
    \item \textbf{Additional implementation details} (\cref{supp:extended_method}): We extend the method description with details on text token extension (\cref{supp:extended_method:length}) and the fine-tuning setup (\cref{supp:finetuningdetails}). Furthermore, we provide computational analysis (\cref{supp:computationalanalysis}).
    
    \item \textbf{Additional experimental analyses} (\cref{supp:more_exp}): We present
    additional ablation analysis on the lexicon filtering (\cref{supp:more_domainspecificlexicon}), additional cross-domain evaluation (\cref{supp:more_exp:cross_domain}), the extended experimental results at deeper ranks (\cref{supp:more_exp:deeprank}) and more qualitative results (\cref{supp:more_exp:qualitatives}) to complement the experimental evaluation in the main paper.
\end{itemize}

\section{Dataset Details}
\label{supp:dataset_details}
%In this section, we provide extended details for all datasets used in our experiments, including statistics, data split protocols, and construction procedures. 
Figs.~\ref{fig:SKETCHY_edge_map}-\ref{fig:DCI_edge_map} shows sample images with their extracted edge maps, associated original textual descriptions and the filtered textual descriptions from each dataset.

\begin{figure*}[t]
\centering
\footnotesize
\captionsetup{type=figure}

\begin{tikzpicture}
    \node[anchor=south west,inner sep=0] (image) at (0,0)
        {\includegraphics[width=\linewidth]{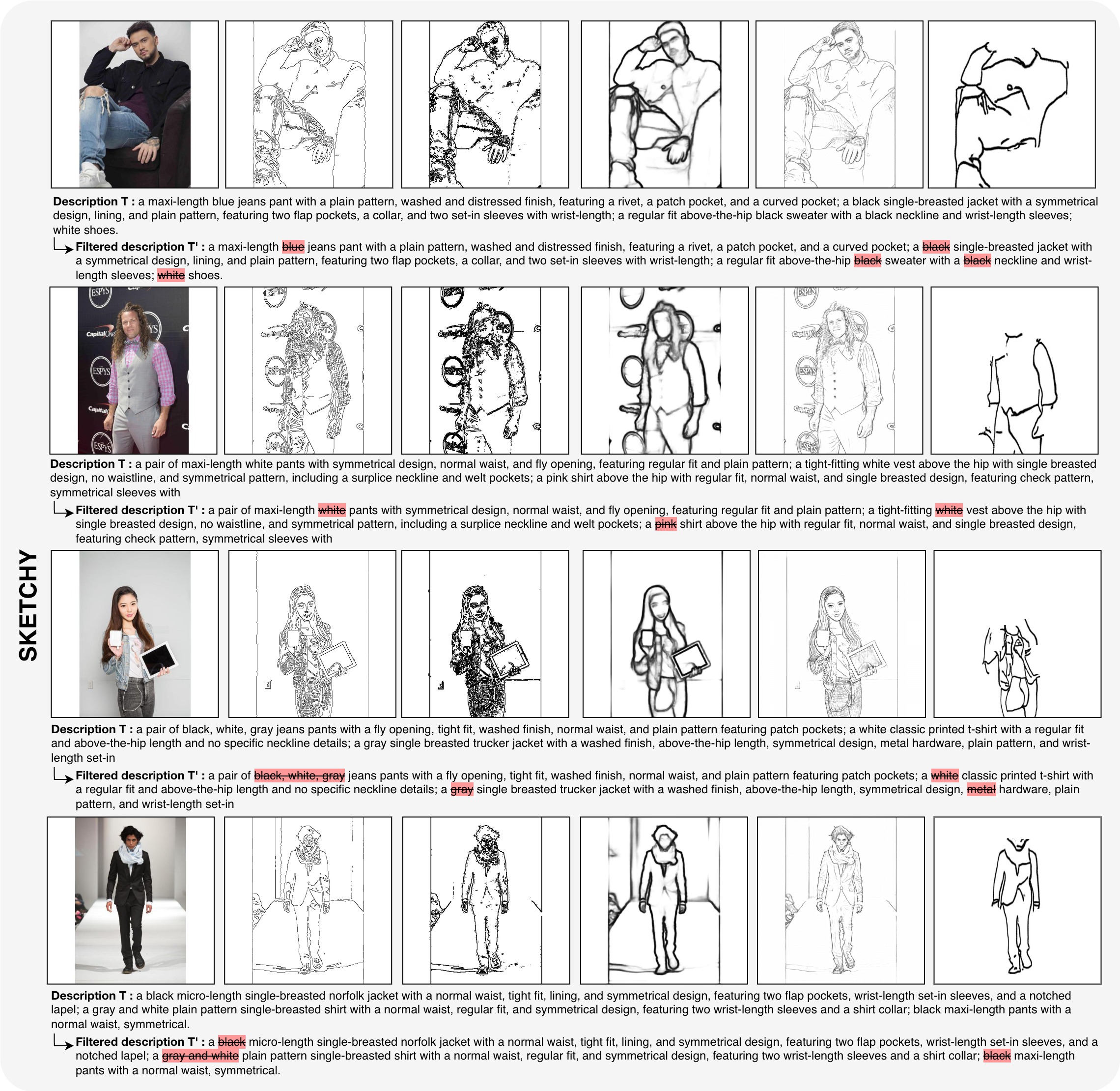}};

    \begin{scope}[x={(image.south east)},y={(image.north west)}]
        % Position words above each image - adjust coordinates as needed
        \node at (0.135,1.02) {\textbf{Original}};
        \node at (0.287,1.02) {\textbf{Canny}\cite{canny2009computational}};
        \node at (0.432,1.02) {\textbf{LoG}\cite{marr1980theory}};
        \node at (0.6,1.02)   {\textbf{HED}\cite{xie2015holistically}};
        \node at (0.752,1.02) {\textbf{LAD}\cite{zhang2023adding}};
        \node at (0.897,1.02) {\textbf{P2S}\cite{liu2021learn} };
    \end{scope}
\end{tikzpicture}

\caption{Illustration of SKETCHY dataset samples across different edge-map representations. For each example, we show the original RGB image together with its caption, followed by Canny, LoG, HED, LAD, and P2S edge maps, as well as the filtered caption $\textstruct$ that preserves only structure-centric information.}
\label{fig:SKETCHY_edge_map}
\end{figure*}

\begin{figure*}[t]
\centering
\footnotesize
\captionsetup{type=figure}

\begin{tikzpicture}
    \node[anchor=south west,inner sep=0] (image) at (0,0)
        {\includegraphics[width=\linewidth]{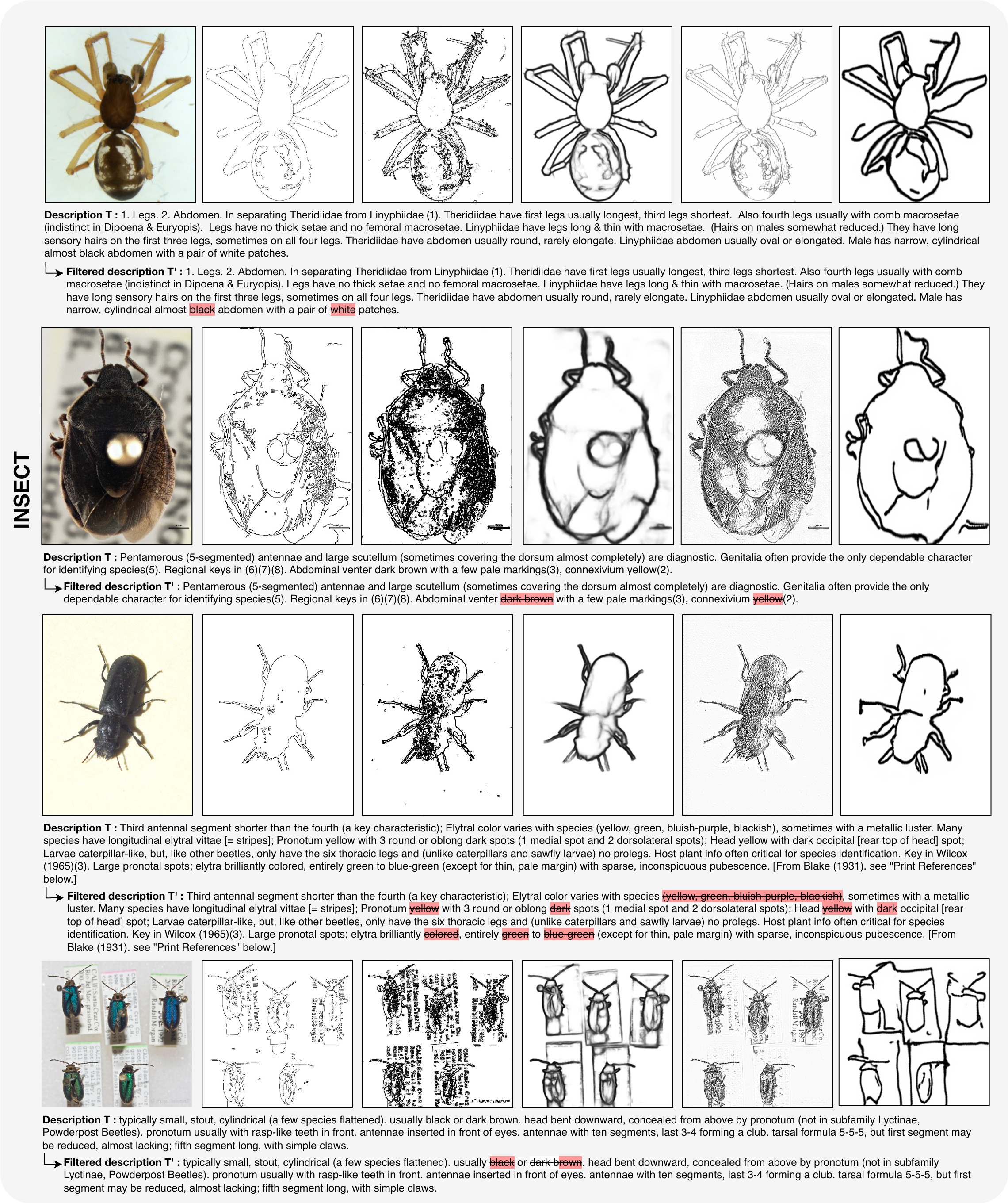}};

    \begin{scope}[x={(image.south east)},y={(image.north west)}]
        % Position words above each image - adjust coordinates as needed
        \node at (0.135,1.02) {\textbf{Original}};
        \node at (0.287,1.02) {\textbf{Canny}\cite{canny2009computational}};
        \node at (0.432,1.02) {\textbf{LoG}\cite{marr1980theory}};
        \node at (0.6,1.02)   {\textbf{HED}\cite{xie2015holistically}};
        \node at (0.752,1.02) {\textbf{LAD}\cite{zhang2023adding}};
        \node at (0.897,1.02) {\textbf{P2S}\cite{liu2021learn} };
    \end{scope}
\end{tikzpicture}

\caption{Illustration of INSECT dataset samples across different edge-map representations.}
\label{fig:INSECT_edge_map}
\end{figure*}

\begin{figure*}[t]
\centering
\footnotesize
\captionsetup{type=figure}

\begin{tikzpicture}
    \node[anchor=south west,inner sep=0] (image) at (0,0)
        {\includegraphics[width=\linewidth]{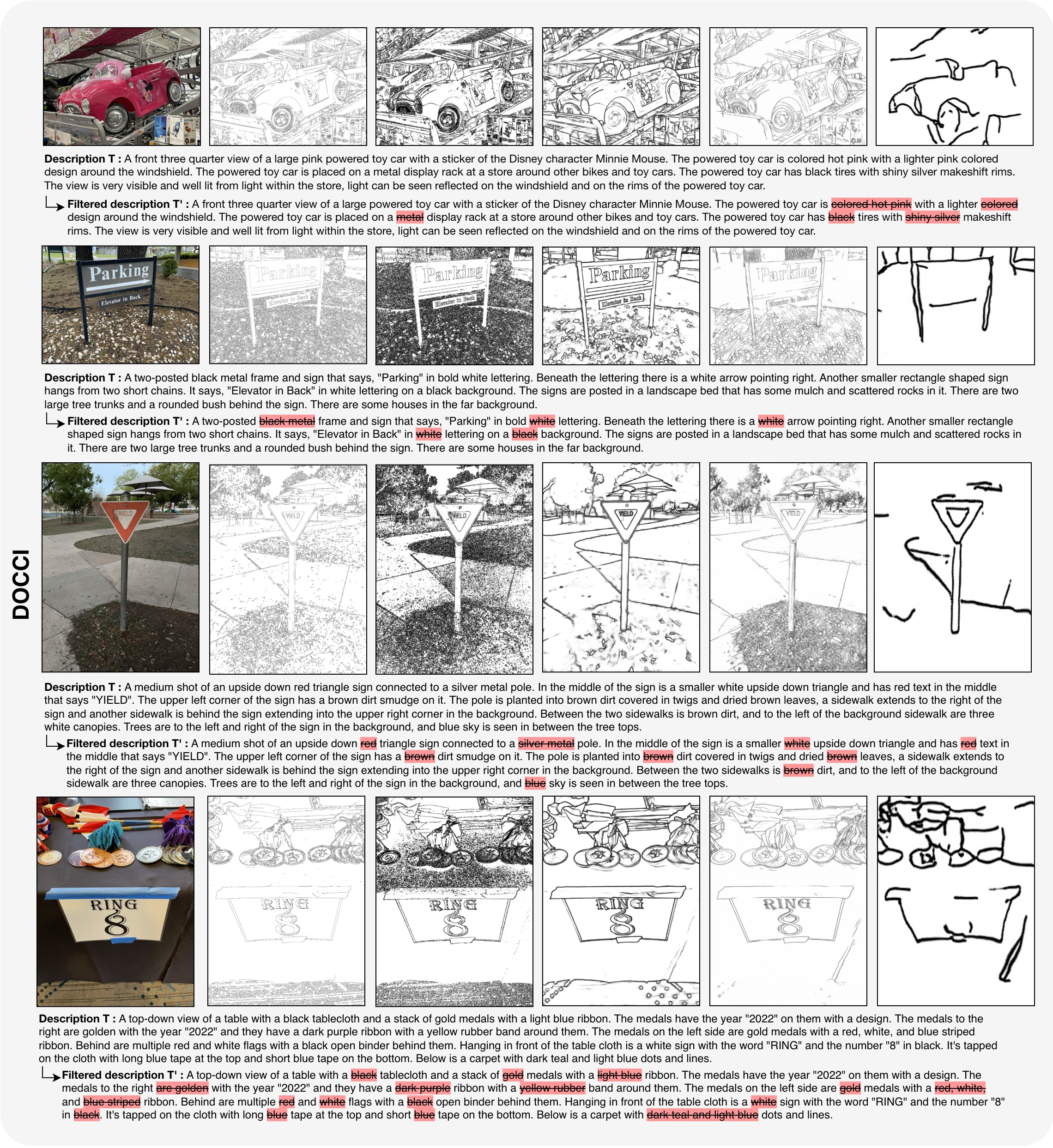}};

    \begin{scope}[x={(image.south east)},y={(image.north west)}]
        % Position words above each image - adjust coordinates as needed
        \node at (0.135,1.02) {\textbf{Original}};
        \node at (0.287,1.02) {\textbf{Canny}\cite{canny2009computational}};
        \node at (0.432,1.02) {\textbf{LoG}\cite{marr1980theory}};
        \node at (0.6,1.02)   {\textbf{HED}\cite{xie2015holistically}};
        \node at (0.752,1.02) {\textbf{LAD}\cite{zhang2023adding}};
        \node at (0.897,1.02) {\textbf{P2S}\cite{liu2021learn} };
    \end{scope}
\end{tikzpicture}

\caption{Illustration of DOCCI dataset samples across different edge-map representations.}
\label{fig:DOCCI_edge_map}
\end{figure*}

\begin{figure*}[t]
\centering
\footnotesize
\captionsetup{type=figure}

\begin{tikzpicture}
    \node[anchor=south west,inner sep=0] (image) at (0,0)
        {\includegraphics[width=\linewidth]{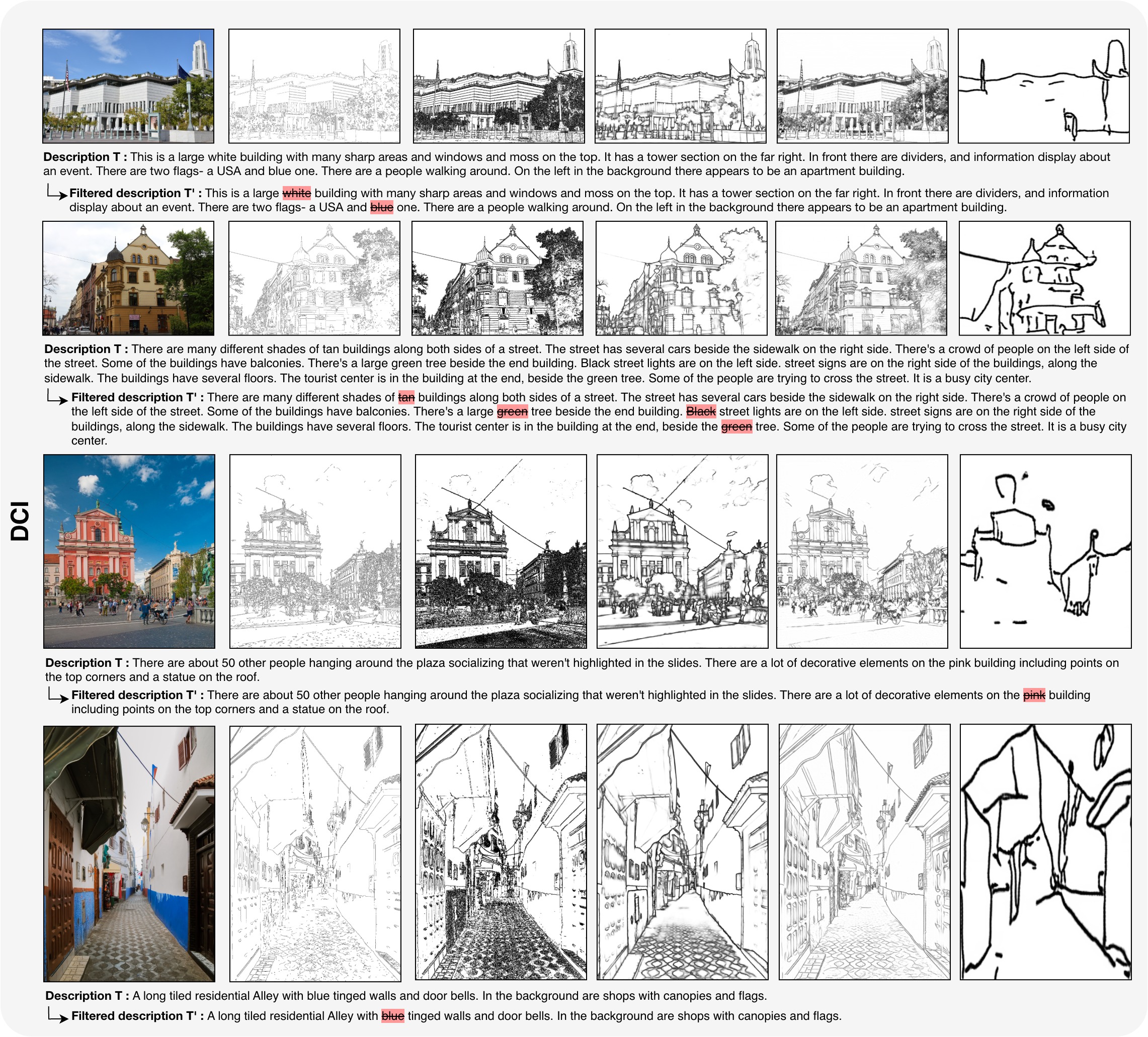}};

    \begin{scope}[x={(image.south east)},y={(image.north west)}]
        % Position words above each image - adjust coordinates as needed
        \node at (0.135,1.02) {\textbf{Original}};
        \node at (0.287,1.02) {\textbf{Canny}\cite{canny2009computational}};
        \node at (0.432,1.02) {\textbf{LoG}\cite{marr1980theory}};
        \node at (0.6,1.02)   {\textbf{HED}\cite{xie2015holistically}};
        \node at (0.752,1.02) {\textbf{LAD}\cite{zhang2023adding}};
        \node at (0.897,1.02) {\textbf{P2S}\cite{liu2021learn} };
    \end{scope}
\end{tikzpicture}

\caption{Illustration of DCI dataset samples across different edge-map representations.}
\label{fig:DCI_edge_map}
\end{figure*}
Our experiments are conducted on four datasets spanning different domains and levels of semantic granularity: DCI \cite{urbanek2024picture}, DOCCI \cite{docci2024}, SKETCHY \cite{girella2025lots}, and INSECT \cite{truong2025insect}. For the general-domain setting, we follow \cite{goal2025}, utilize two human-annotated datasets, DCI and DOCCI, which are originally designed for dense image captioning.
DCI contains 7,805 natural images, each paired with highly detailed and information-dense descriptions.
DOCCI consists of approximately 15k natural scene images collected across diverse geographic regions. Each image is annotated with long, highly compositional and discriminative descriptions, with an average length of 136 words.

In the domain-specific setting, we use SKETCHY and INSECT, two datasets characterized by fine-grained visual concepts and prominent structural properties.
SKETCHY is a large-scale fashion dataset containing roughly 46k outfit images, each paired with detailed descriptions covering garment components, fabric, pattern shapes, and spatial relations between parts. Although the average text length is only about 56 words, the descriptions are semantically dense and rich in visual detail.
INSECT contains 6k insect images with highly fine-grained and biologically rare categories that are underrepresented in general-purpose pretrained VLMs. Each image is paired with expert-verified biological descriptions covering coloration, wing structures, body-segment proportions, and species-level morphological traits, with an average length of 81 words. 
Since the original Insect-1M dataset contains substantial redundancy in both images and text, we construct a refined version suitable for fine-tuning on long-text understanding (see in the following). 

Regarding dataset splits, we follow the standard protocol for DOCCI and DCI as in~\cite{goal2025}, with 5,100 and 2,000 images in their test sets, respectively.
For SKETCHY, we adopt the official test split of 1.2k images.
For INSECT, since the original dataset provides no official split, we construct a split by dividing the curated dataset with an 8:2 train-test ratio.

\subsection{Condensed INSECT Dataset Construction}
\label{insect:dataset}
The INSECT dataset is derived from the large-scale insect image repository Insect-1M, which contains over one million images covering approximately 34,000 species, along with hierarchical textual annotations ranging from Phylum down to Species. Despite its scale and richness, the original dataset is not readily suitable for image–text retrieval or cross-modal fine-tuning. 
First, the dataset does not assign unique descriptions to individual images; instead, it uses an indexed-description mechanism in which large numbers of visually different images share the same short list of text tokens (\eg, descriptionA = [1,2,3]), resulting in \textit{extensive duplication of textual annotations}. Then, the high-level textual descriptions (\eg, Order, Class) are shared across 
thousands of samples, leading to large textual overhead that are not very meaningful for fine-grained discrimination compared to those description on finer granularity concerning, \eg Genus and Species.
% \textit{iii)} the fine-grained and biologically meaningful descriptions (Genus 
% and Species), which are crucial for fine-tuning, are overshadowed by higher-level
% textual noise.
 
We therefore construct a more suitable version of INSECT for long-text fine-tuning. First, to reduce duplication and increase representational diversity, we design a greedy selection algorithm based on description-ID overlap and select a \textit{core sample set} using an overlap threshold of 6, effectively removing redundant images and repeated descriptions.
Second, to create a more challenging cross-family generalization scenario, we apply a strict Family-level out-of-domain split, ensuring that the families in the test set are completely disjoint from those in the training set.
Then, to ensure that the textual side focuses on biologically discriminative features, we retain only the Genus- and Species-level descriptions and discard higher-level labels, yielding more compact and semantically fine-grained image–text pairs. Through these steps, we obtain a curated 6,000-pair high-quality INSECT dataset.
\textit{Both the dataset and the code used for its construction will be released publicly.}

\begin{figure*}[t]
\centering
\begin{tcolorbox}[
    width=0.95\linewidth,
    colback=gray!3,
    colframe=gray,
    boxrule=0.3pt,
    title={(1): General Appearance Lexicon Prompt–Response}
]
\small

% ------------------ USER QUERY ------------------
\begin{minipage}{0.12\linewidth}
\centering
\includegraphics[width=0.8\linewidth]{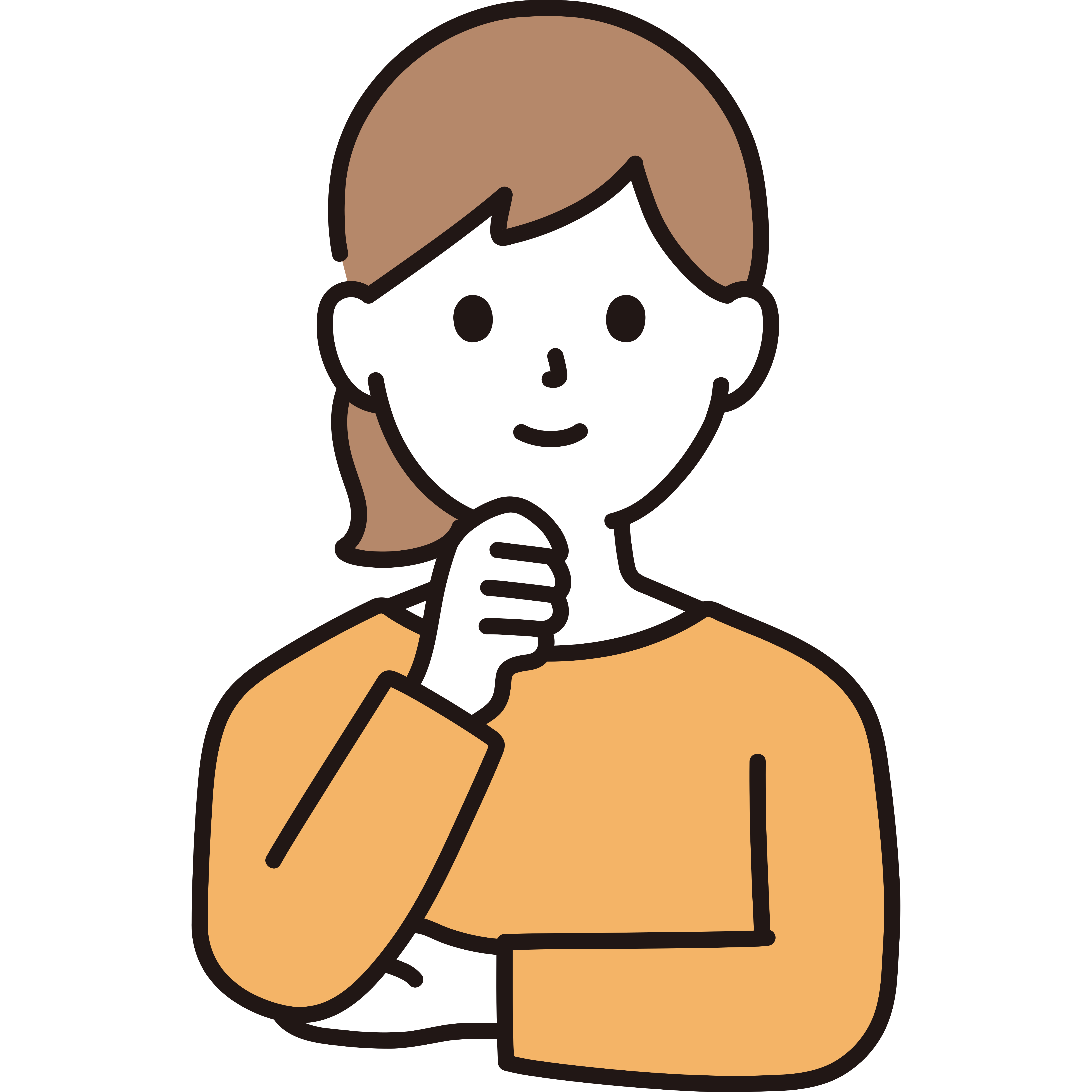}  % robot icon  % user icon
\end{minipage}
\hfill
\begin{minipage}{0.83\linewidth}
\textbf{USER:} You are a vision-language expert. Consider an image representation that contains only edge information 
(i.e., contours, boundaries, structure, and geometric layout), with all 
color, shading, material, and surface cues removed. Please produce a comprehensive list of English words and short phrases 
whose meanings cannot be inferred from such an edge-only representation. Return the list as comma-separated items.
\end{minipage}

\vspace{0.55em}
\rule{0.9\linewidth}{0.3pt}  % visual separator
\vspace{0.55em}

% ------------------ LLM RESPONSE ------------------
\begin{minipage}{0.80\linewidth}
\textbf{LLM:}  [red, blue, green, yellow, black, blackish, white, gray, grey, orange, purple, pink, brown, beige, cyan,  magenta, turquoise, teal, maroon, navy, violet, indigo, gold, silver, ivory, cream, olive, tan,  peach, mint, burgundy, crimson, scarlet, lavender, lilac, azure, teal, aqua, aquamarine, navy blue, 
    sky blue, baby blue, light blue, dark, dark blue, light green, dark green, forest green, lime green, 
    light red, dark red, rose red, wine red, light pink, hot pink, dark gray, light gray, dark grey, light grey, colored, cotton, wool, silk, linen, denim, leather, suede, velvet, satin, chiffon,
    polyester, nylon, spandex, acrylic, rayon, cashmere, fleece, corduroy, lace, mesh,
    canvas, tweed, felt, rubber, plastic, metal, steel, iron, aluminum, bronze, brass,
    ceramic, glass, wood, bamboo, stone, marble, granite, concrete, clay, paper,
    fur, shearling, down, feather, denier, foam]
\end{minipage}
\hfill
\begin{minipage}{0.18\linewidth}
\centering
\includegraphics[width=0.8\linewidth]{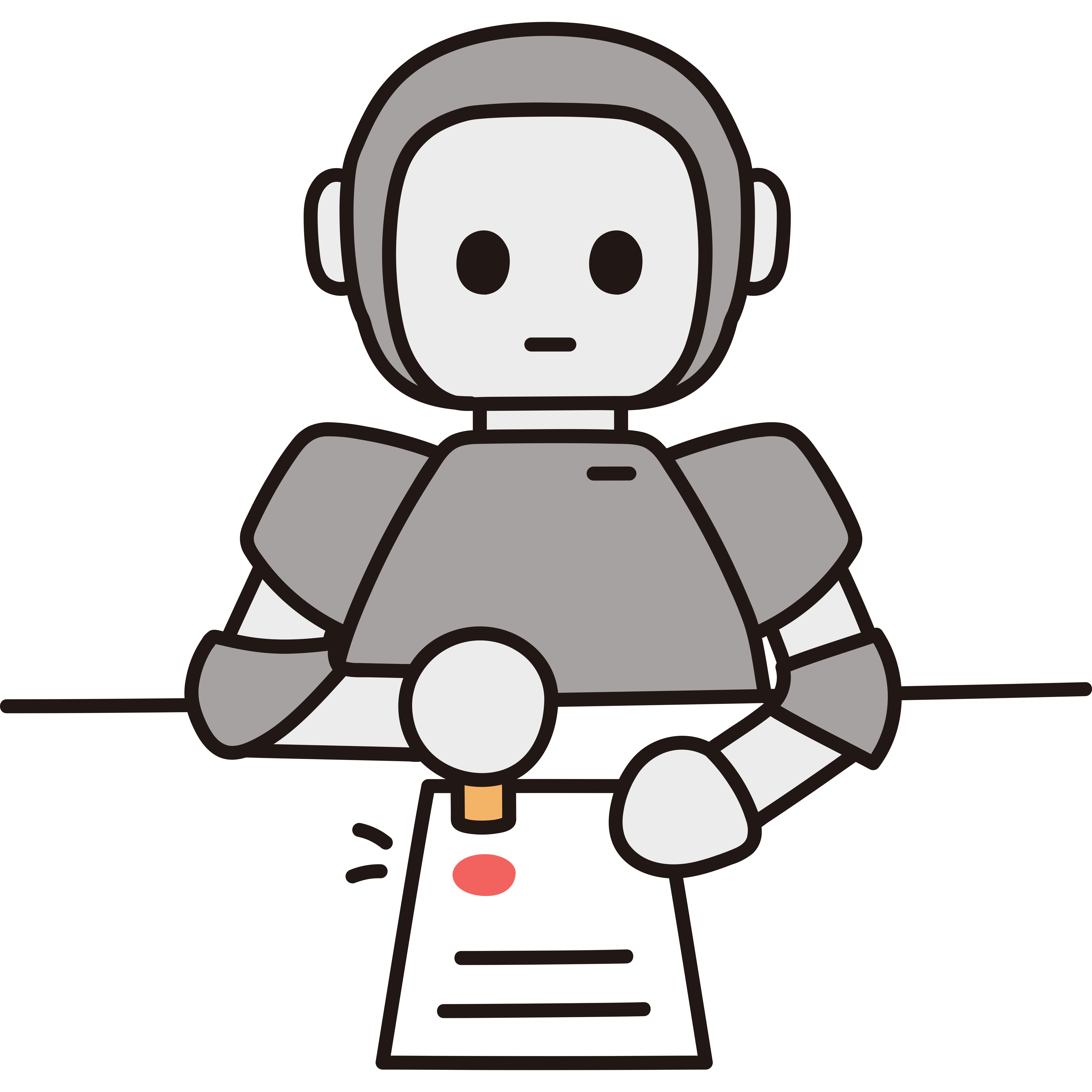}  % robot icon  % robot icon
\end{minipage}

\end{tcolorbox}

\caption{Prompt and response used in our LLM-based construction of the general appearance vocabulary for lexicon filtering. This lexicon is generated only once during preprocessing; the pipeline consistently uses this fixed lexicon thereafter, ensuring full reproducibility.}
\label{fig:general-lexion}
\end{figure*}

\section{Lexicon Filtering}
\label{supp:lexion}
This section provides additional details on the lexicon filtering process described in the main paper (Sec.~\ref{Structural_extraction}), including: i) how we use an LLM to construct the appearance vocabulary; ii) the implementation of our filtering function; and iii) statistical analyses of the filtering effect across the four datasets.

\subsection{Appearance Vocabulary Generation} 
\label{supp:VocabularyGeneration}
To systematically remove appearance-related terms such as color and material from the original text descriptions, we first construct a general-domain appearance lexicon, denoted as $\vocappearance$, which is later used for our lexicon filtering procedure. In the main paper, this lexicon is defined as $\vocappearance = \mathcal{LLM}(\prompfiltering)$. 
Here we provide additional details on how it is obtained.

As shown in Fig.~\ref{fig:general-lexion}, we use a large language model (in practice, ChatGPT-5.1) and provide it with a structured prompt $\prompfiltering$. The prompt instructs the model to consider a structure-centric representation and to produce a list of appearance-related terms that cannot be inferred from such structural information. The LLM returns a list of 600–800 appearance-only terms, including basic colors, color variants, material categories. The set of returned terms is used to form the final appearance lexicon $\vocappearance$. The complete vocabulary is listed in Fig.~\ref{fig:general-lexion}. In all experiments, we use the same general-domain appearance lexicon $\vocappearance$ as the filtering template for both general-domain and domain-specific datasets. The lexicon is built only once during pre-processing and is subsequently kept unchanged, enabling full reproducibility of the pipeline.

\subsection{Lexicon Filtering Function} 
\label{supp:lexiconfilteringfunction}
In the main paper, we denote the filtering function as  $\mathcal{F}(\cdot)$, whose purpose is to remove all appearance-related terms contained in $\vocappearance$ from the original textual description. Here we provide a detailed explanation of how this function is implemented in practice.
We first convert all entries in $\vocappearance$ into regular-expression matchers. During filtering, these expressions are applied to the input text to identify and remove any appearance terms found in the $\vocappearance$. The matching process is case-insensitive and respects word boundaries to avoid unintended partial matches. After removing appearance terms, we apply a lightweight grammatical cleanup. Since such removal can produce unnatural or fragmented text, for example, “a blue and white pattern” may temporarily become “a and pattern”. We first eliminate redundant spaces and punctuation to prevent repeated whitespace or stray commas. We then remove isolated conjunctions such as “and” or “or,” which lose their function once their associated tokens are deleted. Finally, we check whether the filtered sentence still contains sufficient semantic content. If too little meaningful text remains, we revert to the original description to avoid excessive information loss.

\subsection{Statistical Analysis on Lexicon Filtering} 
\label{supp:filter_stats}
We provide a statistical analysis of the effect of lexicon filtering across the four datasets. For consistency and ease of comparison, we coarsely categorize $\vocappearance$ into two classes: color words and material words, and analyze filtering coverage (``\textit{What proportion of captions were modified?}''), target specificity (``\textit{What types of words were removed?}''), and modification intensity (``\textit{How many words were removed on average?}'') based on this grouping.

As shown in Fig.~\ref{fig:statistic},  A reports the proportion of captions that were modified at least once. SKETCHY and DOCCI exhibit nearly 100\% intervention coverage, and DCI reaches 92.8\%, indicating that captions in these datasets commonly contain identifiable color or material descriptors. In contrast, INSECT shows only 58.9\% modified captions, reflecting that its descriptions inherently emphasize morphology and structure rather than appearance attribute, consistent with the style of biological taxonomic text. B summarizes the composition of removed words (computed over the modified captions). In SKETCHY and INSECT, approximately 97\% of removed terms are color-related, with material terms contributing only a negligible fraction. In DOCCI and DCI, material words account for 20\% and 17\%, respectively, aligning with the fact that captions in these datasets often mention building materials or object surface composition. C shows the average number of removed words per caption and the proportion relative to the original caption length. For the shorter captions of SKETCHY and INSECT, only 2.6-2.9 words are removed on average (corresponding to 3.1\% and 10.3\% of total caption length). For the longer captions in DOCCI and DCI, 7.1-7.2 words are removed on average, but this corresponds to only 5-6\% of the overall text. Our lexicon filtering removes appearance attributes effectively, without disrupting the structural or semantic core of the original caption. Overall, these three analyses demonstrate that lexicon filtering achieves high coverage, strong target specificity, and mild editing intensity, making it both effective and stable across domains.

\begin{figure*}[t]
    \centering
    \includegraphics[width=1.0\linewidth]{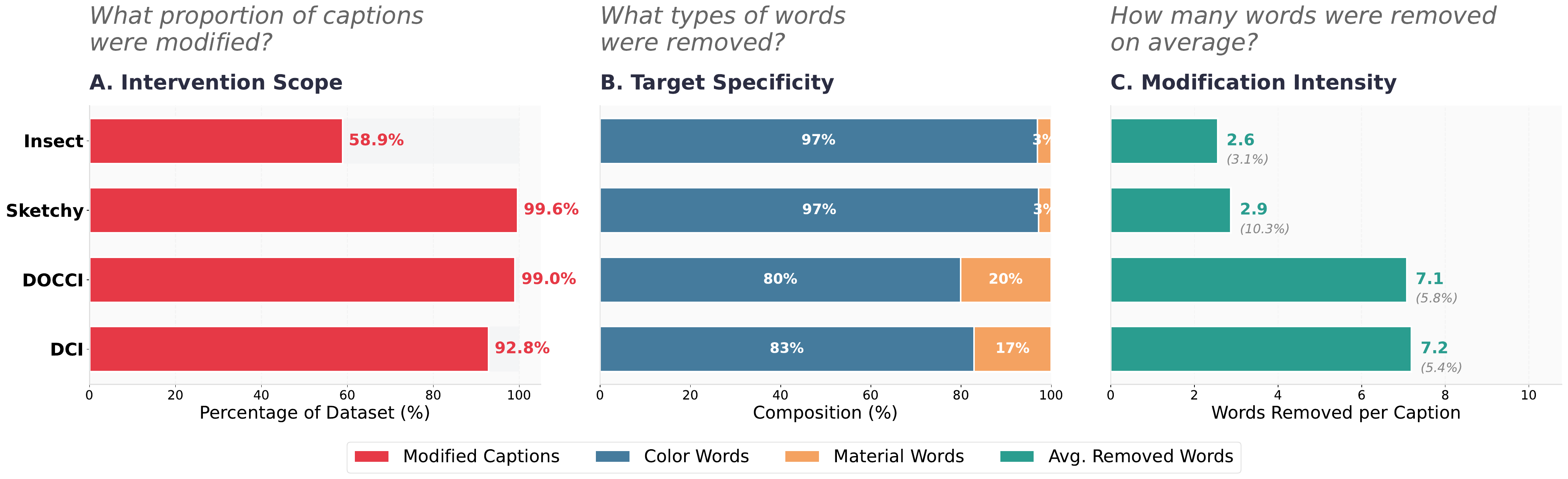}
    \caption{
\textbf{Statistical analysis of the lexicon filtering effect across the four datasets.}
For analysis, removed tokens are coarsely grouped into \emph{color words} and \emph{material words}.
(A) \textbf{Intervention scope:} percentage of captions in which at least one word was removed. 
(B) \textbf{Target specificity:} composition of removed tokens (computed over the modified captions in Panel~A), showing the proportion of color vs.\ material words.
(C) \textbf{Modification intensity:} average number of words removed per caption, with the percentage relative to the original caption length shown in parentheses.}

    \label{fig:statistic}
\end{figure*}
% TODO: full LLM prompt, Va construction pipeline, examples, statistics.

\section{Extended Information-Theoretic Analysis}
\label{supp:theory_analysis}
\subsection{Information-Theoretic Analysis}
\label{supp:theory_analysis:infotheory}
In this section we expand the theoretical view of Sec.~3.3 of the main paper. Specifically, we analyze the effect of the structure-centric objectives
$\mathcal{L}_{I',T'}$ and $\mathcal{L}_{I,I'}$ under the information-theoretic
and optimization lens giving three lemmas and a theorem.

\begin{lemma}[Information ordering for structure-centric views]
Let $I$ and $T$ denote the random variables associated with images and texts,
and let $I' = \mathcal{E}(I)$ and $T' = \mathcal{F}(T)$ be their structure-centric counterparts,
where $\mathcal{E}(\cdot)$ and $\mathcal{F}(\cdot)$ are deterministic maps that remove
appearance-related information and are not invertible.
Then the following inequality holds:
\[
I_{\mathrm{MI}}(I',T') < I_{\mathrm{MI}}(I,T),
\]
where $I_{\mathrm{MI}}(\cdot,\cdot)$ is the mutual information operator.
\end{lemma}

\noindent\textit{Proof sketch.}
The pair $(I',T')$ is obtained from $(I,T)$ through the
deterministic channel $(\mathcal{E},\mathcal{F})$. By the Data Processing Inequality~\cite{tishby99information},
any such transformation cannot increase mutual information, hence
$ I_{\mathrm{MI}}(I',T') \le I_{\mathrm{MI}}(I,T) $.
The inequality is strict whenever either $E$ or $F$ is not injective,
which holds in our setting since edge extraction and lexicon filtering
discard appearance cues, rather than introducing new information.

\medskip

\begin{lemma}[InfoNCE lower bounds for the two objectives]
Let $\imagefeat = \visualencoder(I)$ and $\textfeat = \textencoder(T)$ be the image and
text embeddings, and let $\imagefeat' = \visualencoder(I')$ and
$\textfeat' = \textencoder(T')$ be the structure-centric embeddings.
Assume both $\mathcal{L}_{I,T}$ and $\mathcal{L}_{I',T'}$ are symmetric
InfoNCE losses with batch size $N$. Then
\[
I_{\mathrm{MI}}(\imagefeat,\textfeat) \ge \log N - \mathbb{E}[\mathcal{L}_{I,T}],\quad
I_{\mathrm{MI}}(\imagefeat',\textfeat') \ge \log N - \mathbb{E}[\mathcal{L}_{I',T'}].
\]
\end{lemma}

\noindent\textit{Proof.}
This is the mere application of~\cite{poole2019variational}, stating that the InfoNCE loss is a standard variational lower bound on mutual information.
Applying the result to the pairs $(\imagefeat,\textfeat)$ and $(\imagefeat',\textfeat')$ yields the two
inequalities. The assumptions on symmetry and batch size match the
formulation in Sec. 3.2 of the main paper. In the following, to the sake of clarity, we assume $I_{\mathrm{MI}}(\imagefeat,\textfeat)\approx I_{\mathrm{MI}}(I,T) $ and $I_{\mathrm{MI}}(\imagefeat',\textfeat')\approx I_{\mathrm{MI}}(I',T') $.

\medskip

\begin{lemma}[Directional compatibility of gradients]
Let $\theta$ denote the parameters shared by the vision and text encoders.
Consider the gradients $\nabla_{\theta}\mathcal{L}_{I,T}$ and
$\nabla_{\theta}\mathcal{L}_{I',T'}$.
If the positive pairs in the two losses correspond to the same image–text
instances and the encoders are shared, then the expected cosine similarity
between the two gradients satisfies
\[
\mathbb{E}\big[\cos(\nabla_{\theta}\mathcal{L}_{I,T},
\nabla_{\theta}\mathcal{L}_{I',T'})\big] > 0.
\]
\end{lemma}

\noindent\textit{Proof sketch.}
Both losses use the same positive pairs and differ only in the views
fed to the encoders (full images and captions vs structure-centric
counterparts). The corresponding positive logits are maximized in
both objectives, while negatives are pushed apart.
This induces aligned update directions for parameters that affect
shared features. Under mild regularity assumptions on the encoders,
the expected cosine similarity between the two gradients is strictly
positive. In practice, this is confirmed by the empirical measurements
reported in Fig. 4(c).

\medskip

\begin{theorem}[Effect of the structure-centric auxiliary losses]
Consider the joint objective
\[
\mathcal{L}_{\mathrm{total}} = \mathcal{L}_{I,T}
+ \lambda_{1}\mathcal{L}_{I',T'} + \lambda_{2}\mathcal{L}_{I,I'},
\quad \lambda_{1},\lambda_{2} > 0.
\]
Under the assumptions of the lemmas above, the following properties hold:
\begin{enumerate}
\item The auxiliary alignment task between the pair $(\imagefeat',\textfeat')$ is information-reduced
compared to $(\imagefeat,\textfeat)$, hence $\mathcal{L}_{I',T'}$ optimizes a
harder objective in the sense of Lemma 1 and 2.
\item The gradients of $\mathcal{L}_{I,T}$ and $\mathcal{L}_{I',T'}$
are directionally compatible, so the auxiliary loss does not
conflict with the main alignment.
\item Due to the lower mutual information of $(\imagefeat',\textfeat')$,
the gradient norm of $\nabla_{\theta}\mathcal{L}_{I',T'}$
remains non-negligible even when $\nabla_{\theta}\mathcal{L}_{I,T}$
starts to vanish, which provides persistent optimization signal.
\item The consistency term $\mathcal{L}_{I,I'}$ bounds the drift
between $\imagefeat$ and $\imagefeat'$, so the structure-centric space stays
anchored to the semantic manifold of $\imagefeat$ and fine-tuning remains stable.
\end{enumerate}
\end{theorem}

\noindent\textit{Proof sketch.}

\textbf{Item 1} follows directly from Lemma 1 and the InfoNCE bounds in Lemma 2,
which place the pair $(\imagefeat',\textfeat')$ at a lower mutual information level than $(\imagefeat,\textfeat)$.
More in the detail, the auxiliary alignment task on $(\imagefeat',\textfeat')$ is strictly information-reduced compared to
$(\imagefeat,\textfeat)$, and therefore $\mathcal{L}_{I',T'}$ optimizes a harder objective.
More precisely, since $I' = \mathcal{E}(I)$ and $T' = \mathcal{F}(T)$ are obtained by applying
non-invertible, deterministic maps that remove appearance-related variability,
the entropy $H(\cdot)$ of both variables is reduced:
\[
H(I') < H(I), \qquad H(T') < H(T).
\]
By the Data Processing Inequality, this implies
\[
I_{\mathrm{MI}}(I',T') \le I_{\mathrm{MI}}(I,T).
\]
In addition, the structure-centric views induce a contraction of the positive
pair distribution: the space of valid matches becomes smaller, the intra-class
variability is suppressed, and the negative samples become less separable in
the embedding space. These effects lower the effective signal-to-noise ratio
of the InfoNCE objective, which increases the difficulty of the optimization
landscape associated with $\mathcal{L}_{I',T'}$.
Consequently, the gradients generated by $\mathcal{L}_{I',T'}$ tend to persist
longer during fine-tuning, since reaching its minimum requires modeling more
subtle, geometry-driven correspondences that remain unresolved after the
full-information objective $\mathcal{L}_{I,T}$ has already saturated.

\textbf{Item 2} follows from Lemma~3 and concerns the directional
compatibility of the gradients.  
Since both $\mathcal{L}_{I,T}$ and $\mathcal{L}_{I',T'}$ operate on the same
positive image–text instances and update the same encoder parameters,
their contrastive objectives induce aligned update rules:  
both maximize the positive logits and suppress the negative logits
associated with the same underlying pairs, even though the views differ.
Let the gradients be:  
\[
g = \nabla_{\theta}\mathcal{L}_{I,T}, \qquad
g' = \nabla_{\theta}\mathcal{L}_{I',T'}.
\]
Lemma~3 ensures that their expected cosine similarity satisfies  
\[
\mathbb{E}\big[\cos(g,g')\big] > 0.
\]
Therefore, the auxiliary gradient does not conflict with the semantic
direction promoted by the main loss.  
Instead, it expands the set of admissible descent directions within a
compatibility cone, enriching the optimization trajectory without
introducing destructive interference.

\textbf{Item 3} relies on the fact that, owing to the information-reduced nature of $(I',T')$,
the loss $\mathcal{L}_{I',T'}$ converges more slowly than
$\mathcal{L}_{I,T}$ and produces non-vanishing gradients
even when the main loss has already flattened.
Indeed, $\mathcal{L}_{I,T}$ optimizes an InfoNCE objective
associated with the higher mutual information quantity
$I_{\mathrm{MI}}(I,T)$, whose landscape typically admits a
faster descent: the positive pair $(\imagefeat,\textfeat)$ carries rich
appearance and structural cues, which help discriminate
positives from negatives early in fine-tuning.
In contrast, $\mathcal{L}_{I',T'}$ maximizes the lower
$I_{\mathrm{MI}}(I',T')$, where both $\imagefeat'$ and $\textfeat'$ have reduced
entropy because appearance information has been removed by
$\mathcal{E}(\cdot)$ and $\mathcal{F}(\cdot)$.  This contraction of the feature
space has two effects:
\begin{itemize}
\item the separation between positives and negatives becomes
smaller, making the contrastive objective harder to optimize;

\item the gradients associated with hard positives and hard
negatives decay more slowly, since the model must rely purely
on geometric and structural cues to improve the logits.
\end{itemize}
Formally, since the InfoNCE gradients satisfy
\[
\nabla_{\theta}\mathcal{L}_{I,T}
= \mathbb{E}\big[g(\imagefeat,\textfeat)\big],\qquad
\nabla_{\theta}\mathcal{L}_{I',T'}
= \mathbb{E}\big[g(\imagefeat',\textfeat')\big],
\]
and the score function $g(\cdot)$ for $(\imagefeat',\textfeat')$ has larger
relative variance due to the reduced mutual information,
the expected magnitude of $\nabla_{\theta}\mathcal{L}_{I',T'}$
remains positive for a longer portion of fine-tuning.
Empirically, this manifests in a later convergence time for
$\mathcal{L}_{I',T'}$ compared to $\mathcal{L}_{I,T}$, and in a
sustained gradient norm that continues to provide meaningful
updates even after the main contrastive gradients approach
zero. This behavior is consistent with the dynamics shown
in Fig.~4(a,b) of the main paper, where $\mathcal{L}_{I',T'}$ displays a slower flattening and a more persistent gradient profile.

Finally, \textbf{Item 4} exploits the form of $\mathcal{L}_{I,I'}$, which penalizes large
angular deviations between $\imagefeat$ and $\imagefeat'$ and therefore constrains
the structure-centric representations to remain close to their
full-information counterparts.
Specifically, since $\mathcal{L}_{I,T}$ and $\mathcal{L}_{I',T'}$ correspond to two
correlated but not identical contrastive objectives, their
gradients span a larger subspace than either loss alone.  
Formally, let
\[
g = \nabla_{\theta}\mathcal{L}_{I,T}, \qquad
g' = \nabla_{\theta}\mathcal{L}_{I',T'}.
\]
From Item 2 we know that $\cos(g,g')>0$ in expectation,
which guarantees compatibility.  However, because
$(\imagefeat',\textfeat')$ belongs to an information-reduced space and
responds differently to hard negatives and subtle geometric
patterns, one has
\[
g' \notin \mathrm{span}\{g\},
\]
so the matrix $[g,\;g']$ has rank $2$ almost everywhere.
This implies that the combined update
\[
g_{\mathrm{tot}}
= g + \lambda_{1}g' + \lambda_{2}\nabla_{\theta}\mathcal{L}_{I,I'},
\]
explores descent directions that the main loss alone cannot
access. This additional set of directions reduces the risk of
premature convergence to shallow minima (a common issue
in contrastive objectives), improves escape from flat regions
of the loss surface, and increases the robustness of SGD
trajectories.  In the language of multi-objective optimization,
$\mathcal{L}_{I',T'}$ introduces a complementary gradient
component that expands the effective feasible region of
updates while maintaining alignment with the main semantic
objective. The consistency loss $\mathcal{L}_{I,I'}$ ensures that this
expanded search space remains stable by preventing $\imagefeat'$ from
drifting too far from $\imagefeat$.  As a result, the combined gradient
retains diversity without diverging from the semantic
manifold defined by the original embeddings, yielding a
balanced and stable optimization process.

Taken together, these properties explain why the auxiliary
structure-centric alignment acts as a beneficial regularizer:
\textbf{\textit{it introduces correlated but information-reduced gradients that
improve convergence stability and lead to more robust alignment.}}

\begin{figure*}
    \centering
    \includegraphics[width=0.8\linewidth]{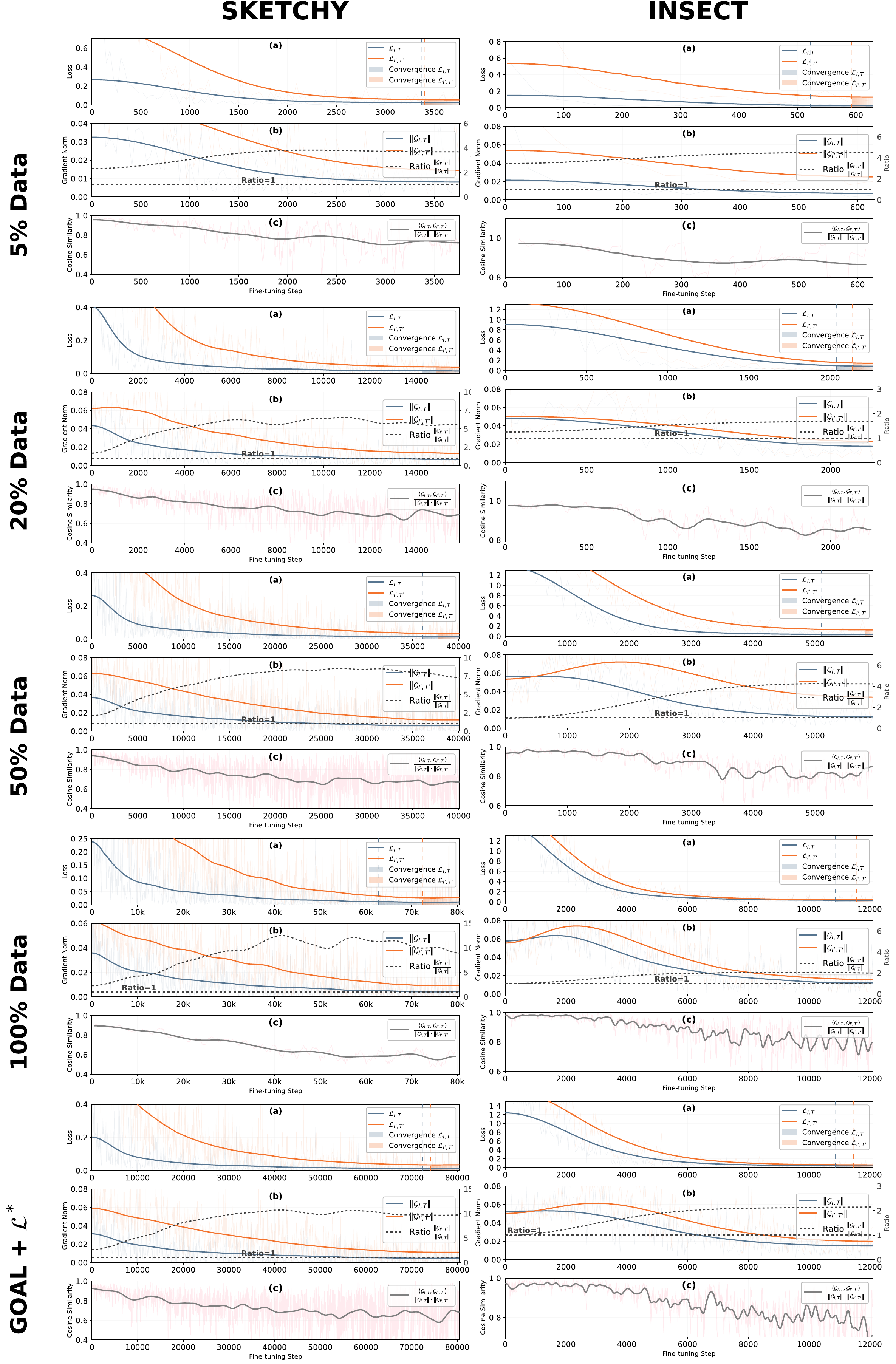}
    \caption{The convergence behavior of the two contrastive losses $\originalloss$ and $\structloss$ (a), the evolution of their gradient norms and their ratio (b), and the cosine similarity between the corresponding gradients (c), evaluated on all two datasets under different percentages of dataset size.}
    \label{fig:theory.exp1}
\end{figure*}

\begin{figure*}
    \centering
    \includegraphics[width=0.8\linewidth]{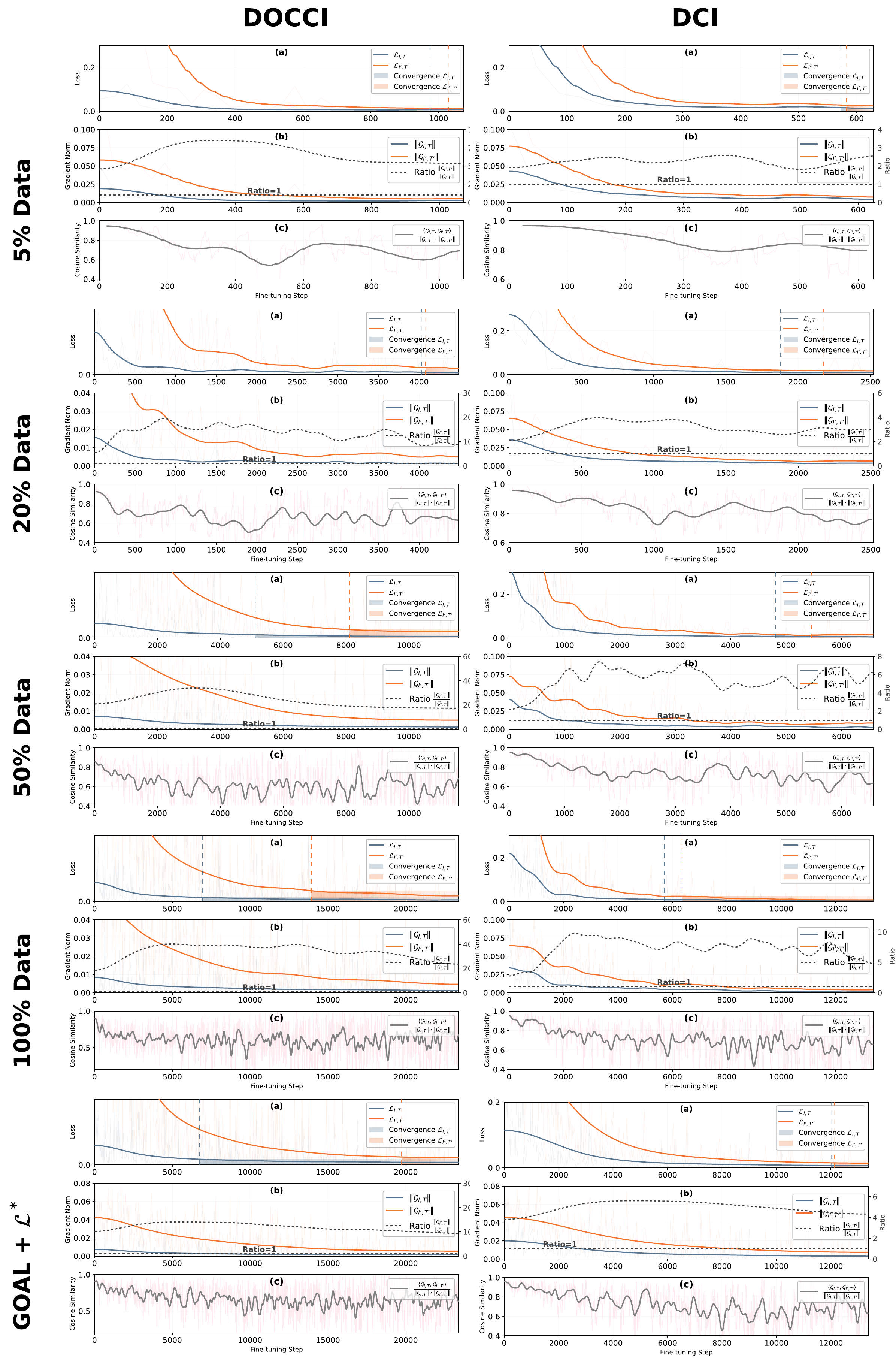}
    \caption{The convergence behavior of the two contrastive losses $\originalloss$ and $\structloss$ (a), the evolution of their gradient norms and their ratio (b), and the cosine similarity between the corresponding gradients (c), evaluated on all two datasets under different percentages of dataset size.}
    \label{fig:theory.exp2}
\end{figure*}

\subsection{Empirical Analyses}
\label{supp:theory_analysis:empirical}
In Sec.~\ref{sec:theory} of the main paper, we reported one representative
study of our information-theoretic analysis using the SKETCHY dataset.
Here we provide the remaining experiments conducted
on the other three datasets used in our evaluation: INSECT, DOCCI and DCI. 
To verify that the information-theoretic analysis holds consistently across different data scales, we conduct the same experiments on reduced-scale versions of each dataset containing 5\%, 20\%, and 50\% of the original data.
We further include experiments performed under the finetuning-agnostic improvement setting combined with the second-best method GOAL \cite{goal2025}, to confirm that the same behaviors arise consistently across different fine-tuning configurations.
For each experiment, we show the same three quantities presented in Fig.~\ref{fig:theory.exp1} and Fig.~\ref{fig:theory.exp2}: 
(a) the convergence behavior of the two contrastive losses
$\originalloss$ and $\structloss$, (b) the evolution of their gradient norms and their ratio, and (c) the cosine similarity between the corresponding gradients.

Across all datasets and settings, we observe the same consistent patterns as reported in
Sec.~\ref{sec:theory}. The auxiliary objective $\structloss$ which maximizes the mutual information $\mutualinfo(\imagestruct,\textstruct)$ under the
information-reduction mappings $\structextractor(\cdot)$ and $\mathcal{F}(\cdot)$, exhibits a later convergence than the main objective $\originalloss$. This is consistent with the
Data Processing Inequality~\cite{tishby99information}, which implies
$\mutualinfo(\imagestruct,\textstruct)\le \mutualinfo(\image,\desc)$, and makes the auxiliary alignment problem inherently more difficult.
The gradient of $\structloss$ remains informative even after the gradient of
$\originalloss$ flattens, providing persistent optimization signals that steer the model toward semantically coherent and structurally consistent minima.
This behavior is further reflected in the positive cosine similarity between
$\nabla_\theta \originalloss$ and $\nabla_\theta \structloss$ throughout the optimization trajectory, confirming that the two tasks pursue compatible optima in the parameter space.

The consistent observation across all datasets and settings provides a strong empirical support to our theoretical interpretation in Sec.~\ref{sec:theory}: information-reduced auxiliary objectives act as implicit regularizers that introduce controlled gradient diversity, as suggested by analyses in multitask learning, multi-objective optimization, and contrastive representation learning~\cite{MTLasMOO.NIPS18,gradientSGDvariantNIPS20,wu2020mutualinformationcontrastivelearning}. 
Such diversity expands the effective search space and
stabilizes convergence, while preserving optimization compatibility through aligned gradient directions. The behavior of \ourmethod is agnostic to the dataset, and the findings drawn from the representative Sketchy
experiment generalizes to other evaluated datasets.

%%%%%%%%%%%%%%%%%%%%%%%%%%%%%%%%%%%%%%%%%%%%%%%%%%%%%%%%%%%%
% 2. Extended Method Details

\section{Additional implementation details}
\label{supp:extended_method}

%-----------------------------
\subsection{Text Token Extension}
\label{supp:extended_method:length}

To enable base models like CLIP and SigLIP-2 to handle text sequences longer than their original 77 tokens limit, we fully follow the method proposed in Long-CLIP~\cite{zhang2024longclip}. Specifically, we implement a positional encoding extension that expands the maximum sequence length of all models to 248 tokens. This process is completed before any fine-tuning and follows a preserve-interpolate-extrapolate strategy.
First, we keep the first 20 original position vectors unchanged. Then, we expand the middle part by pairwise linear interpolation. Finally, we use linear extrapolation based on the last two original encodings to fill the remaining positions.

It is worth noting that Long-CLIP is originally designed and optimized for the CLIP architecture. We applied the same text token extension mechanism to all models in this study, including both the CLIP and SigLIP-2 baselines and their variants using our proposed losses. Nonetheless, we do acknowledge that directly applying this extension to SigLIP-2 might not be the optimal way, as SigLIP-2 is trained differently from CLIP. Yet, as there exists no prior work with established techniques for extending SigLIP-2 to long-text scenarios, we opt to this option to ensure experimental consistency and fair comparison conditions.

% \subsection{Adaptive Sliding Window for Text Chunking}
% \label{subsec:adaptive-window}

\subsection{Fine-tuning Details}
\label{supp:finetuningdetails}
We fine-tune the model for $10$ epochs using the AdamW optimizer with a batch size of 16. The initial learning rate is set to $5 \times 10^{-6}$, and we employ a Cosine Annealing scheduler that decays the learning rate to $0$ over the course of fine-tuning, without restarts. The weight decay is set to $0.05$.
For the contrastive objectives, the global temperature parameter $\tau$ is learnable (initialized from the pre-trained CLIP and clamped to a maximum logit scale of 3.5), while the local structure alignment temperature $\gamma$ is fixed at $0.07$ to encourage sharp local correspondences.

\subsection{Computational Analysis}
\label{supp:computationalanalysis}
To evaluate the additional computational overhead introduced by our method, we measure two key preprocessing steps: the local-region segmentation used in $\structlosslocal$ and the edge-map extraction. In practice, we use FastSAM \cite{fastsam} for local segmentation. Due to the large variation in image resolution across datasets, the average per-image inference time also varies: 3.6 ms/image on the lower-resolution Sketchy dataset, 17.1 ms/image on INSECT, and 17.3 ms / 15.9 ms on the higher-resolution DOCCI and DCI datasets, respectively.
For edge extraction, we use the Canny detector, which is lightweight, with timings of 3.77 ms/image on Sketchy, 4.59 ms/image on INSECT, and 19.05 ms/16.18 ms on DOCCI and DCI. Higher image resolutions lead to more processing time. Importantly, both pre-processing steps are executed \textbf{only once} before fine-tuning. 

During fine-tuning, the overall per-batch runtime of our method is approximately 0.17 s, which is comparable to existing CLIP fine-tuning approaches such as Long-CLIP (0.10 s), GOAL (0.18 s), SmartCLIP (0.07 s), and FineLIP (0.05 s).
The inference time is the same as as standard CLIP-based methods.
Overall, \ourmethod introduces very lightweight (and affordable) additional computation, and maintains the same inference efficiency as standard CLIP-based methods.
\begin{table}[t]
\centering
\caption{
\textbf{Comparison of general vs. domain-specific appearance lexicons.}
\inlineColorbox{lightpink}{\emph{Text$\rightarrow$Image}} and 
\inlineColorbox{lightblue}{\emph{Image$\rightarrow$Text}} retrieval on SKETCHY and INSECT.
\textbf{Bold} denotes the best performance.
}

\label{tab:domain-specific-lexion}
\small
\vspace{-4pt}
\setlength{\tabcolsep}{4pt}
\renewcommand{\arraystretch}{1.08}
\resizebox{\columnwidth}{!}{
\begin{tabular}{l ccccc ccccc}
\toprule
\textbf{Setting} &
\cellcolor{lightpink}\textbf{R@1} &
\cellcolor{lightpink}\textbf{R@5} &
\cellcolor{lightpink}\textbf{R@10} &
\cellcolor{lightpink}\textbf{R@25} &
\cellcolor{lightpink}\textbf{R@50} &
\cellcolor{lightblue}\textbf{R@1} &
\cellcolor{lightblue}\textbf{R@5} &
\cellcolor{lightblue}\textbf{R@10} &
\cellcolor{lightblue}\textbf{R@20} &
\cellcolor{lightblue}\textbf{R@50} \\
\midrule
\multicolumn{11}{c}{\cellcolor{gray!15}\textbf{SKETCHY}} \\
\midrule
General Lexicon Filter
& \textbf{69.86} & 90.85 & 95.42 & \textbf{98.36} & \textbf{99.22}
& \textbf{68.22} & \textbf{90.67} & \textbf{95.68} & \textbf{97.75} & \textbf{99.22} \\
Domain-Specific Lexicon Filter 
& 69.00 & \textbf{90.87} & \textbf{95.60} & 98.33 & \textbf{99.22}
& 68.14 & 90.53 & 95.21 & 97.40 & 99.10 \\
\midrule
\multicolumn{11}{c}{\cellcolor{gray!15}\textbf{INSECT}} \\
\midrule
General Lexicon Filter
& \textbf{9.93} & 26.60 & 38.34 & \textbf{56.99} & \textbf{69.34}
& \textbf{9.50} & 26.60 & 39.64 & \textbf{54.92} & \textbf{68.65} \\
Domain-Specific Lexicon Filter 
& 9.67 & \textbf{28.50} & \textbf{39.46} & 55.09 & 67.70
& 9.07 & \textbf{28.07} & \textbf{39.81} & 54.32 & 68.48 \\
\bottomrule
\end{tabular}
}
\vspace{-6pt}
\end{table}
\section{Additional Experimental Analyses}
\label{supp:more_exp}
\subsection{Ablation on Appearance Lexicons}
\label{supp:more_domainspecificlexicon}
We also conducted the ablation investigating \textit{whether constructing dataset-specific appearance vocabularies would lead to better performance, compared to a general lexicon.} To this end, we built two domain-specific appearance lexicons for the fashion dataset SKETCHY and the biological dataset INSECT. For each dataset, examples shown in Fig.~\ref{fig:domain-specific}, instructing the LLM to generate appearance terms particularly relevant to that domain. We then substituted the general $\vocappearance$ with these domain-specific vocabularies and re-trained our method under exactly the same fine-tuning configuration.

From Tab.~\ref{tab:domain-specific-lexion}, we observe that across both the fashion domain (SKETCHY) and the biological domain (INSECT), the domain-specific lexicons yields on-par performance, compared to the general appearance lexicon on most retrieval metrics. This is because the general lexicon already has a \textit{broad coverage} over appearance-related attributes. Instead, the domain-specific lexicon, despite being more specialized, adds only a small number of additional terms into the vocabulary, offering negligible marginal benefit in practice. 
For this reason, in all our experiments, the appearance vocabulary $\vocappearance$ used in the lexicon filter is obtained from the general prompt described in Sec.~\ref{supp:lexion}.1, and the same $\vocappearance$ is applied across all datasets. 
%In summary, the general lexicon demonstrates robust and reliable performance across domains, making it a preferable choice for all our experiments.

\subsection{Additional Cross-domain Evaluation}\label{supp:more_exp:cross_domain}
Complementing to the cross-domain evaluation in the main paper, \Cref{tab:general_specific_generalization} reports the cross-domain evaluation from General (DOCCI)
to Specific (SKETCHY). When models are fine-tuned on the general-domain dataset DOCCI and tested on the fashion-focused Sketchy dataset, all methods exhibit a clear performance drop due to the large domain gap. DOCCI contains diverse real-world scenes with dense captions, whereas Sketchy mainly features fashion items. Despite this strong domain shift, \ourmethod consistently achieves the best cross-domain performance across all Recall@K metrics.
\input{tables/supp_crossdomain}

\subsection{Results at Deeper Ranks}\label{supp:more_exp:deeprank}
\input{tables/supp_crossmodal}
\input{tables/supp_plugandplay}
\input{tables/supp_dataefficiency}

We report the full experimental analyses with Recall@K also at deeper ranks (K = 25, 50) on all four datasets. 

\Cref{tab:deep_recall_all} shows the cross-modal retrieval at full ranks up to K=50. Our method consistently maintains positive margins over the strongest competitors, even at deeper ranks.

\Cref{tab:finetuning_agnostic_supp} reports the Recall@K at full ranks for the plug-and-play effectiveness of $\ourloss$ onto different finetuning methods. At top ranks (R@1/R@5/R@10), $\ourloss$ yields consistent and often substantial improvements, with particularly large gains for lightweight or parameter-constrained methods such as FineLIP, LoRA, and DoRA (improvements ranging from 5\% to 18\%). On deeper ranks (R@25 and R@50), $\ourloss$ continues to provide stable and uniform benefits across nearly all model–dataset combinations. For DOCCI and DCI, the baselines already achieve extremely high R@50 ($>$97\%), yet $\ourloss$ still contributes an additional 0.1–1.4\% improvement. 

Finally, \cref{tab:data_efficiency_5} reports the data efficiency analysis at full ranks. We observe that \ourmethod consistently shows strong data efficiency across all fine-tuning dataset sizes (5\%, 20\%, 50\%). In the most challenging 5\% low-data regime, \ourmethod achieves the best overall results. When the data size increases to 20\%, all methods improve, yet \ourmethod remains the best performer with noticeable gains on most metrics. At 50\% data, although the baselines begin to saturate, \ourmethod still delivers the best results, providing 1–5\% improvements on R@1 and R@5, and enhancing deeper-rank performance (R@25 / R@50) on DOCCI and DCI. Overall, \ourmethod demonstrates strong generalization and high data efficiency across all four datasets and all data-scale settings.

\subsection{More Qualitative Results}\label{supp:more_exp:qualitatives}

\Cref{fig:qualitative_result} presents additional qualitative results of \ourmethod and the second-best method GOAL on all four datasets. Specifically, we highlight with color some parts of the long texts that are describing some visual objects, and showcase the attention maps between such object-centric texts on the image. It is clear that \ourmethod, compared to the second-best method GOAL, demonstrates a better correspondence between the visual objects and their rich textual descriptions. For instance, in the DOCCI example, given the textual description regarding ``the daisy'', \ourmethod produces more visible attention map on the daisy petals compared to GOAL, and it also yields more localized attention map regarding the ``tree line''. Similar patterns can be observed in the DCI samples too, where \ourmethod is more capable in capturing accurate and localized attention map close to the visual object being described. On the specific-domain Sketchy dataset, while \ourmethod exhibits a better coverage on the described visual part, overall both \ourmethod and GOAL are able to capture the correct object. We hypothesize that this might be due to the fact Sketchy dataset contains mostly object-centric, where the model wearing outfits are dominant and mostly centered, which can be a bias that models can easily capture. On the other hand, the INSECT dataset is less represented by pre-trained VLMs, thus their attention map are generally less accurate and localized. Yet, \ourmethod still demonstrates a better alignment between the rich text and the visual counterpart, as evidenced by the more localized attention map on the ``antennae of females'' and the ``thorax'' compared to GOAL.

\begin{figure*}
    \centering
    \includegraphics[width=1.0\linewidth]{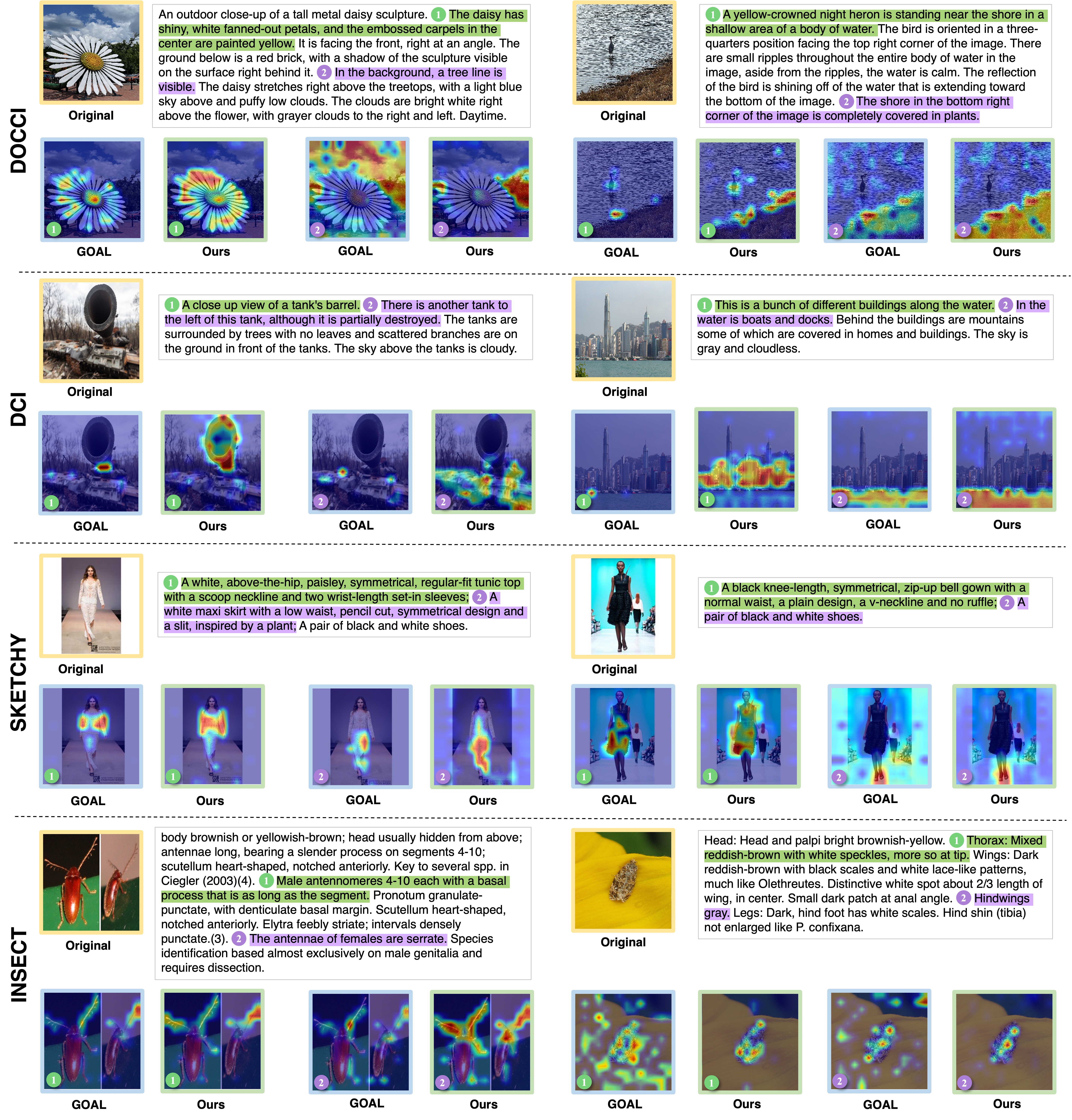}
    %\vspace{-5pt}
    \caption{Qualitative results of \ourmethod and the second-best method GOAL on all four datasets. We present the attention maps between the color-highlighted texts on the visual counterpart. Compared to GOAL, \ourmethod overall shows a better correspondence between the rich textual descriptions and the visual objects.}
    \label{fig:qualitative_result}
\end{figure*}

\begin{figure*}[t]
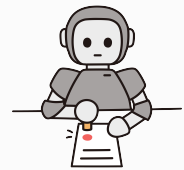

\centering
\begin{tcolorbox}[
    width=0.95\linewidth,
    colback=gray!3,
    colframe=gray,
    boxrule=0.3pt,
    title={(2): Domain-Specific Appearance Lexicon Prompt–Response}
]
\small

% ------------------ USER QUERY ------------------
\begin{minipage}{0.12\linewidth}
\centering
\includegraphics[width=0.8\linewidth]{imgs/supp/user.png}
\end{minipage}
\hfill
\begin{minipage}{0.83\linewidth}
\textbf{USER:}  You are a vision-language expert specializing in \textbf{\textit{fashion}} imagery. Consider a representation of a garment where only edge information is 
preserved (i.e., contours, stitching lines, panel boundaries, silhouettes, 
and geometric layout), with all color, shading, and material cues removed. Please provide a comprehensive list of English words and short phrases 
used in the fashion domain whose meanings cannot be inferred from such an 
edge-only representation. Return the list as comma-separated items.
\end{minipage}

\vspace{0.55em}
\rule{0.9\linewidth}{0.3pt}  % visual separator
\vspace{0.55em}

% ------------------ LLM RESPONSE ------------------
\begin{minipage}{0.80\linewidth}
\textbf{LLM:}  [red, blue, green, yellow, orange, purple, pink, brown, black, white, grey, gray, 
navy, sky blue, baby blue, royal blue, midnight blue, teal, turquoise, aqua, mint, forest green, lime, olive, khaki, sage, emerald, moss, 
maroon, burgundy, wine, cranberry, magenta, fuchsia, rose, blush, salmon, coral, peach, apricot, 
mustard, gold, lemon, amber, 
beige, cream, ivory, off-white, tan, camel, sand, taupe, stone, ecru, 
charcoal, slate, ash, 
lavender, violet, lilac, plum, eggplant, 
silver, bronze, copper, champagne, metallic, 
pastel, neon, bright, dark, light, neutral, 
multicolor, multi-color, multicolored, multi-colored, 
cotton, linen, hemp, ramie, jute, wool, merino, cashmere, mohair, alpaca, angora, silk, mulberry silk, tussah silk, 
denim, chambray, twill, gabardine, canvas, jersey, interlock, french terry, flannel, fleece, felt, 
velvet, velour, corduroy, tweed, boucle, jacquard, brocade, damask, 
viscose, rayon, modal, lyocell, tencel, acetate, cupro, 
polyester, nylon, spandex, elastane, lycra, acrylic, polyamide, microfiber, 
chiffon, georgette, organza, tulle, mesh, crepe, charmeuse, satin, taffeta, batiste, voile, lawn, 
leather, suede, nubuck, shearling, sheepskin, faux leather, vegan leather, faux suede, 
fur, faux fur, mink, fox fur, rabbit fur, 
knit, rib knit, cable knit, fine knit, 
softshell, hardshell, neoprene, Gore-Tex, ripstop, mesh, spacer fabric, 
lamé, bonded fabric, coated fabric, sherpa, denier, pongee, jacquard knit]
\end{minipage}
\hfill
\begin{minipage}{0.18\linewidth}
\centering
\includegraphics[width=0.8\linewidth]{imgs/supp/robot.png}  % robot icon
\end{minipage}
% \vspace{0.95em}
  % visual separator
\vspace{0.55em}
\rule{0.9\linewidth}{0.3pt}
\begin{minipage}{0.12\linewidth}
\centering
\includegraphics[width=0.8\linewidth]{imgs/supp/user.png}  % robot icon % user icon
\end{minipage}
\hfill
\begin{minipage}{0.83\linewidth}
\textbf{USER}: You are an \textit{\textbf{entomology}}-focused vision–language expert. Consider an \textit{\textbf{insect}} image represented only by its edge map, that is, a 
structure-centric depiction containing contours, boundaries, wing venation 
geometry, and overall body segmentation, but with all color, shading, 
pigmentation, and optical surface cues removed. Please produce a comprehensive list of words and short phrases 
used in insect field guides or taxonomic descriptions whose meanings 
cannot be inferred from such an edge-only representation. These should 
strictly correspond to appearance attributes such as color terms, 
pigmentation qualities, optical reflectance (e.g., iridescence, metallic 
sheen), translucency/opacity, and other non-geometric visual properties. Return the list as comma-separated items.
\end{minipage}

\vspace{0.55em}
\rule{0.9\linewidth}{0.3pt}  % visual separator
\vspace{0.55em}

% ------------------ LLM RESPONSE ------------------
\begin{minipage}{0.80\linewidth}
\textbf{LLM:}  [black, brown, dark brown, light brown, tan, beige, cream,
white, off-white, gray, grey, charcoal, slate, ash,
red, reddish, orange, yellow, green, blue, purple,
pink, magenta, violet,
rust, russet, chestnut, mahogany, clay, ochre, umber,
sienna, tawny, fawn,
amber, honey-colored, golden, bronze, coppery,
pale, dusky, washed-out, faded, dim, drab,
rufous, testaceous, fulvous, ferruginous,
castaneous, fuscous, livid, piceous,
violaceous, cyaneous, glaucous,
mottled, blotchy, flecked, freckled, stained, tinged, tinted, smudged,
clouded, diffuse, suffused, irregularly pigmented, unevenly pigmented,
faintly pigmented, deeply pigmented, melanized, depigmented, discolored,
frosted, pruinose, powdery, mealy, chalky, dusty, granular, velvety,
shaded, darkened, lightened, somber, sooty, smoky, smeared,
iridescent, metallic, submetallic, opalescent, pearlescent,
rainbowlike, prismatic, lustrous, holographic,
shiny, glossy, subglossy, dull, matte, satiny, silky, polished,
reflective, non-reflective, sheeny, mirrorlike,
transparent, translucent, semi-translucent, opaque,
hyaline, subhyaline, smoky-hyaline,
warm-toned, cool-toned, earthy, vivid, dull-colored,
bright, pale-colored, dark-colored,
satiny, silken, glassy, resinous, lacquered, oily,
greasy, waxy, glistening, gleaming]
\end{minipage}
\hfill
\begin{minipage}{0.18\linewidth}
\centering
\includegraphics[width=0.8\linewidth]{imgs/supp/robot.png}  % robot icon  % robot icon
\end{minipage}

\end{tcolorbox}

\caption{Prompt and response used in our LLM-based construction of the domain-specific appearance vocabulary lists for lexicon filtering.}
\label{fig:domain-specific}
\end{figure*}

\vspace{-6pt}

%% file: tables/supp_crossdomain.tex
\begin{table}[t!]
\centering
\caption{
\textbf{Cross-domain generalization from General (DOCCI) to Specific (Sketchy).}
Models are trained on the general dense-captioned dataset (DOCCI) and tested on abstract sketches (Sketchy).
Due to the significant domain shift, all models experience a performance drop, but \ourmethod maintains better robustness.
Values are Recall@K (\%). In-domain results (DOCCI$\to$DOCCI) are in \textit{italic} for reference. Best cross-domain results are in \textbf{bold}.
}
\vspace{-5pt}
\label{tab:general_specific_generalization}
\scriptsize
\setlength{\tabcolsep}{3.6pt}
\renewcommand{\arraystretch}{1.1}

\begin{tabular*}{\columnwidth}{@{\extracolsep{\fill}}lccc|ccc@{}}
\toprule
\textbf{Setting} &
\cellcolor{lightpink}\textbf{R@1} &
\cellcolor{lightpink}\textbf{R@5} &
\cellcolor{lightpink}\textbf{R@10} &
\cellcolor{lightblue}\textbf{R@1} &
\cellcolor{lightblue}\textbf{R@5} &
\cellcolor{lightblue}\textbf{R@10} \\
\midrule

\multicolumn{7}{c}{\cellcolor{gray!10}\textbf{Fine-tune on DOCCI (General) → Test on Sketchy (Specific)}} \\
\midrule
% --- Long-CLIP ---
Long-CLIP (In-domain) & \textit{64.49} & \textit{87.67} & \textit{93.43} & \textit{63.08} & \textit{87.45} & \textit{93.14} \\
Long-CLIP (Cross-domain) & 8.12 & 20.03 & 28.24 & 10.71 & 24.27 & 33.68 \\

\lightmidrule
% --- GOAL ---
GOAL (In-domain) & \textit{79.47} & \textit{96.65} & \textit{98.69} & \textit{79.43} & \textit{96.14} & \textit{97.25} \\
GOAL (Cross-domain) & 8.72 & 19.60 & 27.46 & 9.76 & 26.86 & 38.17 \\

\lightmidrule
% --- ours ---
\textbf{\ourmethod} (In-domain) & \textbf{\textit{83.04}} & \textbf{\textit{97.06}} & \textbf{\textit{98.96}} & \textbf{\textit{81.59}} & \textbf{\textit{96.94}} & \textbf{\textit{98.78}} \\
\textbf{\ourmethod} (Cross-domain) & \textbf{8.96} & \textbf{21.85} & \textbf{30.57} & \textbf{12.26} & \textbf{32.21} & \textbf{42.23} \\
\bottomrule
\end{tabular*}
\vspace{-8pt}
\end{table}
%%%%%%%%%%%%%%%%%%%%%%%%%%%%%%%%%%%%%%%%%%%%%%%%%%%%%%%%%%%%

%% file: tables/supp_crossmodal.tex
\begin{table*}[t!]
\centering
\caption{
\textbf{Crossmodal retrieval at full ranks on SKETCHY, INSECT, DOCCI and DCI.}
We report mean Recall@K (\%) $\pm$ standard deviation for K = 1, 5, 10, 25, 50 on both
\inlineColorbox{lightpink}{\emph{Text$\rightarrow$Image}} and
\inlineColorbox{lightblue}{\emph{Image$\rightarrow$Text}}. All results are averaged over three independent runs with random seeds 42, 1337 and 3407. \textbf{Bold} indicates the best result, while \underline{underline} denotes the second-best.
}
\label{tab:deep_recall_all}
\vspace{-5pt}
\footnotesize
\setlength{\tabcolsep}{3pt}

% ===================== SKETCHY =====================
\resizebox{\textwidth}{!}{
\begin{tabular}{@{} l c c c c c | c c c c c @{}}
\toprule
\multicolumn{11}{c}{\cellcolor{panelgray}\textbf{SKETCHY}} \\
\midrule
\multirow{1}{*}{Method} &
\cellcolor{lightpink}\textbf{R@1} &
\cellcolor{lightpink}\textbf{R@5} &
\cellcolor{lightpink}\textbf{R@10} &
\cellcolor{lightpink}\textbf{R@25} &
\cellcolor{lightpink}\textbf{R@50} &
\cellcolor{lightblue}\textbf{R@1} &
\cellcolor{lightblue}\textbf{R@5} &
\cellcolor{lightblue}\textbf{R@10} &
\cellcolor{lightblue}\textbf{R@25} &
\cellcolor{lightblue}\textbf{R@50} \\
\midrule
Long-CLIP{\textcolor{gray}{[ECCV'24]}}
& 54.32$\pm$0.46 & 80.14$\pm$2.27 & 88.43$\pm$1.88 & 95.25$\pm$0.78 & 98.27$\pm$0.75
& 52.76$\pm$1.18 & 80.31$\pm$1.89 & 88.08$\pm$1.50 & 95.16$\pm$1.03 & 97.75$\pm$1.00 \\
FineLIP{\textcolor{gray}{[CVPR'25]}}
& 40.59$\pm$1.91 & 71.16$\pm$1.28 & 81.78$\pm$0.63 & 91.27$\pm$1.72 & 95.94$\pm$0.57
& 40.33$\pm$1.04 & 72.11$\pm$0.05 & 82.38$\pm$0.11 & 91.45$\pm$1.85 & 95.77$\pm$0.69 \\
SmartCLIP{\textcolor{gray}{[CVPR'25]}}
& 50.73$\pm$1.11 & 81.09$\pm$0.93 & \underline{94.56$\pm$1.67} & 96.11$\pm$0.44 & \underline{99.05$\pm$1.28}
& 51.30$\pm$0.52 & 80.83$\pm$1.89 & \underline{94.04$\pm$0.73} & 95.51$\pm$1.21 & 98.96$\pm$0.36 \\
GOAL{\textcolor{gray}{[CVPR'25]}}
& \underline{63.21$\pm$0.47} & \underline{87.13$\pm$1.58} & 93.44$\pm$0.35 & \underline{97.67$\pm$1.09} & \underline{99.05$\pm$0.19}
& \underline{62.44$\pm$1.37} & \underline{87.82$\pm$0.95} & 92.31$\pm$1.44 & \underline{96.98$\pm$0.29} & \underline{99.00$\pm$0.54} \\
\textbf{\ourmethod}
& \textbf{69.86$\pm$0.46} & \textbf{90.85$\pm$0.09} & \textbf{95.42$\pm$0.07} & \textbf{98.61$\pm$0.35} & \textbf{99.22$\pm$0.00}
& \textbf{68.22$\pm$0.45} & \textbf{90.67$\pm$0.12} & \textbf{95.68$\pm$0.13} & \textbf{98.03$\pm$0.39} & \textbf{99.08$\pm$0.20} \\
\bottomrule
\end{tabular}
}
\vspace{0.6em}

% ===================== INSECT =====================
\resizebox{\textwidth}{!}{
\begin{tabular}{@{} l c c c c c | c c c c c @{}}
\toprule
\multicolumn{11}{c}{\cellcolor{panelgray}\textbf{INSECT}} \\
\midrule
\multirow{1}{*}{Method} &
\cellcolor{lightpink}\textbf{R@1} &
\cellcolor{lightpink}\textbf{R@5} &
\cellcolor{lightpink}\textbf{R@10} &
\cellcolor{lightpink}\textbf{R@25} &
\cellcolor{lightpink}\textbf{R@50} &
\cellcolor{lightblue}\textbf{R@1} &
\cellcolor{lightblue}\textbf{R@5} &
\cellcolor{lightblue}\textbf{R@10} &
\cellcolor{lightblue}\textbf{R@25} &
\cellcolor{lightblue}\textbf{R@50} \\
\midrule
Long-CLIP{\textcolor{gray}{[ECCV'24]}}
& 8.20$\pm$0.87 & 23.83$\pm$0.68 & 34.97$\pm$0.25 & 51.21$\pm$0.31 & 63.30$\pm$0.27
& \underline{9.41$\pm$0.84} & 24.78$\pm$0.62 & \underline{37.31$\pm$0.88} & \underline{53.18$\pm$0.47} & 63.39$\pm$0.98 \\
FineLIP{\textcolor{gray}{[CVPR'25]}}
& 8.46$\pm$0.59 & 23.32$\pm$0.81 & 33.59$\pm$0.44 & 51.21$\pm$1.93 & 66.58$\pm$0.69
& 6.86$\pm$0.74 & 23.75$\pm$1.41 & 34.46$\pm$0.53 & 52.42$\pm$1.26 & 65.37$\pm$0.88 \\
SmartCLIP{\textcolor{gray}{[CVPR'25]}}
& 4.84$\pm$0.66 & 16.84$\pm$1.52 & 34.63$\pm$0.37 & 39.03$\pm$1.18 & 57.60$\pm$0.44
& 4.66$\pm$0.91 & 15.46$\pm$0.33 & 34.02$\pm$1.74 & 39.72$\pm$0.85 & 58.38$\pm$1.29 \\
GOAL{\textcolor{gray}{[CVPR'25]}}
& \underline{8.81$\pm$0.07} & \underline{24.35$\pm$0.99} & \underline{35.84$\pm$1.62} & \underline{55.44$\pm$0.72} & \underline{67.46$\pm$0.91}
& 8.55$\pm$0.38 & \underline{25.91$\pm$0.36} & 36.18$\pm$0.66 & 53.02$\pm$1.14 & \underline{66.41$\pm$0.52} \\
\textbf{\ourmethod}
& \textbf{9.93$\pm$0.90} & \textbf{26.60$\pm$1.20} & \textbf{38.34$\pm$1.02} & \textbf{56.99$\pm$0.61} & \textbf{69.34$\pm$0.90}
& \textbf{9.50$\pm$0.80} & \textbf{26.60$\pm$0.29} & \textbf{39.64$\pm$0.64} & \textbf{54.92$\pm$0.46} & \textbf{68.65$\pm$0.67} \\
\bottomrule
\end{tabular}
}
\vspace{0.6em}

% ===================== DOCCI =====================
\resizebox{\textwidth}{!}{
\begin{tabular}{@{} l c c c c c | c c c c c @{}}
\toprule

\multicolumn{11}{c}{\cellcolor{panelgray}\textbf{DOCCI}} \\
\midrule
\multirow{1}{*}{Method} &
\cellcolor{lightpink}\textbf{R@1} &
\cellcolor{lightpink}\textbf{R@5} &
\cellcolor{lightpink}\textbf{R@10} &
\cellcolor{lightpink}\textbf{R@25} &
\cellcolor{lightpink}\textbf{R@50} &
\cellcolor{lightblue}\textbf{R@1} &
\cellcolor{lightblue}\textbf{R@5} &
\cellcolor{lightblue}\textbf{R@10} &
\cellcolor{lightblue}\textbf{R@25} &
\cellcolor{lightblue}\textbf{R@50} \\
\midrule
Long-CLIP{\textcolor{gray}{[ECCV'24]}}
& 64.49$\pm$0.38 & 87.67$\pm$0.43 & 93.43$\pm$0.38 & 97.73$\pm$0.26 & 99.14$\pm$0.11
& 63.08$\pm$0.32 & 87.45$\pm$0.44 & 93.14$\pm$0.30 & 97.45$\pm$0.17 & 99.02$\pm$0.12 \\
FineLIP{\textcolor{gray}{[CVPR'25]}}
& 67.80$\pm$1.28 & 90.22$\pm$0.56 & 94.84$\pm$0.31 & 98.22$\pm$1.49 & 99.45$\pm$0.52
& 66.39$\pm$0.44 & 89.12$\pm$1.22 & 94.47$\pm$0.67 & 97.90$\pm$0.38 & 99.20$\pm$0.93 \\
SmartCLIP{\textcolor{gray}{[CVPR'25]}}
& 74.92$\pm$0.66 & 94.08$\pm$0.29 & 97.31$\pm$1.12 & 99.37$\pm$0.44 & 99.82$\pm$0.18
& 74.91$\pm$0.53 & 94.04$\pm$0.72 & 97.29$\pm$0.36 & 99.32$\pm$0.91 & 99.84$\pm$0.27 \\
GOAL{\textcolor{gray}{[CVPR'25]}}
& \underline{79.47$\pm$0.41} & \underline{96.65$\pm$1.33} & \underline{98.69$\pm$0.02} & \underline{99.69$\pm$0.57} & \underline{99.92$\pm$0.14}
& \underline{79.43$\pm$0.88} & \underline{96.14$\pm$0.38} & \underline{97.25$\pm$0.73} & \underline{99.61$\pm$1.12} & \underline{99.90$\pm$0.19} \\
\textbf{\ourmethod}
& \textbf{83.04$\pm$0.05} & \textbf{97.06$\pm$0.20} & \textbf{98.96$\pm$0.04} & \textbf{99.84$\pm$0.02} & \textbf{99.98$\pm$0.01}
& \textbf{81.59$\pm$0.34} & \textbf{96.94$\pm$0.04} & \textbf{98.78$\pm$0.03} & \textbf{99.76$\pm$0.03} & \textbf{99.92$\pm$0.02} \\
\bottomrule
\end{tabular}
}
\vspace{0.6em}

% ===================== DCI =====================
\resizebox{\textwidth}{!}{
\begin{tabular}{@{} l c c c c c | c c c c c @{}}
\toprule
\multicolumn{11}{c}{\cellcolor{panelgray}\textbf{DCI}} \\
\midrule
\multirow{1}{*}{Method} &
\cellcolor{lightpink}\textbf{R@1} &
\cellcolor{lightpink}\textbf{R@5} &
\cellcolor{lightpink}\textbf{R@10} &
\cellcolor{lightpink}\textbf{R@25} &
\cellcolor{lightpink}\textbf{R@50} &
\cellcolor{lightblue}\textbf{R@1} &
\cellcolor{lightblue}\textbf{R@5} &
\cellcolor{lightblue}\textbf{R@10} &
\cellcolor{lightblue}\textbf{R@25} &
\cellcolor{lightblue}\textbf{R@50} \\
\midrule
Long-CLIP{\textcolor{gray}{[ECCV'24]}}
& 59.23$\pm$0.72 & 80.89$\pm$0.55 & 87.04$\pm$0.48 & 92.60$\pm$0.31 & 95.10$\pm$0.27
& 60.13$\pm$0.69 & 81.44$\pm$0.52 & 87.54$\pm$0.46 & 92.85$\pm$0.33 & 95.60$\pm$0.30 \\
FineLIP{\textcolor{gray}{[CVPR'25]}}
& 66.13$\pm$1.77 & 85.34$\pm$0.44 & 89.79$\pm$1.06 & 94.14$\pm$0.59 & 96.35$\pm$0.52
& 64.58$\pm$0.89 & 84.59$\pm$0.71 & 89.54$\pm$1.68 & 94.00$\pm$0.34 & 96.40$\pm$1.08 \\
SmartCLIP{\textcolor{gray}{[CVPR'25]}}
& 69.88$\pm$1.41 & 86.64$\pm$0.63 & \underline{94.05$\pm$1.18} & 95.00$\pm$0.31 & \underline{97.25$\pm$0.77}
& 70.94$\pm$0.52 & 87.04$\pm$1.29 & 92.77$\pm$0.84 & 95.75$\pm$0.48 & 97.05$\pm$1.03 \\
GOAL{\textcolor{gray}{[CVPR'25]}}
& \underline{72.64$\pm$0.55} & \underline{89.89$\pm$1.22} & 93.70$\pm$0.33 & 95.75$\pm$0.66 & \underline{97.25$\pm$0.41}
& \underline{72.84$\pm$1.11} & \textbf{\underline{90.50$\pm$0.40}} & \underline{93.20$\pm$0.81} & \underline{96.60$\pm$0.57} & \underline{97.60$\pm$0.22} \\
\textbf{\ourmethod}
& \textbf{75.90$\pm$0.50} & \textbf{90.00$\pm$0.40} & \textbf{95.15$\pm$0.39} & \underline{\textbf{95.95$\pm$0.13}} & \textbf{97.85$\pm$0.15}
& \textbf{74.39$\pm$0.16} & \underline{89.90$\pm$0.05} & \textbf{94.30$\pm$0.23} & \textbf{96.75$\pm$0.23} & \textbf{97.75$\pm$0.10} \\
\bottomrule
\end{tabular}
}
\vspace{0.5em}
\end{table*}

%% file: tables/supp_plugandplay.tex
%%%%%PLUG-IN-TABLE
\begin{table*}[t!]
\caption{
\textbf{Plug-and-play enhancement of our $\ourloss$ on CLIP-based finetuning.}
Results on SKETCHY, INSECT, DOCCI, and DCI for \inlineColorbox{lightpink}{\emph{Text$\rightarrow$Image}} and \inlineColorbox{lightblue}{\emph{Image$\rightarrow$Text}} retrieval.
We report R@1, R@5, R@10, R@25, and R@50.
Upper: SKETCHY and INSECT; Lower: DOCCI and DCI.
Best in \textbf{bold}, with gain in \gain{green}.
}
\vspace{-5pt}
\label{tab:finetuning_agnostic_supp}
\centering
\setlength{\tabcolsep}{1.5pt} 

\resizebox{\textwidth}{!}{%
\begin{tabular}{@{} l 
% Group 1 T2I (5 cols)
*{4}{S[table-format=2.2] !{\color{lightlightgray}\vrule width 0.5pt}} S[table-format=2.2] !{\color{lightlightgray}\vrule width 0.5pt}
% Group 1 I2T (5 cols)
*{4}{S[table-format=2.2] !{\color{lightlightgray}\vrule width 0.5pt}} S[table-format=2.2] | 
% Group 2 T2I (5 cols)
*{4}{S[table-format=2.2] !{\color{lightlightgray}\vrule width 0.5pt}} S[table-format=2.2] !{\color{lightlightgray}\vrule width 0.5pt}
% Group 2 I2T (5 cols)
*{4}{S[table-format=2.2] !{\color{lightlightgray}\vrule width 0.5pt}} S[table-format=2.2] 
@{}}

\toprule
% ================= TOP SECTION: SKETCHY & INSECT =================
\multirow{2}{*}{\textbf{Method}} &
\multicolumn{10}{c|}{\cellcolor{gray!15}\textbf{SKETCHY}} &
\multicolumn{10}{c}{\cellcolor{gray!15}\textbf{INSECT}} \\
\cmidrule(lr){2-11}\cmidrule(lr){12-21}
& \multicolumn{1}{c}{\cellcolor{lightpink}{R@1}} & \multicolumn{1}{c}{\cellcolor{lightpink}{R@5}} & \multicolumn{1}{c}{\cellcolor{lightpink}{R@10}} & \multicolumn{1}{c}{\cellcolor{lightpink}{R@25}} & \multicolumn{1}{c}{\cellcolor{lightpink}{R@50}}
& \multicolumn{1}{c}{\cellcolor{lightblue}{R@1}} & \multicolumn{1}{c}{\cellcolor{lightblue}{R@5}} & \multicolumn{1}{c}{\cellcolor{lightblue}{R@10}} & \multicolumn{1}{c}{\cellcolor{lightblue}{R@25}} & \multicolumn{1}{c|}{\cellcolor{lightblue}{R@50}}
& \multicolumn{1}{c}{\cellcolor{lightpink}{R@1}} & \multicolumn{1}{c}{\cellcolor{lightpink}{R@5}} & \multicolumn{1}{c}{\cellcolor{lightpink}{R@10}} & \multicolumn{1}{c}{\cellcolor{lightpink}{R@25}} & \multicolumn{1}{c}{\cellcolor{lightpink}{R@50}}
& \multicolumn{1}{c}{\cellcolor{lightblue}{R@1}} & \multicolumn{1}{c}{\cellcolor{lightblue}{R@5}} & \multicolumn{1}{c}{\cellcolor{lightblue}{R@10}} & \multicolumn{1}{c}{\cellcolor{lightblue}{R@25}} & \multicolumn{1}{c}{\cellcolor{lightblue}{R@50}} \\
\midrule

% --- Long-CLIP ---
Long-CLIP
& 54.32 & 80.14 & 88.43 & 95.25 & 98.27 
& 52.76 & 80.31 & 88.08 & 95.16 & 97.75
& 8.20 & 23.83 & 34.97 & 51.21 & 63.30 
& \textbf{9.41} & 24.78 & 37.31 & 53.18 & 63.39 \\
\textbf{+our $\ourloss$}
& \textbf{59.24} & \textbf{85.32} & \textbf{91.45} & \textbf{96.20} & \textbf{98.53} 
& \textbf{59.59} & \textbf{84.37} & \textbf{91.02} & \textbf{96.29} & \textbf{98.45}
& \textbf{9.24} & \textbf{25.39} & \textbf{36.36} & \textbf{52.33} & \textbf{66.84} 
& 9.38 & \textbf{27.29} & \textbf{38.60} & \textbf{55.35} & \textbf{66.15} \\
\textit{$\Delta$}
& \gain{4.92} & \gain{5.18} & \gain{3.02} & \gain{0.95} & \gain{0.26}
& \gain{6.83} & \gain{4.06} & \gain{2.94} & \gain{1.13} & \gain{0.70}
& \gain{1.04} & \gain{1.56} & \gain{1.39} & \gain{1.12} & \gain{3.54}
& \loss{0.03} & \gain{2.51} & \gain{1.29} & \gain{2.17} & \gain{2.76} \\
\lightmidrule

% --- FineLIP ---
FineLIP
& 40.59 & 71.16 & 81.78 & 91.27 & 95.94 
& 40.33 & 72.11 & 82.38 & 91.45 & 95.77
& 8.46 & 23.32 & 33.59 & 51.21 & 66.58 
& \textbf{6.86} & 23.75 & 34.46 & 52.42 & 65.37 \\
\textbf{+our $\ourloss$}
& \textbf{59.15} & \textbf{85.23} & \textbf{91.28} & \textbf{96.63} & \textbf{99.40}
& \textbf{58.55} & \textbf{84.54} & \textbf{90.07} & \textbf{96.20} & \textbf{99.05}
& \textbf{8.89} & \textbf{23.83} & \textbf{35.75} & \textbf{53.63} & \textbf{67.18}
& 6.56 & \textbf{23.76} & \textbf{36.01} & \textbf{53.45} & \textbf{66.84} \\
\textit{$\Delta$}
& \gain{18.56} & \gain{14.07} & \gain{9.50} & \gain{5.36} & \gain{3.46}
& \gain{18.22} & \gain{12.43} & \gain{7.69} & \gain{4.75} & \gain{3.28}
& \gain{0.43} & \gain{0.51} & \gain{2.16} & \gain{2.42} & \gain{0.60}
& \loss{0.30} & \gain{0.01} & \gain{1.55} & \gain{1.03} & \gain{1.47} \\
\lightmidrule

% --- SmartCLIP ---
SmartCLIP
& 50.73 & 81.09 & 94.56 & 96.11 & 99.05 
& 51.30 & 80.83 & \textbf{94.04} & 95.51 & 98.96
& 4.84 & 16.84 & 34.63 & 39.03 & 57.60 
& 4.66 & 15.46 & 34.02 & 39.72 & 58.38 \\
\textbf{+our $\ourloss$}
& \textbf{52.94} & \textbf{81.26} & \textbf{94.77} & \textbf{96.13} & \textbf{99.20}
& \textbf{52.33} & \textbf{80.92} & \textbf{94.04} & \textbf{95.60} & \textbf{98.96}
& \textbf{5.61} & \textbf{16.89} & \textbf{34.80} & \textbf{40.16} & \textbf{60.61}
& \textbf{5.18} & \textbf{15.80} & \textbf{34.07} & \textbf{39.72} & \textbf{59.20} \\
\textit{$\Delta$}
& \gain{2.21} & \gain{0.17} & \gain{0.21} & \gain{0.02} & \gain{0.15}
& \gain{1.03} & \gain{0.09} & \neutral{0.00} & \gain{0.09} & \neutral{0.00}
& \gain{0.77} & \gain{0.05} & \gain{0.17} & \gain{1.13} & \gain{3.01}
& \gain{0.52} & \gain{0.34} & \gain{0.05} & \neutral{0.00} & \gain{0.82} \\
\lightmidrule

% --- GOAL ---
GOAL
& 63.21 & 87.13 & 93.44 & 97.67 & 99.05 
& 62.44 & 87.82 & 92.31 & 96.98 & 99.00
& \textbf{8.81} & 24.35 & 35.84 & 55.44 & 67.46 
& 8.55 & 25.91 & 36.18 & 53.02 & 66.41 \\
\textbf{+our $\ourloss$}
& \textbf{67.88} & \textbf{90.33} & \textbf{95.16} & \textbf{98.53} & \textbf{99.65}
& \textbf{68.48} & \textbf{89.81} & \textbf{94.82} & \textbf{98.27} & \textbf{99.31}
& \textbf{8.81} & \textbf{28.07} & \textbf{38.36} & \textbf{56.65} & \textbf{68.83}
& \textbf{8.75} & \textbf{27.12} & \textbf{39.21} & \textbf{54.49} & \textbf{68.05} \\
\textit{$\Delta$}
& \gain{4.67} & \gain{3.20} & \gain{1.72} & \gain{0.86} & \gain{0.60}
& \gain{6.04} & \gain{1.99} & \gain{2.51} & \gain{1.29} & \gain{0.31}
& \neutral{0.00} & \gain{3.72} & \gain{2.52} & \gain{1.21} & \gain{1.37}
& \gain{0.20} & \gain{1.21} & \gain{3.03} & \gain{1.47} & \gain{1.64} \\
\lightmidrule

SigLIP2

& 68.91 & 90.85 & 95.16 & \textbf{98.79} & 99.57 & 66.75 & 90.24 & 93.96 & \textbf{98.10} & 99.31
& 6.37 & 21.07 & 31.87 & 48.19 & 60.36  & 6.56 & 21.85 & 31.95 & 49.22 & 60.19  \\
\textbf{+our $\ourloss$}
& \textbf{73.49} & \textbf{91.97} & \textbf{95.77} & \textbf{98.79} & \textbf{99.60} & \textbf{70.38} & \textbf{91.54} & \textbf{95.77} & \textbf{98.10} & \textbf{99.35} & \textbf{7.69} & \textbf{23.49} & \textbf{33.42} & \textbf{50.52} & \textbf{63.04} & \textbf{6.99} & \textbf{23.66} & \textbf{35.06} & \textbf{52.16} & \textbf{63.21} \\
\textit{$\Delta$}
& \gain{4.58} & \gain{1.12} & \gain{0.61} & 0.00 & \gain{0.03} & \gain{3.63} & \gain{1.30} & \gain{1.81} & 0.00 & \gain{0.04} & \gain{1.32} & \gain{2.42} & \gain{1.55} & \gain{2.33} & \gain{2.68} & \gain{0.43} & \gain{1.81} & \gain{3.11} & \gain{2.94} & \gain{3.02} \\
\midrule
% --- LoRA ---
LoRA
& 57.08 & 84.72 & 91.80 & 96.46 & 98.70 
& 56.74 & 85.32 & 91.71 & 96.29 & 98.36
& 5.79 & 16.84 & 25.82 & 42.49 & 56.22 
& 4.84 & 16.49 & 24.96 & 40.33 & 56.30 \\
\textbf{+our $\ourloss$}
& \textbf{62.09} & \textbf{86.79} & \textbf{93.95} & \textbf{97.32} & \textbf{99.05}
& \textbf{59.41} & \textbf{85.92} & \textbf{92.75} & \textbf{97.24} & \textbf{98.97}
& \textbf{5.87} & \textbf{18.31} & \textbf{27.63} & \textbf{44.99} & \textbf{59.67}
& \textbf{5.32} & \textbf{18.65} & \textbf{28.41} & \textbf{44.82} & \textbf{56.99} \\
\textit{$\Delta$}
& \gain{5.01} & \gain{2.07} & \gain{2.15} & \gain{0.86} & \gain{0.35}
& \gain{2.67} & \gain{0.60} & \gain{1.04} & \gain{0.95} & \gain{0.61}
& \gain{0.08} & \gain{1.47} & \gain{1.81} & \gain{2.50} & \gain{3.45}
& \gain{0.48} & \gain{2.16} & \gain{3.45} & \gain{4.49} & \gain{0.69} \\
\lightmidrule

% --- DoRA ---
DoRA
& 61.77 & 86.18 & 91.88 & 97.32 & 98.88 
& 60.94 & 87.33 & 92.31 & 96.46 & 98.46
& \textbf{7.17} & 20.21 & 29.88 & 46.20 & 61.92 
& 6.04 & 20.81 & 31.52 & 46.29 & 60.02 \\
\textbf{+our $\ourloss$}
& \textbf{65.20} & \textbf{90.26} & \textbf{94.91} & \textbf{98.01} & \textbf{99.31}
& \textbf{64.94} & \textbf{88.35} & \textbf{93.52} & \textbf{97.58} & \textbf{98.96}
& 7.08 & \textbf{24.01} & \textbf{35.06} & \textbf{52.59} & \textbf{67.01}
& \textbf{8.29} & \textbf{24.35} & \textbf{35.15} & \textbf{52.76} & \textbf{66.58} \\
\textit{$\Delta$}
& \gain{3.43} & \gain{4.08} & \gain{3.03} & \gain{0.69} & \gain{0.43}
& \gain{4.00} & \gain{1.02} & \gain{1.21} & \gain{1.12} & \gain{0.50}
& \loss{0.09} & \gain{3.80} & \gain{5.18} & \gain{6.39} & \gain{5.09}
& \gain{2.25} & \gain{3.54} & \gain{3.63} & \gain{6.47} & \gain{6.56} \\

\midrule[2pt]
% ================= BOTTOM SECTION: DOCCI & DCI =================
\multirow{2}{*}{\textbf{Method}} &
\multicolumn{10}{c|}{\cellcolor{gray!15}\textbf{DOCCI}} &
\multicolumn{10}{c}{\cellcolor{gray!15}\textbf{DCI}} \\
\cmidrule(lr){2-11}\cmidrule(lr){12-21}
& \multicolumn{1}{c}{\cellcolor{lightpink}{R@1}} & \multicolumn{1}{c}{\cellcolor{lightpink}{R@5}} & \multicolumn{1}{c}{\cellcolor{lightpink}{R@10}} & \multicolumn{1}{c}{\cellcolor{lightpink}{R@25}} & \multicolumn{1}{c}{\cellcolor{lightpink}{R@50}}
& \multicolumn{1}{c}{\cellcolor{lightblue}{R@1}} & \multicolumn{1}{c}{\cellcolor{lightblue}{R@5}} & \multicolumn{1}{c}{\cellcolor{lightblue}{R@10}} & \multicolumn{1}{c}{\cellcolor{lightblue}{R@25}} & \multicolumn{1}{c|}{\cellcolor{lightblue}{R@50}}
& \multicolumn{1}{c}{\cellcolor{lightpink}{R@1}} & \multicolumn{1}{c}{\cellcolor{lightpink}{R@5}} & \multicolumn{1}{c}{\cellcolor{lightpink}{R@10}} & \multicolumn{1}{c}{\cellcolor{lightpink}{R@25}} & \multicolumn{1}{c}{\cellcolor{lightpink}{R@50}}
& \multicolumn{1}{c}{\cellcolor{lightblue}{R@1}} & \multicolumn{1}{c}{\cellcolor{lightblue}{R@5}} & \multicolumn{1}{c}{\cellcolor{lightblue}{R@10}} & \multicolumn{1}{c}{\cellcolor{lightblue}{R@25}} & \multicolumn{1}{c}{\cellcolor{lightblue}{R@50}} \\
\midrule

% --- Long-CLIP ---
Long-CLIP
& 64.49 & 87.67 & 93.43 & 97.73 & 99.14 
& 63.08 & 87.45 & 93.14 & 97.45 & 99.02
& 59.23 & 80.89 & 87.04 & 92.60 & 95.10 
& 60.13 & 81.44 & 87.54 & 92.85 & 95.60 \\
\textbf{+our $\ourloss$}
& \textbf{67.67} & \textbf{90.82} & \textbf{95.59} & \textbf{98.43} & \textbf{99.39}
& \textbf{67.92} & \textbf{90.16} & \textbf{95.10} & \textbf{98.22} & \textbf{99.47}
& \textbf{63.13} & \textbf{84.14} & \textbf{89.69} & \textbf{94.20} & \textbf{96.55}
& \textbf{64.33} & \textbf{86.19} & \textbf{89.14} & \textbf{94.20} & \textbf{96.50} \\
\textit{$\Delta$}
& \gain{3.18} & \gain{3.15} & \gain{2.16} & \gain{0.70} & \gain{0.25}
& \gain{4.84} & \gain{2.71} & \gain{1.96} & \gain{0.77} & \gain{0.45}
& \gain{3.90} & \gain{3.25} & \gain{2.65} & \gain{1.60} & \gain{1.45}
& \gain{4.20} & \gain{4.75} & \gain{1.60} & \gain{1.35} & \gain{0.90} \\
\lightmidrule

% --- FineLIP ---
FineLIP
& 67.80 & 90.22 & 94.84 & 98.22 & 99.45 
& 66.39 & 89.12 & 94.47 & 97.90 & 99.20
& 66.13 & 85.34 & 89.79 & 94.14 & 96.35 
& 64.58 & 84.59 & 89.54 & 94.00 & 96.40 \\
\textbf{+our $\ourloss$}
& \textbf{74.06} & \textbf{94.24} & \textbf{97.35} & \textbf{99.23} & \textbf{99.80}
& \textbf{72.94} & \textbf{93.27} & \textbf{96.55} & \textbf{98.78} & \textbf{99.61}
& \textbf{68.88} & \textbf{86.64} & \textbf{91.10} & \textbf{95.24} & \textbf{96.90}
& \textbf{67.33} & \textbf{86.69} & \textbf{90.75} & \textbf{94.85} & \textbf{97.00} \\
\textit{$\Delta$}
& \gain{6.26} & \gain{4.02} & \gain{2.51} & \gain{1.01} & \gain{0.35}
& \gain{6.55} & \gain{4.15} & \gain{2.08} & \gain{0.88} & \gain{0.41}
& \gain{2.75} & \gain{1.30} & \gain{1.31} & \gain{1.10} & \gain{0.55}
& \gain{2.75} & \gain{2.10} & \gain{1.21} & \gain{0.85} & \gain{0.60} \\
\lightmidrule

% --- SmartCLIP ---
SmartCLIP
& 74.92 & 94.08 & 97.31 & 99.37 & 99.82 
& 74.91 & 94.04 & 97.29 & 99.32 & 99.84
& 69.88 & 86.64 & 94.05 & 95.00 & 97.25 
& 70.94 & 87.04 & 92.77 & 95.75 & 97.05 \\
\textbf{+our $\ourloss$}
& \textbf{77.39} & \textbf{95.57} & \textbf{98.66} & \textbf{99.47} & \textbf{99.94}
& \textbf{77.10} & \textbf{95.49} & \textbf{98.34} & \textbf{99.86} & \textbf{99.89}
& \textbf{69.93} & \textbf{86.94} & \textbf{94.35} & \textbf{95.10} & \textbf{97.30}
& \textbf{71.14} & \textbf{87.64} & \textbf{94.10} & \textbf{95.80} & \textbf{98.60} \\
\textit{$\Delta$}
& \gain{2.47} & \gain{1.49} & \gain{1.35} & \gain{0.10} & \gain{0.12}
& \gain{2.19} & \gain{1.45} & \gain{1.05} & \gain{0.54} & \gain{0.05}
& \gain{0.05} & \gain{0.30} & \gain{0.30} & \gain{0.10} & \gain{0.05}
& \gain{0.20} & \gain{0.60} & \gain{1.33} & \gain{0.05} & \gain{0.55} \\
\lightmidrule

% --- GOAL ---
GOAL
& 79.47 & 96.65 & 98.69 & 99.69 & \textbf{99.92}
& 79.43 & 96.14 & 97.25 & 99.61 & 99.90
& 72.64 & \textbf{89.89} & 93.70 & 95.75 & 97.25 
& 72.84 & \textbf{90.50} & 93.20 & 96.60 & 97.60 \\
\textbf{+our $\ourloss$}
& \textbf{80.96} & \textbf{96.90} & \textbf{98.96} & \textbf{99.76} & \textbf{99.92}
& \textbf{80.31} & \textbf{96.73} & \textbf{98.84} & \textbf{99.75} & \textbf{99.94}
& \textbf{72.89} & 89.79 & \textbf{94.40} & \textbf{96.15} & \textbf{97.65}
& \textbf{73.89} & 89.93 & \textbf{93.50} & \textbf{96.80} & \textbf{98.15} \\
\textit{$\Delta$}
& \gain{1.49} & \gain{0.25} & \gain{0.27} & \gain{0.07} & \neutral{0.00}
& \gain{0.88} & \gain{0.59} & \gain{0.33} & \gain{0.14} & \gain{0.04}
& \gain{0.25} & \loss{0.10} & \gain{0.7} & \gain{0.20} & \gain{0.40}
& \gain{1.05} & \loss{0.57} & \gain{0.30} & \gain{0.20} & \gain{0.25} \\
\lightmidrule

% --- SIGLIP2 ---
SigLIP2
 & 71.80 & 92.53 & 95.88 & 98.78 & 99.43 & 71.51 & 92.41 & 96.06 & 98.51 & 99.39
& 66.13 & 84.49 & 89.34 & 94.00 & 96.40 & 65.08 & 84.54 & 89.84 & 94.65 & \textbf{96.95} \\
\textbf{+our $\ourloss$}
& \textbf{75.47} & \textbf{94.82} & \textbf{97.67} & \textbf{99.30} & \textbf{99.78} & \textbf{73.59} & \textbf{94.33} & \textbf{97.43} & \textbf{99.37} & \textbf{99.84} & \textbf{67.14} & \textbf{86.54} & \textbf{90.65} & \textbf{94.85} & \textbf{96.85} & \textbf{66.78} & \textbf{86.09} & \textbf{90.50} & \textbf{94.75} & 96.90 \\
\textit{$\Delta$}
& \gain{3.67} & \gain{2.29} & \gain{1.79} & \gain{0.52} & \gain{0.35} & \gain{2.08} & \gain{1.92} & \gain{1.37} & \gain{0.86} & \gain{0.45} & \gain{1.01} & \gain{2.05} & \gain{1.31} & \gain{0.85} & \gain{0.45} & \gain{1.7} & \gain{1.55} & \gain{0.66} & \gain{0.10} & \loss{0.05} \\
\midrule

% --- LoRA ---
LoRA
& 77.80 & 96.45 & 98.55 & 99.11 & 99.14 
& 77.00 & 96.02 & 98.05 & 99.32 & 99.14
& 72.04 & 88.69 & 92.35 & 95.50 & 97.15 
& 72.04 & \textbf{89.24} & \textbf{93.20} & 95.00 & 97.40 \\
\textbf{+our $\ourloss$}
& \textbf{79.35} & \textbf{96.61} & \textbf{98.57} & \textbf{99.61} & \textbf{99.40}
& \textbf{78.65} & \textbf{96.31} & \textbf{98.45} & \textbf{99.63} & \textbf{99.86}
& \textbf{74.44} & \textbf{89.44} & \textbf{92.85} & \textbf{95.95} & \textbf{97.20}
& \textbf{73.89} & \textbf{89.24} & 92.60 & \textbf{96.40} & \textbf{97.70} \\
\textit{$\Delta$}
& \gain{1.55} & \gain{0.16} & \gain{0.02} & \gain{0.50} & \gain{0.26}
& \gain{1.65} & \gain{0.29} & \gain{0.40} & \gain{0.31} & \gain{0.72}
& \gain{2.40} & \gain{0.75} & \gain{0.50} & \gain{0.45} & \gain{0.05}
& \gain{1.85} & \neutral{0.00} & \loss{0.60} & \gain{1.40} & \gain{0.30} \\
\lightmidrule

% --- DoRA ---
DoRA
& 66.27 & 90.50 & 95.18 & 98.49 & 99.67 
& 65.65 & 89.55 & 94.78 & 98.41 & 99.36
& 60.33 & 82.19 & 88.09 & 93.45 & 95.75 
& 60.53 & 81.84 & 87.79 & 93.40 & 95.30 \\
\textbf{+our $\ourloss$}
& \textbf{70.65} & \textbf{92.47} & \textbf{96.57} & \textbf{99.02} & \textbf{99.73}
& \textbf{69.22} & \textbf{91.80} & \textbf{95.98} & \textbf{98.98} & \textbf{99.61}
& \textbf{62.58} & \textbf{82.34} & \textbf{88.34} & \textbf{93.75} & \textbf{96.10}
& \textbf{62.78} & \textbf{82.99} & \textbf{88.39} & \textbf{94.25} & \textbf{96.75} \\
\textit{$\Delta$}
& \gain{4.38} & \gain{1.97} & \gain{1.39} & \gain{0.53} & \gain{0.06}
& \gain{3.57} & \gain{2.25} & \gain{1.20} & \gain{0.57} & \gain{0.25}
& \gain{2.25} & \gain{0.15} & \gain{0.25} & \gain{0.30} & \gain{0.35}
& \gain{2.25} & \gain{1.15} & \gain{0.60} & \gain{0.85} & \gain{1.45} \\

\bottomrule
\end{tabular}
}
\vspace{-8pt}
\end{table*}
%%%%%%%%%%%%%%%%%%%%%%%%%%%%%%%%%%%%%%%%%%

%% file: tables/supp_dataefficiency.tex
\begin{table*}[t!]
\centering
\caption{
\textbf{Data Efficiency Analysis.}
We report mean Recall@K (\%) on SKETCHY, INSECT, DOCCI, and DCI using \textbf{5\%}, \textbf{20\%}, \textbf{50\%} respectively of the training data.
Results are reported for K = 1, 5, 10, 25, 50 on both
\inlineColorbox{lightpink}{\emph{Text$\rightarrow$Image}} and
\inlineColorbox{lightblue}{\emph{Image$\rightarrow$Text}}.
\textbf{Bold} indicates the best result, while \underline{underline} denotes the second-best.
}
\label{tab:data_efficiency_5}
\vspace{-5pt}

\setlength{\tabcolsep}{3.5pt}

% ===================== SKETCHY & INSECT (5%) =====================
\resizebox{\textwidth}{!}{
\begin{tabular}{@{} l c c c c c | c c c c c | c c c c c | c c c c c @{}}

\toprule
\textbf{Setting} & \cellcolor{lightpink}R@1 & \cellcolor{lightpink}R@5 & \cellcolor{lightpink}R@10 & \cellcolor{lightpink}R@25 & \cellcolor{lightpink}R@50
& \cellcolor{lightblue}R@1 & \cellcolor{lightblue}R@5 & \cellcolor{lightblue}R@10 & \cellcolor{lightblue}R@25 & \cellcolor{lightblue}R@50
& \cellcolor{lightpink}R@1 & \cellcolor{lightpink}R@5 & \cellcolor{lightpink}R@10 & \cellcolor{lightpink}R@25 & \cellcolor{lightpink}R@50
& \cellcolor{lightblue}R@1 & \cellcolor{lightblue}R@5 & \cellcolor{lightblue}R@10 & \cellcolor{lightblue}R@25 & \cellcolor{lightblue}R@50 \\ 
\midrule

\cellcolor{yellow!10}\textbf{5\% Data} & \multicolumn{10}{c|}{\cellcolor{panelgray}\textbf{SKETCHY}} 
& \multicolumn{10}{c}{\cellcolor{panelgray}\textbf{INSECT}}\\

\lightmidrule
Long-CLIP{\textcolor{gray}{[ECCV'24]}}
&21.42 & 47.58 & 59.59 & 76.77 & 86.18 & 23.40 & 49.14 & 62.44 & 78.24 & 87.31 & 2.50 & 6.99 & 10.88 & 18.83 & 27.98 & 2.68 & 8.03 & 12.61 & 20.55 & 30.66 \\
FineLIP{\textcolor{gray}{[CVPR'25]}}
&30.92 & 60.28 & 71.68 & 85.49 & 91.88 & 31.00 & 59.50 & 70.12 & 84.97 & 92.57 & 2.50 & \underline{8.12} & \underline{11.87} & \underline{22.13} & \underline{32.99} & 3.28 & 8.33 & 13.04 & \underline{22.31} & 32.06 \\
SmartCLIP{\textcolor{gray}{[CVPR'25]}}
&25.22 & 53.02 & \underline{77.12} & 81.35 & 90.59 & 27.98 & 54.15 & 75.76 & 82.90 & 91.97 & \underline{2.68} & 7.43 & 11.66 & 19.86 & 30.14 & \underline{3.37} & \underline{8.44} & \underline{13.15} & 21.66 & \underline{32.25} \\
GOAL{\textcolor{gray}{[CVPR'25]}}
& \underline{33.42} & \underline{63.99} & 74.18 & \underline{86.23} & \underline{92.49} & \underline{35.05} & \underline{64.08} & \underline{75.82} & \underline{88.69} & \underline{93.53} & \underline{2.68} & 7.51 & 11.83 & 22.11 & 31.35 & 3.34 & 8.20 & 12.33 & 21.24 & 31.78  \\
\textbf{\ourmethod}
& \textbf{36.70} & \textbf{68.05} & \textbf{78.07} & \textbf{88.43} & \textbf{94.39} & \textbf{35.75} & \textbf{68.05} & \textbf{77.46} & \textbf{89.12} & \textbf{94.91} & \textbf{3.33} & \textbf{8.46} & \textbf{12.87} & \textbf{22.28} & \textbf{33.68} & \textbf{3.57} & \textbf{8.98} & \textbf{13.56} & \textbf{22.54} & \textbf{33.16}  \\
\lightmidrule
\textit{$\Delta$}
& \gain{3.28} & \gain{4.06} & \gain{0.95} & \gain{2.20} & \gain{1.90} & \gain{0.70} & \gain{3.97} & \gain{1.64} & \gain{0.43} & \gain{1.38} & \gain{0.65} & \gain{0.34} & \gain{1.00} & \gain{0.15} & \gain{0.69} & \gain{0.20} & \gain{0.54} & \gain{0.41} & \gain{0.23} & \gain{0.91}  \\
									
\midrule

% ===================== DOCCI & DCI (5%) =====================

\cellcolor{yellow!10}\textbf{5\% Data}& \multicolumn{10}{c|}{\cellcolor{panelgray}\textbf{DOCCI}} & \multicolumn{10}{c}{\cellcolor{panelgray}\textbf{DCI}}\\

\lightmidrule
Long-CLIP{\textcolor{gray}{[ECCV'24]}}
& 61.20 & 86.33 & 92.61 & 97.43 & 99.08 & 61.57 & 86.51 & 92.51 & 97.16 & 98.98 & 53.73 & 75.34 & 82.24 & 88.79 & 92.25 & 53.98 & 76.69 & 83.74 & 90.05 & 93.85 \\
FineLIP{\textcolor{gray}{[CVPR'25]}}
& 66.00 & 89.73 & 94.75 & 97.92 & 99.15 & 63.65 & 87.39 & 92.20 & 97.06 & 98.20 & 58.08 & 79.74 & 85.79 & 91.45 & \textbf{95.15} & 60.88 & 80.29 & 86.64 & 93.35 & 95.80 \\
SmartCLIP{\textcolor{gray}{[CVPR'25]}}
& 65.12 & 89.33 & 94.37 & \underline{98.25} & 99.31 & 65.71 & 89.76 & 95.20 & \underline{98.27} & \underline{99.00} & 57.78 & 78.14 & 87.33 & 90.60 & 93.78 & 56.13 & 78.84 & 88.14 & 91.05 & 95.30 \\
GOAL{\textcolor{gray}{[CVPR'25]}}
& \underline{66.45} & \underline{91.29} & \underline{95.00} & 98.17 & \underline{99.57} & \underline{67.85} & \underline{91.47} & \underline{95.44} & 98.20 & \underline{99.00} & \underline{60.88} & \underline{81.54} & \underline{87.59} & \underline{92.25} & 94.90 & \underline{63.13} & \underline{83.14} & \underline{88.94} & \underline{93.55} & \underline{96.15}  \\
\textbf{\ourmethod}
& \textbf{68.92} & \textbf{91.84} & \textbf{95.84} & \textbf{98.49} & \textbf{99.61} & \textbf{69.76} & \textbf{91.61} & \textbf{95.76} & \textbf{98.43} & \textbf{99.49} & \textbf{62.53} & \textbf{82.99} & \textbf{88.20} & \textbf{93.15} & \underline{95.10} & \textbf{64.03} & \textbf{84.34} & \textbf{89.19} & \textbf{93.80} & \textbf{96.60}  \\
\lightmidrule
\textit{$\Delta$}
& \gain{2.47} & \gain{0.55} & \gain{0.84} & \gain{0.24} & \gain{0.04} & \gain{1.91} & \gain{0.14} & \gain{0.32} & \gain{0.16} & \gain{0.49} & \gain{1.65} & \gain{1.45} & \gain{0.61} & \gain{0.90} & \loss{0.05} & \gain{0.9} & \gain{1.20} & \gain{0.25} & \gain{0.25} & \gain{0.45}  \\

\midrule

\cellcolor{yellow!10}\textbf{20\% Data}& \multicolumn{10}{c|}{\cellcolor{panelgray}\textbf{SKETCHY}} 
& \multicolumn{10}{c}{\cellcolor{panelgray}\textbf{INSECT}}\\

% ===================== SKETCHY & INSECT (20%) =====================

\lightmidrule
Long-CLIP{\textcolor{gray}{[ECCV'24]}}
& 35.75 & 64.85 & 75.73 & 85.92 & 93.96
& 36.01 & 65.72 & 76.86 & 88.08 & 94.39 
& \underline{4.05} & 11.05 & 18.74 & 30.92 & 41.54
& 4.15 & \textbf{14.68} & 20.29 & 32.82 & 43.70 \\
FineLIP{\textcolor{gray}{[CVPR'25]}}
& 38.17 & 69.17 & 78.84 & 89.81 & 95.51
& 35.75 & 67.18 & 78.41 & 88.32 & 94.99 
& 3.63 & 12.44 & \underline{19.08} & 30.31 & 43.09
& 4.23 & 13.64 & 20.81 & 31.69 & 43.52\\
SmartCLIP{\textcolor{gray}{[CVPR'25]}}
& 38.00 & 68.48 & \underline{88.08} & 90.85 & \underline{97.24}
& 38.95 & 69.86 & \underline{88.43} & 90.85 & 96.29 
& 3.45 & 9.54 & 15.28 & 28.14 & 40.39
& 3.20 & 12.15 & 19.39 & 29.74 & 40.01\\
GOAL{\textcolor{gray}{[CVPR'25]}}
& \underline{46.63} & \underline{77.37} & 85.06 & \underline{92.45} & 96.90
& \underline{46.46} & \underline{76.94} & 86.10 & \underline{92.66} & \underline{96.43} 
& 3.21 & \underline{12.80} & 18.83 & \underline{33.51} & \underline{45.37}
& \underline{4.33} & 13.99 & \underline{21.16} & \underline{33.15} & \underline{46.06} \\
\textbf{\ourmethod}
& \textbf{53.20} & \textbf{80.22} & \textbf{88.26} & \textbf{94.39} & \textbf{97.50}
& \textbf{49.48} & \textbf{79.45} & \textbf{88.51} & \textbf{94.73} & \textbf{97.15} & \textbf{4.58} & \textbf{13.30} & \textbf{19.60} & \textbf{34.11} & \textbf{47.58}
& \textbf{4.92} & \underline{14.16} & \textbf{22.12} & \textbf{34.22} & \textbf{46.72}\\
\lightmidrule
\textit{$\Delta$}
& \gain{6.57} & \gain{2.85} & \gain{0.18} & \gain{1.94} & \gain{0.26}
& \gain{3.02} & \gain{2.51} & \gain{0.08} & \gain{2.07} & \gain{0.72} 
& \gain{0.53} & \gain{0.50} & \gain{0.52} & \gain{0.60} & \gain{2.21}
& \gain{0.59} & \loss{0.52} & \gain{0.96} & \gain{1.07} & \gain{0.66} \\
									
\midrule

% ===================== DOCCI & DCI (20%) =====================

\cellcolor{yellow!10}\textbf{20\% Data}& \multicolumn{10}{c|}{\cellcolor{panelgray}\textbf{DOCCI}} & \multicolumn{10}{c}{\cellcolor{panelgray}\textbf{DCI}}\\

\lightmidrule
Long-CLIP{\textcolor{gray}{[ECCV'24]}}
& 62.16 & 86.75 & 92.61 & 97.14 & 99.00
& 61.61 & 86.67 & 92.00 & 97.37 & 99.02 
& 56.53 & 79.64 & 85.54 & 91.05 & 94.00
& 57.33 & 79.14 & 85.84 & 91.00 & 94.40 \\
FineLIP{\textcolor{gray}{[CVPR'25]}}
& 67.18 & 89.92 & 94.00 & 98.10 & 99.22
& 63.96 & 87.80 & 92.40 & 97.41 & 98.37 
& 63.98 & 83.79 & 88.69 & 93.05 & 95.85
& 62.83 & 83.79 & 88.69 & 93.05 & 95.85\\
SmartCLIP{\textcolor{gray}{[CVPR'25]}}
& 70.00 & 91.69 & 95.31 & 98.80 & 99.41
& 72.12 & 92.64 & 96.24 & \underline{99.00} & 99.38 
& \underline{64.78} & 82.84 & 89.00 & 92.89 & 95.97
& 62.13 & 83.09 & 89.44 & 93.45 & 96.15\\
GOAL{\textcolor{gray}{[CVPR'25]}}
& \underline{71.43} & \underline{93.12} & \underline{96.90} & \underline{99.18} & \underline{99.75}
& \underline{72.94} & \underline{93.63} & \underline{96.84} & \underline{99.00} & \underline{99.67} 
& 64.63 & \underline{84.84} & \underline{90.05} & \underline{93.90} & \underline{96.20}
& \underline{65.73} & \underline{85.29} & \underline{90.65} & \underline{94.65} & \underline{97.00}\\
\textbf{\ourmethod}
& \textbf{77.18} & \textbf{95.61} & \textbf{97.82} & \textbf{99.47} & \textbf{99.82}
& \textbf{76.47} & \textbf{95.31} & \textbf{97.86} & \textbf{99.39} & \textbf{99.82} 
& \textbf{67.03} & \textbf{86.24} & \textbf{90.60} & \textbf{94.59} & \textbf{96.50}
& \textbf{69.03} & \textbf{87.34} & \textbf{91.35} & \textbf{95.45} & \textbf{97.50} \\
\lightmidrule
\textit{$\Delta$}
& \gain{5.75} & \gain{2.49} & \gain{0.92} & \gain{0.29} & \gain{0.07}
& \gain{3.53} & \gain{1.68} & \gain{1.02} & \gain{0.39} & \gain{0.15} 
& \gain{2.25} & \gain{1.40} & \gain{0.55} & \gain{0.69} & \gain{0.30}
& \gain{3.30} & \gain{2.05} & \gain{0.70} & \gain{0.80} & \gain{0.50}\\

\midrule

\cellcolor{yellow!10}\textbf{50\% Data}& \multicolumn{10}{c|}{\cellcolor{panelgray}\textbf{SKETCHY}} 
& \multicolumn{10}{c}{\cellcolor{panelgray}\textbf{INSECT}}\\

\lightmidrule
Long-CLIP{\textcolor{gray}{[ECCV'24]}}
&47.93 & 74.01 & 82.99 & 92.23 & 96.29 & 43.96 & 73.58 & 82.56 & 91.88 & 96.03 & 5.18 & 15.80 & 25.47 & 40.76 & 53.28 & 4.75 & 18.13 & \underline{27.63} & 40.59 & 53.11 \\
FineLIP{\textcolor{gray}{[CVPR'25]}}
&39.98 & 69.44 & 79.97 & 90.15 & 95.60 & 37.65 & 68.91 & 79.10 & 88.95 & 94.99 & 5.61 & 17.10 & 24.96 & \underline{41.11} & 54.75 & 5.01 & 17.44 & 26.34 & 41.02 & 54.06 \\
SmartCLIP{\textcolor{gray}{[CVPR'25]}}
&45.34 & 76.34 & \underline{90.93} & 93.01 & 97.25 & 46.46 & 74.70 & \underline{90.16} & 92.06 & 97.58 & 3.57 & 11.92 & 19.74 & 33.31 & 43.71 & 3.63 & 13.31 & 21.67 & 32.97 & 42.06 \\
GOAL{\textcolor{gray}{[CVPR'25]}}
& \underline{55.79} & \underline{82.73} & 88.35 & \underline{96.20} & \underline{98.10} & \underline{55.27} & \underline{82.04} & 89.12 & \underline{95.42} & \underline{98.36} & \underline{5.87} & \underline{18.13} & \underline{25.51} & 41.02 & \underline{55.22} & \underline{5.35} & \underline{18.77} & 26.84 & \underline{43.05} & \underline{56.55}  \\
\textbf{\ourmethod}
& \textbf{60.97} & \textbf{87.48} & \textbf{92.49} & \textbf{96.89} & \textbf{98.70} & \textbf{59.41} & \textbf{85.58} & \textbf{91.80} & \textbf{95.55} & \textbf{98.70} & \textbf{6.56} & \textbf{20.12} & \textbf{29.17} & \textbf{43.78} & \textbf{57.77} & \textbf{6.22} & \textbf{20.64} & \textbf{29.84} & \textbf{44.13} & \textbf{57.94}\\
\lightmidrule
\textit{$\Delta$}
& \gain{5.18} & \gain{4.75} & \gain{1.56} & \gain{0.69} & \gain{0.60} & \gain{4.14} & \gain{3.54} & \gain{1.64} & \gain{0.13} & \gain{0.34} & \gain{0.69} & \gain{1.99} & \gain{3.66} & \gain{2.67} & \gain{2.55} & \gain{0.87} & \gain{1.87} & \gain{2.21} & \gain{1.08} & \gain{1.39}\\
									
\midrule

% ===================== DOCCI & DCI (50%) =====================

\cellcolor{yellow!10}\textbf{50\% Data}& \multicolumn{10}{c|}{\cellcolor{panelgray}\textbf{DOCCI}} & \multicolumn{10}{c}{\cellcolor{panelgray}\textbf{DCI}}\\

\lightmidrule
Long-CLIP{\textcolor{gray}{[ECCV'24]}}
& 63.49 & 87.49 & 93.20 & 97.33 & 99.06 & 62.20 & 87.41 & 92.90 & 97.40 & 99.02 & 58.63 & 80.79 & 85.99 & 91.65 & 94.65 & 59.04 & 80.89 & 85.69 & 92.20 & 95.35 \\
FineLIP{\textcolor{gray}{[CVPR'25]}}
& 67.20 & 90.06 & 94.75 & 98.22 & 99.31 & 64.29 & 87.94 & 93.10 & 97.72 & 98.86 & 65.43 & 84.44 & 88.10 & 94.00 & 96.05 & 63.83 & 84.44 & 89.19 & 93.14 & 95.00 \\
SmartCLIP{\textcolor{gray}{[CVPR'25]}}
& 73.75 & 94.24 & 97.41 & 99.31 & \underline{99.77} & 74.00 & 94.04 & 97.17 & 98.84 & \underline{99.65} & 66.98 & 84.79 & 90.77 & 93.40 & 96.00 & 66.20 & 84.84 & 91.15 & 94.65 & \underline{97.12} \\
GOAL{\textcolor{gray}{[CVPR'25]}}
& \underline{75.25} & \underline{94.61} & \underline{97.59} & \underline{99.51} & \textbf{99.88} & \underline{75.73} & \underline{94.53} & \underline{97.37} & \underline{99.45} & \textbf{99.80} & \underline{69.03} & \underline{86.84} & \underline{91.10} & \underline{94.95} & \underline{96.50} & \underline{67.63} & \underline{86.69} & \underline{91.20} & \underline{95.65} & 97.05 \\
\textbf{\ourmethod}
& \textbf{79.78} & \textbf{96.37} & \textbf{98.31} & \textbf{99.61} & \textbf{99.88} & \textbf{78.45} & \textbf{95.98} & \textbf{98.43} & \textbf{99.57} & \textbf{99.80} & \textbf{71.19} & \textbf{87.84} & \textbf{91.30} & \textbf{95.55} & \textbf{97.05} & \textbf{71.64} & \textbf{87.39} & \textbf{92.65} & \textbf{96.90} & \textbf{97.50}  \\
\lightmidrule
\textit{$\Delta$}
& \gain{4.53} & \gain{1.76} & \gain{0.72} & \gain{0.10} & 0.00 & \gain{2.72} & \gain{1.45} & \gain{1.06} & \gain{0.12} & 0.00 & \gain{2.16} & \gain{1.00} & \gain{0.20} & \gain{0.60} & \gain{0.55} & \gain{4.01} & \gain{0.70} & \gain{1.45} & \gain{1.25} & \gain{0.38}  \\

\bottomrule

\end{tabular}
}
\vspace{0.6em}
\end{table*}